\documentclass{article} % For LaTeX2e
\usepackage{iclr2024_conference,times}

\usepackage{xcolor}   
\usepackage{hyperref}
\usepackage{url}
\usepackage{amsmath}
\usepackage{amssymb}
\usepackage{amsthm}
\usepackage{amsfonts}
\usepackage{bm}
\usepackage{graphicx}
\usepackage{caption}
\usepackage{subcaption}
\usepackage{algorithm}
\usepackage[noend]{algpseudocode}
\usepackage{thmtools}
\usepackage{thm-restate}
\usepackage{cleveref}
\usepackage{enumerate}
\usepackage{verbatim}

\newenvironment{myquote}[1]%
{\list{}{\leftmargin=#1\rightmargin=#1}\item[]}%
{\endlist}

\numberwithin{equation}{section}

\algnewcommand\algorithmicforeach{\textbf{for each}}
\algdef{S}[FOR]{ForEach}[1]{\algorithmicforeach\ #1\ \algorithmicdo}
\let\oldReturn\Return
\renewcommand{\Return}{\State\oldReturn}

\newtheorem{theorem}{Theorem}[section]
\newtheorem{lemma}[theorem]{Lemma}
\newtheorem{definition}{Definition}[section]

\newtheorem{corollary}{Corollary}[section]
\newtheorem{proposition}[theorem]{Proposition}

\DeclareMathOperator*{\argmin}{arg\,min}
\newcommand{\E}{\mathbb{E}}
\newcommand{\R}{\mathbb{R}}
\newcommand{\bs}[1]{\boldsymbol{#1}}
\newcommand{\bv}[1]{\mathbf{#1}}

\title{Improved Active Learning via Dependent Leverage Score Sampling}

\author{Atsushi Shimizu,  Xiaoou Cheng, Christopher Musco, Jonathan Weare \\
	New York University \\
	\texttt{\{as15106,chengxo,cmusco,weare\}@nyu.edu} \\
}

  \usepackage{nth}
  \usepackage{intcalc}

  \newcommand{\cAAAI}[1]{AAAI\ Conference\ on\ Artificial (AAAI)}

\iclrfinalcopy % Uncomment for camera-ready version, but NOT for submission.
\begin{document}
	
\maketitle

\begin{abstract}
    We show how to obtain improved active learning methods in the agnostic (adversarial noise) setting by combining marginal leverage score sampling with non-independent sampling strategies that promote spatial coverage. In particular, we propose an easily implemented method based on the \emph{pivotal sampling algorithm}, which we test on problems motivated by learning-based methods for parametric PDEs and uncertainty quantification. In comparison to independent sampling, our method reduces the number of samples needed to reach a given target accuracy by up to $50\%$.
    We support our findings with two theoretical results. First, we show that any non-independent leverage score sampling method that obeys a weak \emph{one-sided $\ell_{\infty}$ independence condition} (which includes pivotal sampling) can actively learn $d$ dimensional linear functions with $O(d\log d)$ samples, matching independent sampling. This result extends recent work on matrix Chernoff bounds under $\ell_{\infty}$ independence, and may be of interest for analyzing other sampling strategies beyond pivotal sampling. 
    Second, we show that, for the important case of polynomial regression, our pivotal method obtains an improved bound on $O(d)$ samples.
\end{abstract}

\section{Introduction}\label{sec:introduction}
In the active linear regression problem, we are given a data matrix $\mathbf{A} \in \R^{n \times d}$ with $n\gg d$ rows and query access to a target vector $\bv{b}\in \R^n$. The goal is to learn parameters $\bv{x}\in \R^d$ such that $\bv{A}\bv{x} \approx \bv{b}$ while observing as few entries in $\bv{b}$ as possible. 
We study this problem in the challenging agnostic learning or ``adversarial noise'' setting, where we do not assume any underlying relationship between $\mathbf{A}$ and $\mathbf{b}$. Instead, our goal is to find parameters competitive with the best possible fit, good or bad. Specifically, considering $\ell_2$ loss, let $\mathbf{x}^* = \argmin_{\mathbf{x}} \lVert \mathbf{A} \mathbf{x} - \mathbf{b} \rVert_2^2$ be optimal model parameters. We want to find $\tilde{\mathbf{x}}^*$ using a small number of queried target values in $\bv{b}$ such that 
\begin{align}
\label{main:gaur}
\begin{split}
     \lVert \mathbf{A} \tilde{\mathbf{x}}^* - \mathbf{b} \rVert_2^2 \leq (1 + \epsilon) \lVert \mathbf{A} \mathbf{x}^* - \mathbf{b} \rVert_2^2,
\end{split}
\end{align}
 for some error parameter $\epsilon > 0$.
Beyond being a fundamental learning problem, active regression has emerged as a fundamental tool in learning based methods for the solution and uncertainty analysis of parametric partial differential equations (PDEs) \citep{esiam_2015,siam_review_2022}. For such applications, the agnostic setting is crucial, as a potentially complex quantity of interest is approximated by a simple surrogate model (e.g. polynomials, sparse polynomials, single layer neural networks, etc.) \citep{sparse_review_2021,ridge_functions_2018}. Additionally, reducing the number of labels used for learning is crucial, as each label usually requires the computationally expensive numerical solution of a PDE for a new set of parameters \citep{cohen_devore_2015}.

\subsection{Leverage Score Sampling}\label{sec:lev_score}
Active linear regression has been studied for decades in the statistical model where $\bv{b}$ is assumed to equal $\bv{A}\bv{x}^*$ plus i.i.d. random noise. In this case, the problem can be addressed using tools from optimal experimental design \citep{Pukelsheim:2006}. In the agnostic case, near-optimal sample complexity results were only obtained relatively recently using tools from non-asymptotic matrix concentration \citep{Tropp11}. In particular, it was shown independently in several papers that collecting entries from $\bv{b}$ \emph{randomly} with probability proportional to the \emph{statistical leverage scores} of rows in $\bv{A}$ can achieve \eqref{main:gaur} with $O(d\log d + d/\epsilon)$ samples \citep{Sarlos:2006,RauhutWard:2012,HamptonDoostan:2015,CohenMigliorati:2017}. The leverage scores are defined as follows:
\begin{definition}[Leverage Score]\label{def:leveragescore}
    Let $\bv{U}\in \R^{n \times r}$ be any orthogonal basis for the column span of a matrix $\mathbf{A} \in \R^{n \times d}$. Let $\mathbf{a}_i$ and $\mathbf{u}_i$ be the $i$-th rows of $\mathbf{A}$ and $\bv{U}$, respectively. The \textit{leverage score} $\tau_i$ of the $i$-th row in $\mathbf{A}$ can be equivalently written as:
    \vspace{-.25em}
    \begin{align}\begin{split}
    \label{eq:lev_defs}
        \tau_i = \lVert \mathbf{u}_i \rVert_2^2 = \mathbf{a}_i^T (\mathbf{A}^T \mathbf{A})^{-1} \mathbf{a}_i = \max_{\mathbf{x} \in \R^d} {(\mathbf{a}_i^T \mathbf{x})^2}/{\lVert \mathbf{A} \mathbf{x} \rVert_2^2}.
    \end{split}\end{align}
    \vspace{-2em}
    
    \noindent Notice that $\tau_i = \lVert \mathbf{u}_i \rVert_2^2$, $\sum_{i=1}^n \tau_i = d$ when $\mathbf{A}$ is full-rank and thus $\bv{U}$ has $r =d$ columns. 
    \end{definition}
    \vspace{-.25em}

See \citep{pmlr-v70-avron17a} for a short proof of the final equality in \eqref{eq:lev_defs}. This last definition, based on a maximization problem, gives an intuitive understanding of the leverage scores. The score of row $\bv{a}_i$ is higher if it is more ``exceptional'', meaning that we can find a vector $\bv{x}$ that has large inner product with $\bv{a}_i$ relative to its average inner product (captured by $\|\bv{A}\bv{x}\|_2^2$) with all other rows in the matrix. Based on leverage score, rows that are more exceptional are sampled with higher probability.

Prior work considers independent leverage score sampling, either with or without replacement. The typical approach for sampling without replacement, which we call ``Bernoulli sampling'' is as follows: Each row $\bv{a}_i$ is assigned a probability $p_i = \min(1,c\cdot \tau_i)$ for an oversampling parameter $c \geq 1$. Then each row is sampled independently with probability $p_i$. We construct a subsampled data matrix $\tilde{\bv{A}}$ and subsampled target vector $\tilde{\bv{b}}$ by adding $\bv{a}_i/\sqrt{p_i}$ to $\tilde{\bv{A}}$ and $b_i/\sqrt{p_i}$ to $\tilde{\bv{b}}$ for any index $i$ that is
 sampled. To solve the active regression problem, we return $\tilde{\bv{x}}^* = \argmin_{\bv{x}} \|\tilde{\bv{A}}\bv{x} - \tilde{\bv{b}}\|_2$.
 
 \begin{figure}[t]
 	\centering
 	\begin{subfigure}[t]{0.32\textwidth}
 		\centering
 		\includegraphics[width=.9\linewidth]{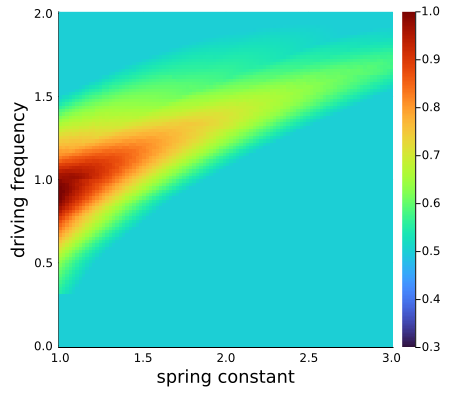}
 		\vspace{-.5em}
 		\caption{Target Function}    
 	\end{subfigure}
 	\hfill
 	\centering
 	\begin{subfigure}[t]{0.32\textwidth}
 		\centering
 		\includegraphics[width=.9\linewidth]{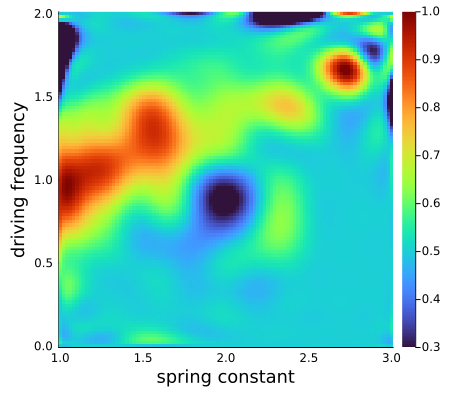}
 		\vspace{-.5em}
 		\caption{Bernoulli Sampling}    
 	\end{subfigure}
 	\hfill
 	\centering
 	\begin{subfigure}[t]{0.32\textwidth}
 		\centering
 		\includegraphics[width=.9\linewidth]{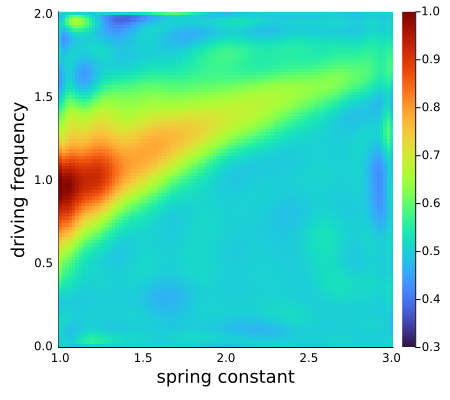}
 		\vspace{-.5em}
 		\caption{Pivotal Sampling (our method)}    
 	\end{subfigure}
 	\vspace{-.5em}
 	\caption{Polynomial approximations to the maximum displacement of a damped harmonic oscillator, as a function of driving frequency and spring constant. 
 		(a) is the target value, and samples can be obtained through the numerical solution of a differential equation governing the oscillator. Both (b) and (c) draw $250$ samples using leverage score sampling and perform polynomial regression of degree $20$. (b) uses Bernoulli sampling while (c) uses our pivotal sampling method.  
 		%		We run the sampling 100 times and show the result with a median error. 
 		Our method gives a better approximation, avoiding artifacts that result from gaps between the Bernoulli samples.}
 	\vspace{-1.5em}
 	\label{fig:sec1:comparison}
 \end{figure} 

\subsection{Our Contributions}\label{subsec:contributions}
In applications to PDEs, the goal is often to approximate a function over a low dimensional distribution $\mathcal{X}$. E.g. $\mathcal{X}$ might be uniform over an interval $[-1,1] \subset \R$ or over a box $[-1,1]\times\ldots \times [-1,1] \subset \R^q$. In this setting, the length $d$ rows of $\bv{A}$ correspond to feature transformations of samples from $\mathcal{X}$. For example, in the ubiquitous task of polynomial regression, we start with  $\bv{x}\sim \mathcal{X}$ and add to $\bv{A}$ a row containing all combinations of entries in $\bv{x}$ with total degree $p$, i.e., $x_1^{\ell_1}x_2^{\ell_2} \ldots x_q^{\ell_q}$ for all non-negative integers $\ell_1, \ldots, \ell_q$ such that $\sum_{i=1}^q \ell_i \leq p$.
For such problems, ``grid'' based interpolation is often used in place of randomized methods like leverage scores sampling, i.e., the target 
$\bv{b}$ is queried on a deterministic grid tailored to $\mathcal{X}$. For example, when $\mathcal{X}$ is uniform on a box, the standard approach is to use a grid based on the Chebyshev nodes \citep{Xiu2016}. Pictured in Figure \ref{fig:sampledist}, the Cheybshev grid concentrates samples near the boundaries of the box, avoiding the well known issue of Runge's phenomenon for uniform grids. Leverage score sampling does the same. In fact, the methods are closely related: in the high degree limit, the leverage scores for polynomial regression over the box match the asymptotic density of the Chebyshev nodes \citep{sparse_review_2021}.

 \begin{figure}[t]
	\vspace{-1em}
	\centering
	\begin{subfigure}[t]{0.32\textwidth}
		\centering
		\includegraphics[width=.6\linewidth]{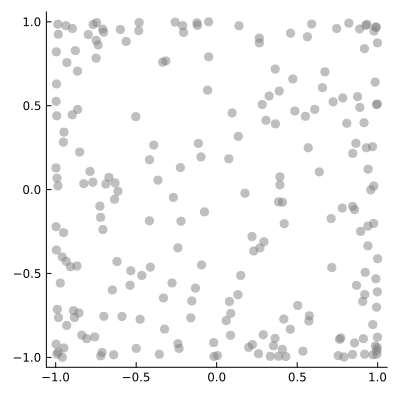}
		\vspace{-.5em}
		\caption{Bernoulli Sampling}    
	\end{subfigure}
	\hfill
	\centering
	\begin{subfigure}[t]{0.32\textwidth}
		\centering
		\includegraphics[width=.6\linewidth]{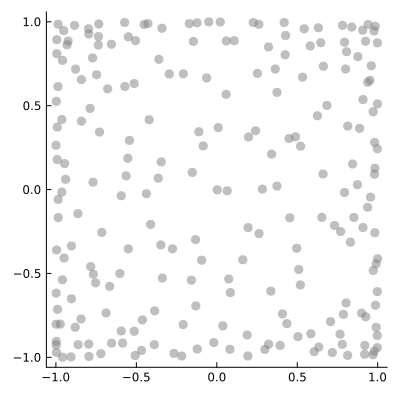}
		\vspace{-.5em}
		\caption{Pivotal Sampling (our method)}    
	\end{subfigure}
	\hfill
	\centering
	\begin{subfigure}[t]{0.32\textwidth}
		\centering
		\includegraphics[width=.6\linewidth]{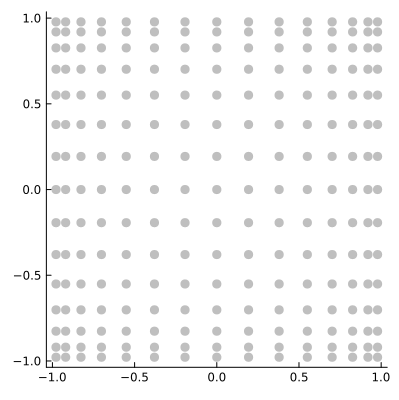}
		\vspace{-.5em}
		\caption{Chebyshev Grid}    
	\end{subfigure}
	\vspace{-.5em}
	\caption{The results of three different active learning methods used to collect samples to fit a polynomial over $[-1,1]\times [-1,1]$. The image on the left was obtained by collecting points independently at random with probability according to their statistical leverage scores. The image on the right was obtained by collecting samples at the 2-dimensional Chebyshev nodes. The image in the middle shows our method, which collects samples according to leverage scores, but using a non-independent pivotal sampling strategy that ensures samples are evenly spread in spatially.}   
	\label{fig:sampledist}
	\vspace{-1em}
\end{figure}

So how do the deterministic and randomized methods compare? The advantage of randomized methods based on leverage score sampling is that they yield strong provable approximation guarantees, and easily generalize to any distribution $\mathcal{X}$.\footnote{As discussed in \citep{CP19}, no deterministic method can provably solve the agnostic regression problem with few samples. Since we make no assumptions on $\bv{b}$, all error in $\bv{A}\bv{x}^* - \bv{b}$ could be concentrated only at the deterministic indices to be selected. Randomization is needed to avoid high error outliers in $\bv{b}$.} Deterministic methods are less flexible on the other hand, and do not yield provable guarantees.
However, the advantage of grid based methods is that they more ``evenly'' distribute samples over the original data domain, which can lead to better performance in practice.  Randomized methods are prone to ``missing'' larger regions of $\mathcal{X}$'s support, as shown in Figure \ref{fig:sampledist}. The driving question behind our work is:

\begin{myquote}{2em}
	\vspace{-.5em}
 \textit{
    Is it possible to obtain the ``best of both worlds'' for fitting functions over low-dimensional domains? I.e., can we match or improve on the strong theoretical guarantees of leverage score sampling with a method that produces spatially well-distributed samples?}
    	\vspace{-.5em}
\end{myquote}

We answer this question in the affirmative. Instead of sampling rows from $\bv{A}$ \emph{independently} with probability proportional to the leverage scores, we adopt a tool from survey sampling known as
pivotal sampling \citep{deville1998unequal}.
Our specific version of pivotal sampling is  \textit{spatially-aware}, meaning that it covers the domain in a well-balanced manner, while the marginal probabilities remain proportional to the leverage score. 
% Details appear in Section \ref{sec:methods}. 
At a high-level, the pivotal method is a ``competition'' based sampling approach, where candidate rows compete in a binary tree tournament. By structuring the tournament so that spatially close points compete at lower levels (we use a novel recursive PCA procedure to build the tree), we ensure better spatial spread than Bernoulli leverage score sampling. 

We show that our pivotal sampling method matches or beats the complexity of independent leverage score sampling in theory, while performing significantly better in practice.
On the practice side, we offer Figure \ref{fig:sec1:comparison} as an example from a PDE test problem. In comparison to independent sampling, our spatially-aware method obtains a much better approximation to the target for a fixed number of samples (more details in Section \ref{sec:experiments}).
On the theory side, we prove two results. The first is general: we show that, as long as it samples rows from $\bv{A}$ with \emph{marginal probabilities} proportional to the leverage scores,  \emph{any} sampling strategy that obeys a weak ``one-sided $\ell_{\infty}$ independence'' condition (which includes pivotal sampling) matches the complexity of independent leverage score sampling:
\begin{theorem}\label{thm:main}
    Let $\mathbf{A} \in \R^{n \times d}$ be a data matrix and $\mathbf{b} \in \R^n$ be a target vector. Consider any algorithm which samples exactly $k$ rows from $\bv{A}$ (and observes the corresponding entries in $\bv{b}$) from a distribution that 1) satisfies one-sided $\ell_{\infty}$ independence (Defn. \ref{def:l-infty}) with parameter $D_{\text{inf}}$ and 2) the marginal probability of sampling any row $\bv{a}_i$ is proportional to $\tau_i$.
    \footnote{Formally, we assume that the marginal probability of sampling $\bv{a}_i$  equals $\min(1,c\tau_i)$ for a fixed constant $c \geq 1$. Our proof easily generalizes to the case when some probabilities exceed this bound (since sampling more never hurts) although the total sample complexity will depend on the sum of the marginal probabilities.} 
    Let $\tilde{\mathbf{A}}$ and $\tilde{\mathbf{b}}$ be the scaled sampled data and target, as defined in Section \ref{sec:lev_score}, and let $\tilde{\mathbf{x}}^* = \argmin_{\mathbf{x} \in \R^d} \lVert \tilde{\mathbf{A}} \mathbf{x} - \tilde{\mathbf{b}} \rVert_2^2$.
  As long as $k \geq c\cdot \left(d\log d \cdot D_{\text{inf}}^2 + \frac{d}{\epsilon}\cdot D_{\text{inf}}\right)$ for a fixed positive constant $c$, then with probability $99/100$, 
    \begin{align}\begin{split}
    \label{eq:main_thm_gaur}
        \lVert \mathbf{A} \tilde{\mathbf{x}}^* - \mathbf{b} \rVert_2^2 \leq (1 + \epsilon) \lVert \mathbf{A} \mathbf{x}^* - \mathbf{b} \rVert_2^2
    \end{split}\end{align}
\vspace{-2em}
\end{theorem}
One-sided $\ell_{\infty}$ independence was introduced in a recent paper by \cite{KKS22} on matrix Chernoff bounds (we provide a formal definition in Sec. \ref{sec:analysis}). It is the \emph{weakest} condition under which  a tight matrix Chernoff bound is known to hold. For example, the condition is implied with constant $D_{\text{inf}} = O(1)$ by almost all existing notions of negative dependence between random variables, including conditional negative association (CNA) and the strongly Rayleigh property \citep{permantle_survey}.  As discussed in Sec. \ref{sec:analysis}, a tight matrix Chernoff bound is a prerequisite for proving relative error active learning results like \Cref{eq:main_thm_gaur}, however does not imply such a result alone. Our proof of  \Cref{thm:main} requires adapting an approximate matrix-multiplication method from \citep{DrineasKannanMahoney:2006} to non-independent sampling. It  can be viewed as extending the work of \cite{KKS22} to show that essentially all sampling distributions known to yield tight matrix Chernoff bounds also yield near optimal active regression bounds in the agnostic setting. Importantly, this includes  binary-tree-based pivotal sampling methods like those introduced in this work. Such methods are known to satisfy the strongly Rayleigh property \citep{BJ12}, and thus one-sided $\ell_{\infty}$ independence with $D_{\text{inf}} = O(1)$.
So, as a corollary of \Cref{thm:main}, we obtain:
\begin{corollary}
\label{corr:main}
    The spatially-aware pivotal sampling methods introduced in Section \ref{sec:methods} (which use a fixed binary tree) return with probability $99/100$ a vector $\tilde{\mathbf{x}}^*$ satisfying $\lVert \mathbf{A} \tilde{\mathbf{x}}^* - \mathbf{b} \rVert_2^2 \leq (1 + \epsilon) \lVert \mathbf{A} \mathbf{x}^* - \mathbf{b} \rVert_2^2$ while only observing $O\left(d\log d + \frac{d}{\epsilon}\right)$ entries in $\bv{b}$.  
\end{corollary}
We hope that \Cref{thm:main} will be valuable in obtaining similar results for other sampling methods beyond our own. However, the result falls short of justifying why pivotal sampling  performs \emph{better} than independent leverage score sampling in experiments. Towards that end, we prove a second result specific to pivotal sampling, which shows that the method actually improves on the complexity of independent sampling by a log factor in the important special case of polynomial regression:
\begin{theorem}\label{thm:poly}
	Consider any function $b: [\ell, u]\rightarrow \R$ defined on an interval $ [\ell, u]\subset \R$, and consider fitting $b$ with a degree $d$ polynomial based on evaluations of the function at $x_1, \ldots, x_k \in [\ell, u]$. 
	If $x_1, \ldots, x_k$ are collected via pivotal sampling with leverage score marginals (see \Cref{app:poly_proof} for details), then as long as $k\geq c \cdot (d + \frac{d}{\epsilon})$ for a fixed positive constant $c$, there is a procedure that uses these samples to construct a degree $d$ polynomial $\tilde{p}$ which, with probability $99/100$, satisfies:
	\begin{align*}
		\|\tilde{p}-b\|_2^2 \leq (1+\epsilon)\min_{\text{degree $d$ polynomial $p$}} \|{p}-b\|_2^2.
	\end{align*}
		\vspace{-.5em}
\noindent Here $\|f\|_2^2$ denotes the average squared magnitude $\int_{\ell}^u f(x)^2 dx$ of a function $f$.
\end{theorem}
The problem of finding a polynomial approximation to a real-valued function that minimizes the average square error $\|{p}-b\|_2^2$ can be modeled as an active regression problem involving a matrix $\bv{A}$ with $d+1$ columns and an infinite number of rows. In fact, polynomial approximation is one of the primary applications of prior work on leverage score-based active learning methods \citep{CohenMigliorati:2017,universal_sampling}. Such methods require $O\left(d\log d + \frac{d}{\epsilon}\right)$ samples, so \Cref{thm:poly} is better by a $\log d$ factor. Since polynomial approximation is a representative problem where spatially-distributed samples are important, \Cref{thm:poly} provides theoretical justification for the strong performance of pivotal sampling in experiments.  Our proof is inspired by a result of \cite{KaneKarmalkarPrice:2017}, and relies on showing a tight relation between the leverage scores of the polynomial regression problem and the orthogonality measure of the Chebyshev polynomials on $[\ell, u]$.

\subsection{Related Work}\label{subsec:relatedwork}
The application of leverage score sampling to the agnostic active regression problem has received significant recent attention. Beyond the results discussed above, extensions of leverage score sampling have been studied for norms beyond $\ell_2$ \citep{pmlr-v134-chen21d,focs2022,MeyerMuscoMusco:2023,parulekar_et_al}, in the context where the sample space is infinite (i.e. $\bv{A}$ is an operator with infinite rows) \citep{tamas_2020,universal_sampling}, and for functions that involve non-linear transformations \citep{GHM22,MaiRaoMusco:2021,MunteanuSchwiegelshohnSohler:2018}.

Theoretical improvements on leverage score sampling have also been studied. Notable is a recent result that improves on the $O(d\log d + d/\epsilon)$ bound by a $\log d$ factor, showing that the active least squares regression problem can be solved with $O(d/\epsilon)$ samples \citep{CP19}. This is provably optimal. However, the algorithm in \citep{CP19} is complex, and appears to involve large constant factors: in our initial experiments, it did not empirically improve on independent leverage score sampling. In contrast, by \Cref{thm:poly}, our pivotal sampling method matches the theoretical sample complexity of \citep{CP19} for the special case of polynomial regression, but performs well in experiments (significantly better than independent leverage score sampling).
There have been a few other efforts to develop practical improvements on leverage score sampling.  Similar to our work, \cite{lev_volume_neurips2018} study a variant of volume sampling that matches the theoretical guarantees of leverage score sampling, but performs better experimentally.  However, this method does not explicitly take into account spatial-structure in the underlying regression problem. 
While the method from  \citep{lev_volume_neurips2018} does not quite fit our  \Cref{thm:main}  (e.g., it samples indices \emph{with} replacement) we expect similar methods could be analyzed as a special case of our result, as volume sampling induces a strongly Rayleigh distribution.

While pivotal sampling has not been studied in the context of agnostic active regression, it is widely used in other applications, and its negative dependence properties have been studied extensively. \citep{DJR05} proves that pivotal sampling satisfies the negative association (NA) property.  \citep{BBL08} introduced the notion of a strongly Rayleigh distribution and proved that it implies a stronger notion of conditional negative association (CNA), and \citep{BJ12} showed that pivotal sampling run with an arbitrary binary tree is strongly Rayleigh. It follows that the method satisfies CNA.
\citep{GWBW22} discusses an efficient algorithm for pivotal sampling by parallelization and careful manipulation of inclusion probabilities. Another variant of pivotal sampling that is \textit{spatially-aware} is proposed in \citep{GLS12}. Though their approach is out of the scope of our analysis as it involves randomness in the competition order used during sampling, our method is inspired by their work.

\subsection{Notation and Preliminaries}\label{sec:preliminaries}
\textbf{Notation.} We let $[n]$ denote $\{1, \cdots, n\}$. $\E[X]$ denotes the expectation of a random variable $X$. We use bold lower-case letters for vectors and bold upper-case letters for matrices. For a vector $\mathbf{z} \in \R^n$ with entries $z_1, \cdots, z_n$, $\lVert \mathbf{z} \rVert_2 = (\sum_{i=1}^n z_i^2)^{1/2}$ denotes the Euclidean norm of $\mathbf{z}$.
Given a matrix $\mathbf{A} \in \R^{n \times d}$, we let $\mathbf{a}_i$ denote the $i$-th row, and $a_{ij}$ denote the entry in the $i$-th row and $j$-th column. 

\noindent \textbf{Importance sampling.} All of the methods studied in this paper solve the active regression problem by collecting a single random sample of rows in $\bv{A}$ and corresponding entries in $\bv{b}$. We introduce a vector of binary random variables $\bm{\xi} = \{\xi_1, \cdots, \xi_n\}$, where $\xi_i$ is $1$ if $\bv{a}_i$ (and thus $b_i$) is selected, and $0$ otherwise. $\xi_1, \cdots, \xi_n$ will not necessarily be independent depending on our sampling method. 
Given a sampling method, let $p_i = \E[\xi_i]$ denote the marginal probability that row $i$ is selected. We return an approximate regression solution as follows: let $\tilde{\bv{A}}\in \R^{k\times d}$  contain $\bv{a}_i/\sqrt{p_i}$ for all $i$ such that $\xi_i = 1$, and similarly let $\tilde{\bv{b}}\in \R^k$ contain $b_i/\sqrt{p_i}$ for the same values of $i$. This scaling ensures that, for any fixed $\bv{x}$, $\E\|\tilde{\bv{A}}\bv{x} - \tilde{\bv{b}}\|_2^2 = \|{\bv{A}}\bv{x} - {\bv{b}}\|_2^2$. To solve the active regression problem, we return $\tilde{\mathbf{x}}^* = \argmin_{\mathbf{x} \in \R^d} \lVert \tilde{\mathbf{A}} \mathbf{x} - \tilde{\mathbf{b}} \rVert_2^2$. Computing $\tilde{\mathbf{x}}^*$ only requires querying $k$ target values in $\bv{b}$. 

\noindent \textbf{Leverage Score Sampling.} 
We consider methods that choose the marginal probabilities proportional to $\bv{A}$'s leverage scores. Specifically, our methods sample row $\bv{a}_i$ with marginal probability $\tilde{p}_i = \min(1, c_k\cdot \tau_i)$, where $c_k$ is chosen so that $\sum_{i=1}^n \tilde{p}_i = k$. Details of how to find $c_k$ are discussed in Appendix \ref{app:preprocess}. We note that we always have $c_k \geq k/d$ since $\sum_{i=1}^n p_i \leq \sum_{i=1}^n \frac{k}{d} \cdot\tau_i \leq \frac{k}{d}\cdot d = k$.

\section{Our Methods}\label{sec:methods}
In this section, we present our sampling scheme which consists of two steps; deterministically constructing a binary tree, and choosing samples by running the pivotal method on this tree. The pivotal method is described in Algorithm \ref{algo:pivotal}. It takes as input a binary tree with $n$ leaf nodes, each corresponding to a single index $i$ to be sampled. For each index, we also have an associated probability $\tilde{p}_i$. The algorithm collects a set of exactly $k$ samples $\mathcal{S}$ where $k = \sum_{i=1}^n \tilde{p}_i$. It does so by percolating up the tree and performing repeated head-to-head comparisons of the indices at sibling nodes in the tree. After each comparison, one node promotes to the parent node with updated inclusion probability, and the other node is determined to be sampled or not to be sampled.

\begin{figure}[t]
	\vspace{-2em}
\begin{algorithm}[H]\caption{Binary Tree Based Pivotal Sampling \citep{deville1998unequal}}\label{algo:pivotal}
    \begin{algorithmic}[1]
        \Require Depth $t$ full binary tree $T$ with $n$ leaves, inclusion probabilities $\{\tilde{p}_1, \cdots, \tilde{p}_n\}$ for each leaf.
        \Ensure Set of $k$ sampled indices $\mathcal{S}$.
        \State Initialize $\mathcal{S} = \emptyset$.
        \While{$T$ has at least two remaining children nodes}
            \State Select any pair of sibling nodes $S_1,S_2$ with parent $P$. Let $i,j$ be the indices stored at $S_1,S_2$.
                \If{$\tilde{p}_i + \tilde{p}_j \leq 1$}
                    \State With probability $\frac{\tilde{p}_i}{\tilde{p}_i + \tilde{p}_j}$, set $\tilde{p}_i \gets \tilde{p}_i + \tilde{p}_j$, $\tilde{p}_j \gets 0$. Store $i$ at $P$.
                    \State Otherwise, set $\tilde{p}_j \gets \tilde{p}_i + \tilde{p}_j$, $\tilde{p}_i \gets 0$. Store $j$ at $P$.
                \ElsIf{$\tilde{p}_i + \tilde{p}_j > 1$}
                    \State With probability $\frac{1-\tilde{p}_i}{2-\tilde{p}_i-\tilde{p}_j}$, set $\tilde{p}_i \gets \tilde{p}_i+\tilde{p}_j-1$, $\tilde{p}_j=1$. Store $i$ at $P$ and set $\mathcal{S} \gets \mathcal{S} \cup \{j\}$.
                    \State Otherwise, set $\tilde{p}_j \gets \tilde{p}_i+\tilde{p}_j-1$, $\tilde{p}_i \gets 1$. Store $j$ at $P$ and set $\mathcal{S} \gets \mathcal{S} \cup \{i\}$.
                \EndIf
                \State Remove $S_1,S_2$ from $T$.
        \EndWhile
        \Return $\mathcal{S}$
    \end{algorithmic}
\end{algorithm}
	\vspace{-2em}
\end{figure}

\begin{figure}[t]
	\vspace{-1.5em}
	\begin{algorithm}[H]\caption{Binary Tree Construction by Coordinate or PCA Splitting}\label{algo:btree}
		\begin{algorithmic}[1]
			\Require Matrix $\mathbf{X} \in \R^{n \times d'}$, split method $\in \{\text{PCA, coordinate}\}$, inclusion probabilities $\tilde{p}_1, \cdots, \tilde{p}_n$.
			\Ensure Binary tree  $\mathcal{T}$ where each leaf corresponds to a row in $\mathbf{X}$.
			\State Create a tree $T$ with a single root node. Assign set $\mathcal{R}$ to the root where $\mathcal{R} = \{i \in [n]; \tilde{p}_i<1\}$.
			\While{There exists a node in $T$ that holds set $\mathcal{K}$ such that $|\mathcal{K}|>1$}
			\State Select any such node $N$ and let $t$ be its level in the tree. Construct $\mathbf{X}_{(\mathcal{K})} \in \R^{|\mathcal{K}| \times d'}$.
			\If{split method $=$ PCA}
			\State Sort $\mathbf{X}_{(\mathcal{K})}$ according to the direction of the maximum variance.
			\ElsIf{split method $=$ coordinate}
			\State Sort $\mathbf{X}_{(\mathcal{K})}$ according to values in its $((t \,\, \mathrm{mod} \,\, d')+1)$-th column.
			\EndIf
			\State Create a left child of $N$. Assign to it all indices associated with the first $\lfloor\frac{|\mathcal{K}|}{2}\rfloor$ rows of $\mathbf{X}_{(\mathcal{K})}$.
			\State Create a right child of $N$. Assign to it the all remaining indices in $\mathbf{X}_{(\mathcal{K})}$. Delete $\mathcal{K}$ from $N$.
			\EndWhile
			\Return $T$
		\end{algorithmic}
	\end{algorithm}
	\vspace{-2.5em}
\end{figure}

It can be checked that, after running Algorithm \ref{algo:pivotal}, index $i$ is always sampled with probability $\tilde{p}_i$, regardless of the choice of $T$. However, the samples collected by the pivotal method are not independent, but rather negatively correlated: siblings in $T$ are unlikely to both be sampled, and in general, the events that close neighbors in the tree are both sampled are negatively correlated. In particular, if index $i$ could at some point compete with an index $j$ in the pivotal process, the chance of selecting $j$ decreases if we condition on $i$ being selected. We take advantage of this property to generate spatially distributed samples by constructing a binary tree that matches the underlying geometry of our data. In particular, assume we are given a set of points $\bv{X}\in \R^{n\times d'}$. $\bv{X}$  will eventually be used to construct a regression matrix $\bv{A}\in \R^{n\times d}$ via feature transformation (e.g. by adding polynomial features). However, we construct the sampling tree using  $\bv{X}$ alone.

\begin{figure}[t]
	\vspace{-0.9em}
    \centering
    \includegraphics[width=.7\linewidth]{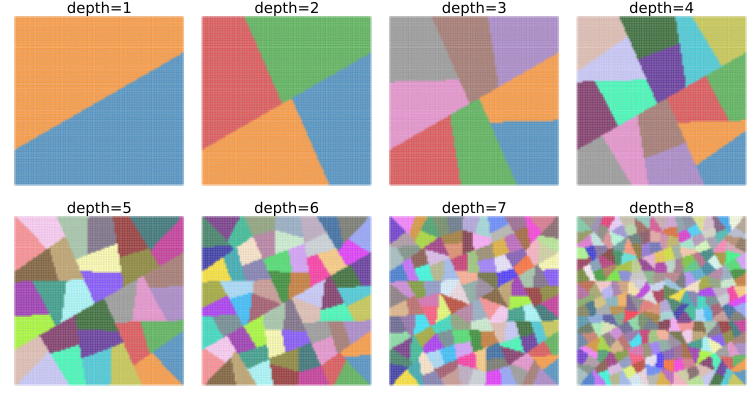}
        \vspace{-.9em}
    \caption{Visualization of a binary tree constructed via our Algorithm \ref{algo:btree} using the PCA method for a matrix $\bv{X}\in\R^{n\times 2}$ containing points on a uniform square grid. For each depth, data points are given the same color if they compose a subtree with root at that depth. As we can see, the method produces uniform recursive spatial partitions, which encourage spatially separated samples.}
    \label{fig:sec3:binarytree}
    \vspace{-1em}
\end{figure}

Our tree construction method is given as Algorithm \ref{algo:btree}. $\bv{X}_{\mathcal{K}}$ denotes the subset of rows of $\bv{X}$ with indices in the set $\mathcal{K}$. First, the algorithm eliminates all data points with inclusion probability $\tilde{p}_i = 1$. Next, it recursively partitions the remaining data points into two subgroups of the same size until all the subgroups have only one data point. Our two methods, PCA-based and coordinate-wise, only differ in how to partition. The PCA-based method performs principal component analysis to find the direction of the maximum variance and splits the space by a hyperplane orthogonal to the direction so that the numbers of data points on both sides are equal. The coordinate-wise version takes a coordinate (corresponding to a column in $\bv{X}$) in cyclic order and divides the space by a hyperplane orthogonal to the chosen coordinate.
An illustration of the PCA-based binary tree construction run on a fine uniform grid of data points in $\R^2$ is shown in Figure \ref{fig:sec3:binarytree}. Note that our tree construction method ensures the the number of indices assigned to each subgroup (color) at each level is equal to with $\pm 1$ point. As such, we end of with an even partition of data points into spatially correlated sets. 
Two indices will be more negatively correlated if they lie in the same set at a higher depth number. 

\section{Theoretical Analysis}
\label{sec:analysis}
As will be shown experimentally, when using probabilities $\tilde{p}_1, \ldots, \tilde{p}_n$ proportional to the statistical leverage scores of $\bv{A}$, our tree-based pivotal methods significantly outperform its Bernoulli counterpart for active regression. We provide two results theoretically justifying this operation. We first show that, no matter what our original data matrix $\bv{X}$ is, and what feature transformation is used to construct $\bv{A}$, our methods never perform 
 \emph{worse} than Bernoulli sampling. In particular, they match the $O(d\log d + d/\epsilon)$ sample complexity of independent leverage score sampling. 
  This result is stated as Theorem \ref{thm:main}. Its proof is given to Appendix \ref{sec:proofs}, but we outline our main approach here. 

Following existing proofs for independent random sampling (e.g. \citep{Woodruff:2014} Theorem \ref{thm:main} requires two main ingredients: a subspace embedding result, and an approximate matrix-vector multiplication result. In particular, let $\bv{U}\in \R^{n\times d}$ be any orthogonal span for the columns of $\bv{A}$. Let $\bv{S} \in \R^{k\times n}$ be a subsampling matrix that contains a row for every index $i$ selected by our sampling scheme, which has value ${1}/{\sqrt{\tilde{p}_i}}$ at entry $i$, and is $0$ everywhere else. So, in the notation of  Theorem \ref{thm:main}, $\tilde{\bv{A}} = \bv{S}\bv{A}$ and $\tilde{\bv{b}} = \bv{S}\bv{b}$. To prove the theorem, it suffices to show that with high probability, 
\begin{enumerate}
\item \textbf{Subspace Embedding:} For all $\bv{x}\in \R^d$, $\frac{1}{2}\|\bv{x}\|_2 \leq \|\bv{S}\bv{U}\bv{x}\|_2 \leq 1.5\|\bv{x}\|_2$.
\item \textbf{Approximate Matrix-Vector Multiplication}: $\| \mathbf{U}^T \mathbf{S}^T \mathbf{S} (\mathbf{b} - \mathbf{A} \mathbf{x}^*) \|_2^2 \leq \epsilon \| \mathbf{b} - \mathbf{A} \mathbf{x}^* \|_2^2$.
\end{enumerate}
The first property is equivalent to $\|\bv{U}^T\bv{S}^T\bv{S}\bv{U} - \bv{I}\|_2 \leq 1/2$. I.e., after subsampling, $\bv{U}$ should remain nearly orthogonal. The second property requires that after subsampling with $\bv{S}$, the optimal residual, $\bv{b} - \bv{A}\bv{x}^*$, should have small product with $\bv{U}$. Note that without subsampling,  $\|\bv{U}^T(\mathbf{b} - \mathbf{A} \mathbf{x}^*)\|_2^2 = 0$. 

We show that both of the above bounds can be established for any sampling method that 1) samples index $i$ with marginal probability proportional to its leverage score 2) is homogeneous, meaning that it takes a fixed number of samples $k$, and 3) produces a distribution over binary vectors satisfying the following property:
\begin{restatable}[One-sided $\ell_{\infty}$-independence]{definition}{linfinityind}\label{def:l-infty}
Let $\xi_1, \cdots, \xi_n \in \{0,1\}^n$ be random variables with joint distribution $\mu$. Let $\mathcal{S}\subseteq [n]$ and let $i,j \in [n] \backslash \mathcal{S}$. Define the one-sided influence matrix $\mathcal{I}_{\mu}^{\mathcal{S}}$ as:\vspace{-.1em}
    \begin{align*}\begin{split}
        \mathcal{I}_{\mu}^{\mathcal{S}}(i, j) = \Pr_{\mu} [\xi_j=1 | \xi_i = 1 \wedge \xi_{\ell} = 1 \forall \ell \in \mathcal{S}] - \Pr_{\mu} [\xi_j=1 | \xi_{\ell} = 1 \forall \ell \in \mathcal{S}]
    \end{split}\end{align*}
Let $\lVert \mathcal{I}_{\mu}^{\mathcal{S}} \rVert_{\infty} = \max_{i \in [n]} \sum_{j \in [n]} | \mathcal{I}_{\mu}^{\mathcal{S}}(i, j) |$. $\mu$ is one-sided $\ell_{\infty}$-independent with param. $D_{\text{inf}}$ if, for all $\mathcal{S} \subset [n]$, $\lVert \mathcal{I}_{\mu}^{\mathcal{S}} \rVert_{\infty} \leq D_{\text{inf}}$.
    Note that if $\xi_1, \ldots, \xi_n$ are truly independent, we have $D_{\text{inf}} = 1$.
\end{restatable}
To prove \Cref{thm:main}, our required subspace embedding property follows immediately from recently established matrix Chernoff-type bounds for sums of random matrices involving one-sided $\ell_{\infty}$-independent random variables in \citep{KKS22}. In fact, this work is what inspired us to consider this property, as it is the minimal condition under which such bounds are known to hold. Our main theoretical contribution is thus to prove the matrix-vector multiplication property. We do this by generalizing the approach of \citep{DrineasKannanMahoney:2006} (which holds for independent random samples) to any distribution that satisfies one-sided $\ell_{\infty}$-independence. 

\subsection{Improved Bounds for Polynomial Regression}
We do not believe that \Cref{thm:main} can be strengthened in general, as spatially well-spread samples may not be valuable in all settings. For example, in the case when $\bv{A} = \bv{X}$,  the points in $\bv{X}$ live in $d$ dimensional space, and we are only collecting $O(d\log d)$ samples, so intuitively any set of points is well-spread. However, we can show that better spatial distribution {does} offer \emph{provably} better bounds in some settings where $\bv{X}$ is low-dimensional and $\bv{A}$ is a high dimensional feature transformation.
 In particular, consider the case when $\bv{X}$ is a (quasi) matrix containing just one column, with an infinite number of rows corresponding to every point $t$ in an interval $[\ell, u]$. Every row of $\bv{A}$ is a polynomial feature transformation, meaning of the form $\bv{a}_t = [1, t, t^2, \ldots, t^d]$ for degree $d$. Now, consider a target vector $\bv{b}$ that can be indexed by real numbers in $[\ell, u]$. Solving the infinite dimensional regression problem $\min_{\bv{x}} \|\bv{A}\bv{x} - \bv{b}\|_2^2 = \min_{\bv{x}} \int_{\ell}^u (\bv{a}_t^T\bv{x} - \bv{b}_t)^2 dt$ is equivalent to finding the best polynomial approximation to $\bv{b}$ in the $\ell_2$ norm, a well studied application of leverage score sampling \cite{CohenMigliorati:2017}. For this problem, we  show that pivotal sampling obtains a sample complexity of $O(d/\epsilon)$, improving on independent leverage score sampling by a $\log d$ factor.
 
This result is stated as Theorem \ref{thm:poly} and proven in Appendix \ref{app:poly_proof}. Importantly, we note that the dependence on $d\log d$ in the general active regression analysis comes from the proof of the subspace embedding guarantee -- the required approximate matrix-multiplication guarantee already follows with $O(d/\epsilon)$ samples. Our proof eliminates the $\log d$ by avoiding the use of a matrix Chernoff bound entirely. Instead, by taking advantage of connections between leverage scores of the polynomial regression problem and the orthogonality measure of the Chebyshev polynomials, we directly use tools from polynomial approximation theory to prove a subspace embedding bound that holds \emph{deterministically} with just $O(d)$ samples. Our approach is similar to a recently result of \cite{KaneKarmalkarPrice:2017}, which also obtains  $O(d/\epsilon)$ sample complexity for active degree-$d$ polynomial regression, albeit with a sampling method not based on leverage scores.

\section{Experiments}\label{sec:experiments}
We experimentally evaluate our pivotal sampling methods on active regression problems with low-dimensional structure. The benefits of leverage score sampling over uniform sampling for such problems has already been established (see e.g. \citep{HamptonDoostan:2015} or \citep{GHM22}) and we provide additional evidence in Appendix \ref{sec:appendix:experiments}. So, we focus on comparing our pivotal methods to the widely used baseline of \emph{independent} Bernoulli leverage score sampling. 

\begin{figure}[b]
	\vspace{-1em}
	\centering
	\begin{subfigure}[b]{0.245\textwidth}
		\centering
		\includegraphics[width=\linewidth]{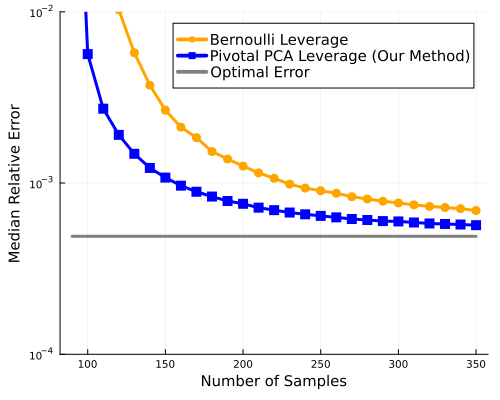}
		\vspace{-1em}
		\caption{Damped Harmonic Oscillator, degree $p=12$.}    
	\end{subfigure}
	\hfill
	\centering
	\begin{subfigure}[b]{0.245\textwidth}
		\centering
		\includegraphics[width=\linewidth]{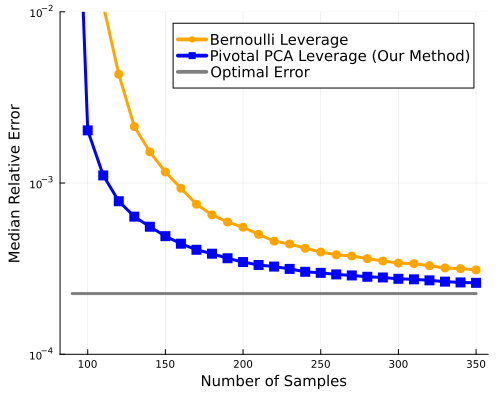}
		\vspace{-1em}
		\caption{Heat Equation, degree $p=12$.}    
	\end{subfigure}
 \begin{subfigure}[b]{0.245\textwidth}
		\centering
		\includegraphics[width=\linewidth]{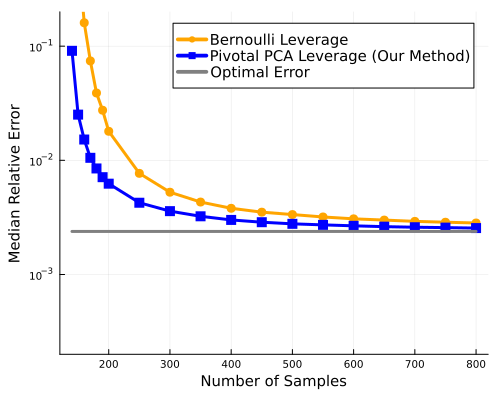}
		\vspace{-1em}
		\caption{Surface Reaction, degree $p=15$.}    
	\end{subfigure}
	\hfill
	\centering
	\begin{subfigure}[b]{0.245\textwidth}
		\centering
		\includegraphics[width=\linewidth]{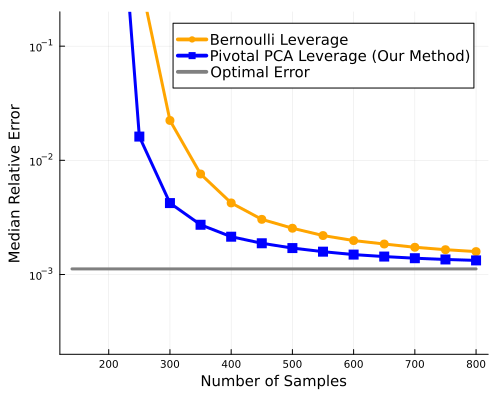}
		\vspace{-1em}
		\caption{Surface Reaction, degree $p=20$.}    
	\end{subfigure}
 	\vspace{-1em}
	\caption{
 Results for active polynomial regression for the damped harmonic oscillator QoI, the heat equation QoI, and the surface reaction model with polynomials varying degree. Our leverage-score based pivotal method outperforms standard Bernoulli leverage score sampling, suggesting the benefits of spatially-aware sampling.}
	\label{fig:sec5:result}
	\vspace{-.5em}
\end{figure}

\noindent \textbf{Test Problems.} We consider several test problems inspired by applications to parametric PDEs. In these problems, we are given a differential equation with parameters, and seek to compute some observable quantity of interest (QoI) for different choices of parameters. This can be done directly by solving the PDE, but doing so is computationally expensive. Instead, the goal is to use a small number of solutions to the PDE to fit a surrogate model (in our case, a low-degree polynomial) that approximates the QoI well, either over a given parameter range, or on average for parameters drawn from a distribution (e.g., Gaussian). The problem is naturally an active regression problem because we can choose exactly what parameters to solve the PDE for. 
The first equation we consider models the displacement $x$ of a damped harmonic oscillator with a sinusoidal force over time $t$. 
\begin{align*}\begin{split}
    \frac{d^2 x}{d t^2}(t) + c \frac{dx}{dt}(t) + k x(t) = f \cos(\omega t), \quad x(0)=x_0, \quad \frac{dx}{dt}(0)=x_1.
\end{split}\end{align*}
The equation has four parameters; damping coefficient $c$, spring constant $k$, forcing amplitude $f$, and frequency $\omega$. As a QoI, we consider the maximum oscillator displacement after $20$ seconds. We fix $c,f = 0.5$, and seek to approximate this displacement over the range domain $k \times \omega = [1,3] \times [0,2]$.

We also consider the heat equation for values of $x \in [0,1]$ with a time-dependent boundary equation and sinusoidal initial condition parameterized by a frequency $\omega$. The heat equation that we consider describes the temperature $f(x,t)$ by the partial differential equation.
\begin{align*}
    \pi \frac{\partial f}{\partial t} = \frac{\partial^2 f}{\partial x^2}, \quad f(0,t)=0, \quad f(x,0)=\sin(\omega \pi x), \quad \pi e^{-t} + \frac{\partial f(1, t)}{\partial t} = 0.
\end{align*}
As a QoI, we estimate the maximum temperature over all values of $x$ for $t \in [0,3]$ and  $\omega \in [0,5]$.

\begin{figure}[t]
    \vspace{-1em}
    \centering
    \begin{subfigure}[t]{0.32\textwidth}
        \centering
        \includegraphics[width=.8\linewidth]{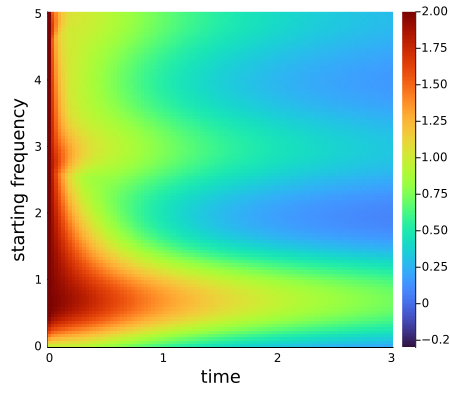}
                \vspace{-.5em}
        \caption{Target Function.}    
    \end{subfigure}
    \hfill
    \centering
    \begin{subfigure}[t]{0.32\textwidth}
        \centering
        \includegraphics[width=.8\linewidth]{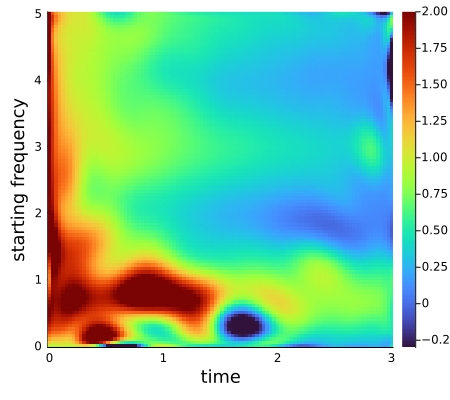}
                \vspace{-.5em}
        \caption{Bernoulli Sampling.}    
    \end{subfigure}
    \hfill
    \centering
    \begin{subfigure}[t]{0.32\textwidth}
        \centering
        \includegraphics[width=.8\linewidth]{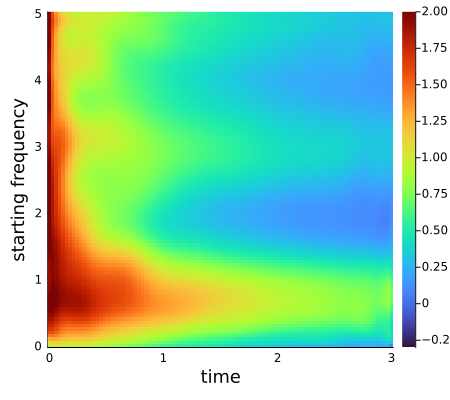}
                \vspace{-.5em}
        \caption{Pivotal Sampling (our method).}    
    \end{subfigure}
\vspace{-.5em}
    \caption{
    Polynomial approximation to the maximum temperature of a heat diffusion problem, as a function of time and starting condition. (a) is the target value and both (b) and (c) draw $240$ samples using the leverage score and perform polynomial regression of degree $20$.  However, (b) uses Bernoulli sampling while (c) employs our PCA-based pivotal sampling. 
}
    \label{fig:sec5:heat}
    \vspace{-1em}
\end{figure}

\noindent \textbf{Data Matrix.} For both problems, we construct $\bv{A}$ by uniformly selecting $n = 10^5$ data points in the 2-dimensional parameter range of interest. We then add all polynomial features of degree $p=12$ as discussed in Section \ref{subsec:contributions}. We compute sampled entries from the target vector $\bv{b}$ using standard MATLAB routines. Results comparing our PCA-based pivotal method and Bernoulli leverage score sampling are show in Figure \ref{fig:sec5:result}. We report median normalized error ${\lVert \mathbf{A} \tilde{\mathbf{x}}^* - \mathbf{b} \rVert_2^2}/{\lVert \mathbf{b} \rVert_2^2}$ after $1000$ trials. 
By drawing more samples from $\bv{b}$, the errors of all methods eventually converge to the optimal error  ${\lVert \mathbf{A} {\mathbf{x}}^* - \mathbf{b} \rVert_2^2}/{\lVert \mathbf{b} \rVert_2^2}$, but clearly the pivotal method requires less samples to achieve a given level of accuracy, confirming the benefits of spatially-aware sampling. We also visualize results for the damped harmonic oscillator in Figure \ref{fig:sec1:comparison}, showing approximations obtained with $250$ samples. Visualizations for the heat equation are given in Figure \ref{fig:sec5:heat}. For both targets, one can directly see that pivotal sampling improves the performance over Bernoulli sampling. 

We also consider a chemical surface coverage problem from \citep{HamptonDoostan:2015}. Details are relegated to \Cref{sec:appendix:experiments}. Again, we vary two parameters and seek to fit a surrogate model using a small number of example pairs of parameters. Instead of a uniform distribution, for this problem, the rows of $\bv{X}$ are drawn from a Gaussian distribution with $0$-mean and $7.5$ standard deviation. We construct $\bv{A}$ using polynomial features of varying degrees. Convergence results are shown in Figure \ref{fig:sec5:result} and the fit visualized for data near the origin in Figure \ref{fig:sec7.5:surface}. This is a challenging problem since the target function has sharp threshold behavior that is difficult to approximate with a polynomial. However, our pivotal based leverage score sampling performs well, providing a better fit than Bernoulli sampling. Additional experiments, including on 3D problems are reported in Appendix \ref{sec:appendix:experiments}.

\begin{figure}[t]
	\centering
	\begin{subfigure}[t]{0.24\textwidth}
		\centering
		\includegraphics[width=\linewidth]{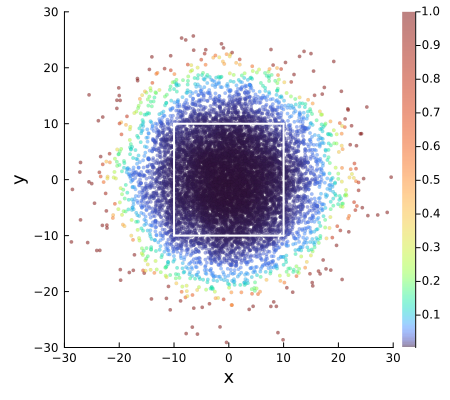}
		\vspace{-1.5em}
		\caption{Leverage Score.}    
	\end{subfigure}
	\centering
	\begin{subfigure}[t]{0.24\textwidth}
		\centering
		\includegraphics[width=\linewidth]{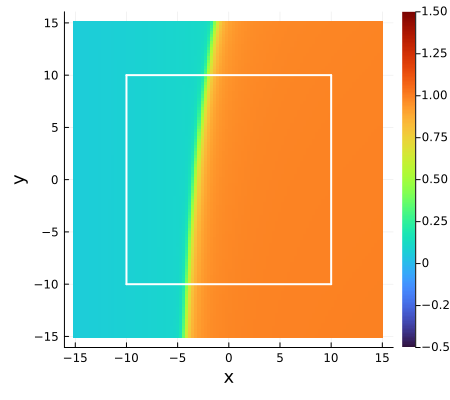}
		\vspace{-1.5em}
		\caption{Target Function.}    
	\end{subfigure}
	\hfill
	\centering
	\begin{subfigure}[t]{0.24\textwidth}
		\centering
		\includegraphics[width=\linewidth]{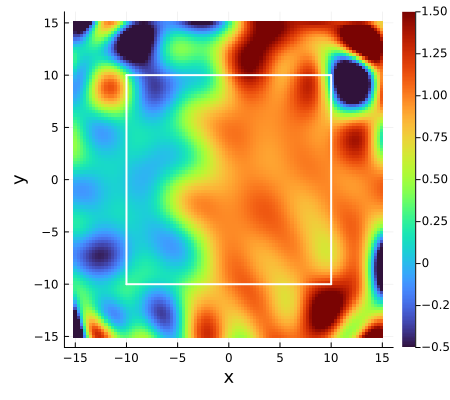}
		\vspace{-1.5em}
		\caption{Bernoulli Sampling.}    
	\end{subfigure}
	\centering
	\begin{subfigure}[t]{0.24\textwidth}
		\centering
		\includegraphics[width=\linewidth]{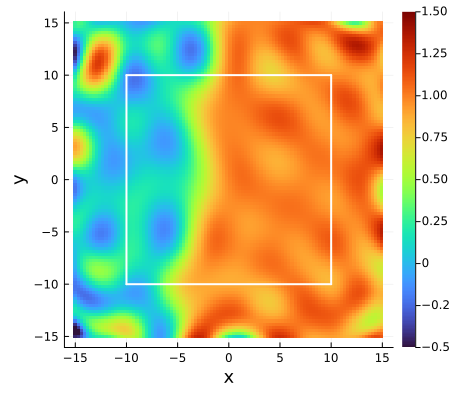}
		\vspace{-1.5em}
		\caption{Pivotal Sampling.}    
	\end{subfigure}
 \vspace{-.5em}
	\caption{Approximation of the surface reaction model when $425$ samples are used to fit a degree $25$ polynomial. To help readers focus on the area near the origin, we draw $[-10,10]^2$ box.}
	\label{fig:sec7.5:surface}
	\vspace{-1.75em}
\end{figure}

\section{Conclusion and Future Work}\label{sec:conclusion}
\vspace{-.25em}
In this paper, we introduce a \textit{spatially-aware} pivotal sampling method and empirically demonstrate its effectiveness for active linear regression for spatially-structured problems. We prove a general theorem that can be used to analyze the number of samples required in the agnostic setting for any similar method that samples with marginal probabilities proportional to the leverage scores using a distribution that satisfies one-sided $\ell_{\infty}$ independence. We provide a stronger bound for the important special case of polynomial regression, showing that our method can obtain a sample complexity of $O(d/\epsilon)$, removing a $\log(d)$ factor from independent leverage score sampling. This result provides initial theoretical evidence for the strong performance of pivotal sampling in practice. Extending it to other spatially structured function classes would be an interesting direction for future work. 

\section*{Acknowledgements} Xiaoou Cheng and Jonathan Weare were supported by NSF award 2054306. Christopher Musco was supported by DOE award DE-SC0022266 and NSF award 2045590. We would like to thank NYU IT for the use of the Greene computing cluster.

\bibliographystyle{plainnat}
\bibliography{references}

\appendix
\section{Probability Pre-processing}
\label{app:preprocess}
As discussed in Section \ref{sec:preliminaries}, our sampling methods require computing probabilities $\tilde{p}_1, \ldots, \tilde{p}_n$ where $\tilde{p}_i = \min(1, c_k\cdot \tau_i)$ for some fixed constant $c_k$ chosen so that $\sum_{i=1}^n \tilde{p}_i = k$. We can find such probabilities using a simple iterative method, which takes as input initial ``probabilities''
$
    p_i = \frac{k}{d}\tau_i
$
that are proportional to the leverage scores $\tau_1, \ldots, \tau_n$ of $\bv{A}$. Note that $p_i$ could be larger than $1$ if $k > d$. Pseudocode for how to adjust these probabilities is included in Algorithm \ref{algo:ceiling}.
\begin{algorithm}[h!]\caption{Probability Ceiling}\label{algo:ceiling}
    \begin{algorithmic}[1]
        \Require Number of samples to choose $k$, inclusion probabilities $\{p_1, \cdots, p_n\}$.
        \State Set $\tilde{p}_i = p_i$ for all $i \in [n]$.
        \While{$\mathcal{S} = \{i \in [n]; \tilde{p}_i > 1\}$ is not empty}
            \State $\tilde{p}_i \gets 1$ for all $i \in \mathcal{S}$.
            \State $\mathcal{D} = \{i \in [n]; \tilde{p}_i = 1\}$, $\mathcal{R} = \{i \in [n]; \tilde{p}_i < 1\}$.
            \State $\tilde{p}_i \gets \frac{k - |\mathcal{D}|}{\sum_{i \in \mathcal{R}} \tilde{p}_i} \tilde{p}_i$ for all $i \in \mathcal{R}$.
        \EndWhile 
        \Return $\{\tilde{p}_1, \cdots, \tilde{p}_n\}$
    \end{algorithmic}
\end{algorithm}

To see that the method returns $\tilde{p}_i$ with the desired properties, first note that $\frac{k - |\mathcal{D}|}{\sum_{i \in \mathcal{R}} \tilde{p}_i}$ is always greater than $1$. In particular, at the beginning of the while loop, we always have the invariant that $\sum_{i=1}^n \tilde{p}_i = k$. Accordingly, $\sum_{i\in \mathcal{R}}\tilde{p}_i = k - \sum_{i\in \mathcal{D}}\tilde{p}_i < k - |\mathcal{D}|$. As a result, at the end of the algorithm, $\tilde{p}_i$ certainly equals $\min(1, c\cdot \tau_i)$ for some constant $c \geq 1$. And in fact, it must be that $c = c_k$ by the invariant that $\sum_{i=1}^n \tilde{p}_i = k$.

The main loop in Algorithm \ref{algo:ceiling} terminates as soon as $\mathcal{S}$ is empty, which happens after at most $k$ steps. In practice however, the method usually converges in 2 or 3 iterations. Nevertheless, our primary objective is to minimize the number of samples required for an accurate regression solution -- we do not consider runtime costs in detail.

\section{Proof of \Cref{thm:main}}\label{sec:proofs}
In this section, we prove our first theoretical result, which holds for any row sampling distribution whose marginal probabilities are proportional to the leverage scores of $\bv{A}$, and which satisfies a few additional conditions. In particular, we require that the distribution is both $k$-homogeneous and is $\ell_\infty$-independent with constant parameter $D$. We define the first requirement below, and recall the definition of $\ell_\infty$-independence from Section \ref{sec:analysis}. For both definitions, we view our subsampling process as a method that randomly selects a binary vector $\bm{\xi} = \{\xi_1, \cdots, \xi_n\}$ from $\{0,1\}^n$. Entries of $1$ in the vector correspond to rows that are sampled from $\bv{A}$ (and thus labels that are observed in $\bv{b}$). Under this notation, our requirement on marginal probabilities is equivalent to requiring that $\E[\xi_i] = \min(1,c \cdot \tau_i) = \tilde{p}_i$ for some fixed constant $c$, where $\tau_i$ is the $i$-th leverage score of $\bv{A}$.

\begin{definition}[Homogeneity]\label{def:homogeneity}
    A distribution over binary vectors $\bm{\xi} = \{\xi_1, \cdots, \xi_n\}$ is \textit{$k$-homogeneous} if all possible realizations of  $\bm{\xi}$ contain exactly $k$ ones.
\end{definition}

\linfinityind*

With these definitions in place, we are ready for our proof:
\begin{proof}[Proof of Theorem \ref{thm:main}]
As outlined in Section \ref{sec:analysis}, we need to establish two results: a subspace embedding guarantee and an approximate matrix-vector multiplication result. We prove these separately below, then combine the results to complete the proof of the theorem.

\noindent{\textbf{Subspace Embedding.}}
Let $\bv{S}$ be a subsampling matrix corresponding to a vector $\bs{\xi}$ selected from a distribution with the properties above. That is, $\bv{S} \in \R^{k\times n}$ contains a row for every index $i$ where $\xi_i = 1$. That row has value ${1}/{\sqrt{\tilde{p}_i}}$ at entry $i$, and is $0$ everywhere else. Let $\bv{U}\in \R^{n\times d}$ be an orthonormal span for $\bv{A}$'s columns. We will show that, if $k = O\left(\frac{d \log (d / \delta)}{\alpha^2}\cdot D_{\text{inf}}^2\right)$, with prob. $1 - \delta$,
\begin{align}\begin{split}\label{subspacegoal}
    \lVert \mathbf{U}^T \mathbf{S}^T \mathbf{S} \mathbf{U} - \mathbf{I} \rVert_2 \leq \alpha.
\end{split}\end{align}
Note that, since $\left| \lVert \mathbf{S} \mathbf{U} \mathbf{x} \rVert_2^2 - \lVert \mathbf{U} \mathbf{x} \rVert_2^2 \right|
    = \left| \mathbf{x}^T \left(\mathbf{U}^T \mathbf{S}^T \mathbf{S} \mathbf{U} - \mathbf{I}\right) \mathbf{x} \right|
    \leq \lVert \mathbf{U}^T \mathbf{S}^T \mathbf{S} \mathbf{U} - \mathbf{I} \rVert_2 \lVert \mathbf{x} \rVert_2^2$, this is equivalent to showing that, for all $\bv{x}\in \R^d$,
\begin{align*}\begin{split}
    (1-\alpha) \lVert \mathbf{U} \mathbf{x} \rVert_2^2 \leq \lVert \mathbf{S} \mathbf{U} \mathbf{x} \rVert_2^2 \leq (1 + \alpha) \lVert \mathbf{U} \mathbf{x} \rVert_2^2.
\end{split}\end{align*}

We establish \eqref{subspacegoal} using the following matrix Chernoff bound from \citep{KKS22}:
\begin{lemma}[Matrix Chernoff for $\ell_{\infty}$-independent Distributions]\label{lemma:chernoff}
    Let $\xi_1, \cdots, \xi_n$ be binary random variables with joint distribution $\mu$ that is $z$-homogeneous for any positive integer $z$ and $\ell_{\infty}$-independent with parameter $D_{\text{inf}}$. Let $\mathbf{Y}_1, \ldots, \mathbf{Y}_n \in \R^{d \times d}$ be positive semidefinite (PSD) matrices with spectral norm $\|\mathbf{Y}_i\|_2 \leq R$ for some $R>0$. Let $\mu_{\max} = \lambda_{\max}(\mathbb{E}_{\bm{\xi} \sim \mu}[\sum_{i=1}^n \xi_i \bv{Y}_i])$ and $\mu_{\min} = \lambda_{\min}(\mathbb{E}_{\bm{\xi} \sim \mu}[\sum_{i=1}^n \xi_i \bv{Y}_i])$. Then, for any $\alpha \in (0,1)$ and a fixed constant $c$,
    \begin{align*}\begin{split}
        \Pr \left[\lambda_{\max}\left( \sum_{i=1}^n \xi_i \mathbf{Y}_i \right) \right. &  \geq (1+\alpha) \mu_{\max}\Biggr] \leq d \exp \left( - \frac{\alpha^2 \mu_{\max}}{cRD_{\text{inf}}^2} \right),\\
         \Pr \left[\lambda_{\min}\left( \sum_{i=1}^n \xi_i \mathbf{Y}_i \right)\right. & \leq (1-\alpha) \mu_{\min}\Biggr] \leq d \exp \left( - \frac{\alpha^2 \mu_{\min}}{cRD_{\text{inf}}^2} \right).
    \end{split}\end{align*}
\end{lemma}

Let $\mathcal{D} = \{i \in [n]: \tilde{p}_i = 1\}$ and $\mathcal{R} = \{i \in [n]: \tilde{p}_i < 1\}$. 
Letting $\bv{u}_i$ denote the $i$-th row of $\bv{U}$,
$\mathbf{U}^T \mathbf{S}^T \mathbf{S} \mathbf{U}$ can be decomposed as:
\begin{align*}\begin{split}
    \mathbf{U}^T \mathbf{S}^T \mathbf{S} \mathbf{U} = \sum_{i\in \mathcal{R}} \xi_i \frac{\mathbf{u}_i \mathbf{u}_i^T}{\tilde{p}_i} + \sum_{i\in \mathcal{D}} \mathbf{u}_i \mathbf{u}_i^T.
\end{split}\end{align*}
Recall that $\tilde{p}_i = \min(1,c\cdot\tau_i)$ for some constant $c$, and in fact, since $\sum_{i=1}^n \tilde{p}_i = k$ and $\sum_{i=1}^n \tau_i \leq d$, it must be that $c \geq \frac{k}{d}$. Define $p_i = \frac{k}{d}\tau_i$.
For each $i \in \mathcal{D}$, let $m_i$ be an integer such that $m_i \geq p_i$. For $i \in \mathcal{D}$, $p_i = \frac{k}{d}\tau_i \geq 1$, so $m_i \geq 1$. Additionally let $\xi_i^1, \ldots, \xi_i^{m_i} = \xi_i$ be variables that are deterministically $1$. Then we can trivially split up the second sum above so that:
\begin{align}\label{eq:split_sum}
\begin{split}
    \mathbf{U}^T \mathbf{S}^T \mathbf{S} \mathbf{U} = \sum_{i\in \mathcal{R}} \xi_i \frac{\mathbf{u}_i \mathbf{u}_i^T}{\tilde{p}_i} + \sum_{i\in \mathcal{D}} \sum_{j=1}^{m_i} \xi_i^j \frac{\mathbf{u}_i \mathbf{u}_i^T}{m_i}.
\end{split}\end{align}
It is easy to see that $\{\xi_i : i\in \mathcal{R}\} \cup \{\xi_i^j : i\in \mathcal{D}, j\in [m_i]\}$ are $\ell_{\infty}$-independent random variables with parameter $D_{\text{inf}}$ and form a distribution that is $\bar{k}$ homogeneous for $\bar{k} = k + \sum_{i\in \mathcal{D}} (m_i - 1)$. So, noting that $\bv{u}_i\bv{u}_i^T$ is PSD, we can apply Lemma \ref{lemma:chernoff} to \eqref{eq:split_sum}. We are left to bound $R$ and $\mu_{\max}$. 

To bound $R$, for $i \in \mathcal{R}$, let $\bv{Y}_i = \frac{\mathbf{u}_i \mathbf{u}_i^T}{\tilde{p}_i}$. For such $i$, we have that $\tilde{p}_i \geq p_i$. So,
\begin{align*}\begin{split}
    \lVert \mathbf{Y}_i \rVert_2 = \frac{\lVert \mathbf{u}_i \mathbf{u}_i^T \rVert_2}{\tilde{p}_i} = \frac{\lVert \mathbf{u}_i \rVert_2^2}{\tilde{p}_i} = \frac{\tau_i}{\tilde{p}_i} \leq \frac{\tau_i}{p_i} = \frac{d}{k}.
\end{split}\end{align*}
I.e., for all $i \in \mathcal{R}$, $\|\bv{Y}_i\|_2 \leq R$ for $R = \frac{d}{k}$. Additionally, for $i \in \mathcal{D}$, let  $\bv{Y}_i = \frac{\mathbf{u}_i \mathbf{u}_i^T}{m_i}$. Since we chose $m_i \geq p_i$, by the same argument, we have that $\lVert \mathbf{Y}_i \rVert_2 \leq R$ for $R = \frac{d}{k}$.

Next, we bound $\mu_{\max}$ by simply noting that
\begin{align*}\begin{split}
    \mathbb{E}\left[ \bv{U}^T\bv{S}^T\bv{S}\bv{U}\right] = \sum_{i\in \mathcal{R}} \mathbb{E}[\xi_i] \frac{\mathbf{u}_i \mathbf{u}_i^T}{\tilde{p}_i} + \sum_{i\in \mathcal{D}} \mathbf{u}_i \mathbf{u}_i^T = \sum_{i\in \mathcal{R}}\bv{u}_i\bv{u}_i^T + \sum_{i\in \mathcal{D}} \mathbf{u}_i \mathbf{u}_i^T = \bv{I}, 
\end{split}\end{align*}
where $\bv{I}$ denotes the $d\times d$ identity matrix. It follows that  $\mu_{\max} = \mu_{\min} = 1$.

Plugging into Lemma \ref{lemma:chernoff} and applying a union bound, we obtain
\begin{align*}\begin{split}
    \Pr \left[ \lVert \mathbf{U}^T \mathbf{S}^T \mathbf{S} \mathbf{U} - \mathbf{I} \rVert_2 \geq \alpha \right] &\leq d \exp \left( - \frac{k \alpha^2}{c'D_{\text{inf}}^2d} \right),
\end{split}\end{align*}
where $c'$ is a fixed constant. 
Setting $k = O(\frac{d \log (d / \delta)}{\alpha^2}\cdot D_{\text{inf}}^2)$, we have $\lVert \mathbf{U}^T \mathbf{S}^T \mathbf{S} \mathbf{U} - \mathbf{I} \rVert_2 \leq \alpha$ with probability at least $1 - \delta$. This establishes \eqref{subspacegoal}.

\noindent{\textbf{Approximate Matrix-Vector Multiplication}.}
With our subspace embedding result in place, we move onto proving the necessary approximate matrix-vector multiplication result introduced in Section \ref{sec:analysis}. In particular, we wish to show that, with good probability, $\|\mathbf{U}^T \mathbf{S}^T \mathbf{S} (\mathbf{b} - \mathbf{A} \mathbf{x}^*) \|_2^2 \leq \epsilon \| \mathbf{b} - \mathbf{A} \mathbf{x}^* \|_2^2$ where $\bv{x}^* = \argmin_\bv{x}\|\bv{A}\bv{x} - \bv{b}\|_2^2$. Since the residual of the optimal solution, $\mathbf{b} - \mathbf{A} \mathbf{x}^*$, is the same under any column transformation of $\bv{A}$, we have $\mathbf{b} - \mathbf{A} \mathbf{x}^* = \bv{b} - \bv{U}\bv{y}^*$ where $\bv{y}^* = \argmin_\bv{y}\|\bv{U}\bv{y} - \bv{b}\|_2^2$. We can then equivalently show that, if we take $k = O\left(\frac{d}{\epsilon\delta} \cdot D_{\text{inf}}\right)$ samples, with probability $1-\delta$,
\begin{align}
\label{eq:matvec_result}
    \| \mathbf{U}^T \mathbf{S}^T \mathbf{S} (\mathbf{b} - \mathbf{U} \mathbf{y}^*) \|_2^2 \leq \epsilon \| \mathbf{b} - \mathbf{U} \mathbf{y}^* \|_2^2.
\end{align}
As in the proof for independent random samples \citep{DrineasKannanMahoney:2006}, we will prove \eqref{eq:matvec_result} by bounding the expected squared error and applying Markov's inequality. In particular, we have
\begin{align}\begin{split}\label{eq:markov}
    \Pr \left[ \lVert \mathbf{U}^T \mathbf{S}^T \mathbf{S} (\mathbf{b} - \mathbf{U} \mathbf{y}^*) \rVert_2^2 \geq \epsilon \lVert \mathbf{b} - \mathbf{U} \mathbf{y}^* \rVert_2^2 \right] \leq \frac{\mathbb{E}\big[\lVert \mathbf{U}^T \mathbf{S}^T \mathbf{S} (\mathbf{b} - \mathbf{U} \mathbf{y}^*) \rVert_2^2 \big]}{\epsilon \lVert \mathbf{b} - \mathbf{U} \mathbf{y}^* \rVert_2^2} .
\end{split}\end{align}
Let $\mathbf{z} = \mathbf{b} - \mathbf{U} \mathbf{y}^*$ and note that $\mathbf{U}^T \mathbf{z} = \mathbf{0}$. The numerator on the right side can be transformed as
\begin{align*}\begin{split}
    \mathbb{E}\left[ \lVert \mathbf{U}^T \mathbf{S}^T \mathbf{S} \mathbf{z} \rVert_2^2 \right] = \mathbb{E}\left[ \lVert \mathbf{U}^T \mathbf{S}^T \mathbf{S} \mathbf{z} - \mathbf{U}^T \mathbf{z} \rVert_2^2 \right] 
    = \mathbb{E}\left[ \lVert \mathbf{U}^T (\mathbf{S}^T \mathbf{S} - \mathbf{I}) \mathbf{z} \rVert_2^2 \right]. 
\end{split}\end{align*}
Note that above $\mathbf{S}^T \mathbf{S} - \mathbf{I}$ is a diagonal matrix with $i$-th diagonal entry equal to $\frac{1}{\tilde{p}_i} - 1$ if $\xi_i=1$ and $-1$ if $\xi_i=0$. Expanding the $\ell_2$-norm and using that $\xi_i = 1$ and $\tilde{p}_i=1$ for $i \notin \mathcal{R}$, we have
\begin{align*}
    \mathbb{E}\big[ \lVert \mathbf{U}^T \mathbf{S}^T \mathbf{S} \mathbf{z} \rVert_2^2 \big] = \sum_{j=1}^d \mathbb{E} \left[ \left( \sum_{i=1}^n \left(\frac{\xi_i}{\tilde{p}_i}-1\right) u_{ij} z_i \right)^2 \right]
    &= \sum_{j=1}^d \mathbb{E} \left[ \left( \sum_{i \in \mathcal{R}} \left(\frac{\xi_i}{\tilde{p}_i}-1\right) u_{ij} z_i \right)^2 \right] \nonumber \\
    &= \sum_{j=1}^d \sum_{i \in \mathcal{R}} \sum_{l \in \mathcal{R}} \frac{c_{il}}{\tilde{p}_i \tilde{p}_l} u_{ij} z_i u_{lj} z_l, 
\end{align*}
where $c_{il} = \text{Cov}(\xi_i, \xi_l) = \E[(\xi_i - \tilde{p}_i)(\xi_l - \tilde{p}_l)] = \E[\xi_i\xi_l] - \tilde{p}_i\tilde{p}_l$. We will show that 
\begin{align*}
    \sum_{j=1}^d \sum_{i \in \mathcal{R}} \sum_{l \in \mathcal{R}} \frac{c_{il}}{\tilde{p}_i \tilde{p}_l} u_{ij} z_i u_{lj} z_l 
    \leq D_{\text{inf}} \sum_{j=1}^d \sum_{i \in \mathcal{R}} \frac{u_{ij}^2 z_i^2}{\tilde{p}_i},
\end{align*}
where $D_{\text{inf}}$ is the $\ell_\infty$-independence parameter. It suffices to show that for every $j$, we have
\begin{align}
\label{eq:before_psd}
    \sum_{i \in \mathcal{R}} \sum_{l \in \mathcal{R}} \frac{c_{il}}{\tilde{p}_i \tilde{p}_l} u_{ij} z_i u_{lj} z_l \leq D_{\text{inf}} \sum_{i \in \mathcal{R}} \frac{u_{ij}^2 z_i^2}{\tilde{p}_i}.
\end{align}
Define a symmetric matrix $\bv{M} \in \R^{|\mathcal{R}| \times |\mathcal{R}|}$ with entries $m_{il} = \frac{c_{il}}{\sqrt{\tilde{p}_i} \sqrt{\tilde{p}_l}}$ and a vector $\bv{v} \in \R^{|\mathcal{R}|}$ with entries $v_i = \frac{u_{ij} z_i}{\sqrt{\tilde{p}_i}}$. The desired result in \eqref{eq:before_psd} can be expressed as 
\begin{align*}\begin{split}
    \bv{v}^T \bv{M} \bv{v} \leq D_{\text{inf}} \lVert \bv{v} \rVert_2^2.
\end{split}\end{align*}
With $\mathcal{S} = \emptyset$, the one-sided $\ell_\infty$-independence condition implies that, for all $i \in [n]$, 
\begin{align*}\begin{split}
    \sum_{l \in \mathcal{R}} \left| \frac{\E[\xi_i \xi_l]}{\tilde{p}_i} -\tilde{p}_l \right| = \sum_{l \in \mathcal{R}} \frac{\sqrt{\tilde{p}_l}}{\sqrt{\tilde{p}_i}} | m_{il} | \leq D_{\text{inf}}.
\end{split}\end{align*}
Equivalently, if we define a diagonal matrix $\bm{\Lambda} \in \R^{|\mathcal{R}| \times |\mathcal{R}|}$ such that $\Lambda_{ii} = \sqrt{\tilde{p}_i}$,  we have shown:
\begin{align*}\begin{split}
    \lVert \bm{\Lambda}^{-1} \bv{M} \bm{\Lambda} \rVert_{\infty} \leq D_{\text{inf}},
\end{split}\end{align*}
where for a matrix $\bv{B}$, $\|\bv{B}\|_{\infty}$ denotes $\max_i \sum_{l} |\bv{B}_{il}| = \max_\bv{x} \frac{\|\bv{B}\bv{x}\|_{\infty}}{\|\bv{x}\|_{\infty}}$.
It follows that the largest eigenvalue of $\bm{\Lambda}^{-1} \bv{M} \bm{\Lambda}$ is at most $D_{\text{inf}}$ and thus, the largest eigenvalue of $\bv{M}$ is also at most $D_{\text{inf}}$. Therefore, we have $\bv{v}^T \bv{M} \bv{v} \leq D_{\text{inf}} \lVert \bv{v} \rVert_2^2$. Considering that $\tilde{p}_i \geq p_i = \frac{k}{d}\tau_i$ for $i \in \mathcal{R}$, we have
\begin{align*}
    \sum_{j=1}^d \mathbb{E} \left[ \left( \sum_{i \in \mathcal{R}} \left(\frac{\xi_i}{\tilde{p}_i}-1\right) u_{ij} z_i \right)^2 \right] \leq D_{\text{inf}} \sum_{j=1}^d \sum_{i \in \mathcal{R}} \frac{u_{ij}^2 z_i^2}{\tilde{p}_i}
    = D_{\text{inf}}\sum_{i \in \mathcal{R}} \frac{z_i^2}{\tilde{p}_i}\sum_{j=1}^d u_{ij}^2
    &= D_{\text{inf}}\sum_{i \in \mathcal{R}} \frac{z_i^2}{\tilde{p}_i}\tau_i \nonumber\\ &\leq D_{\text{inf}} \frac{d}{k} \lVert \bv{z} \rVert_2^2.
\end{align*}

Recalling that $\bv{z} = \bv{U}\bv{y}^* - \bv{b}$, we can plug into (\ref{eq:markov}) with $k = O\left( \frac{d}{\epsilon \delta} \cdot D_{\text{inf}} \right)$ samples, which proves that \eqref{eq:matvec_result}  holds with probability at least $1 - \delta$.

\noindent{\textbf{Putting it all together.}}
With \eqref{subspacegoal} and \eqref{eq:matvec_result} in place, we can prove our main result, \eqref{eq:main_thm_gaur} of Theorem \ref{thm:main}. We follow a similar approach to \citep{Woodruff:2014}. By reparameterization, proving this inequality is equivalent to showing that
\begin{align}\label{eq:reparam}\begin{split}
    \lVert \mathbf{U} \tilde{\mathbf{y}}^* - \mathbf{b} \rVert_2^2 \leq (1+\epsilon) \lVert \mathbf{U} \mathbf{y}^* - \mathbf{b} \rVert_2^2,
\end{split}\end{align}
where $\tilde{\mathbf{y}}^* = \argmin_{\bv{y}}\|\bv{S}\bv{U}\bv{y} - \bv{S}\bv{b}\|_2^2$ and ${\mathbf{y}}^* = \argmin_{\bv{y}}\|\bv{U}\bv{y} - \bv{b}\|_2^2$.
Since $\mathbf{y}^*$ is the minimizer of $\lVert \mathbf{U}\mathbf{y} - \mathbf{b} \rVert_2^2$, we have $\nabla_{\mathbf{y}} \lVert \mathbf{U} \mathbf{y} - \mathbf{b} \rVert_2^2 = 2 \mathbf{U}^T (\mathbf{U}\mathbf{y}- \mathbf{b}) = \mathbf{0}$ at $\mathbf{y}^*$. This indicates that $\mathbf{U} \mathbf{y}^* - \mathbf{b}$ is orthogonal to any vector in the column span of $\mathbf{U}$. Particularly, $\mathbf{U} \mathbf{y}^* - \mathbf{b}$ is orthogonal to $\mathbf{U} \tilde{\mathbf{y}}^* - \mathbf{U} \mathbf{y}^*$. Therefore, by the Pythagorean theorem, we have 
\begin{align}
\label{original_split}
    \lVert \mathbf{U} \tilde{\mathbf{y}}^* - \mathbf{b} \rVert_2^2 = \lVert \mathbf{U} \mathbf{y}^* - \mathbf{b} \rVert_2^2 + \lVert \mathbf{U} \tilde{\mathbf{y}}^* - \mathbf{U} \mathbf{y}^* \rVert_2^2 = \lVert \mathbf{U} \mathbf{y}^* - \mathbf{b} \rVert_2^2 + \lVert \tilde{\mathbf{y}}^* - \mathbf{y}^* \rVert_2^2.
\end{align}
So to prove \eqref{eq:reparam}, it suffices to show that 
\begin{align*}\begin{split}
    \lVert \tilde{\mathbf{y}}^* - \mathbf{y}^* \rVert_2^2 \leq \epsilon \lVert \mathbf{U} \mathbf{y}^* - \mathbf{b} \rVert_2^2.
\end{split}\end{align*}
Applying \eqref{subspacegoal} with $k = O(d\log d\cdot D_{\text{inf}}^2 + d\cdot D_{\text{inf}}/\epsilon)$ and $\delta = 1/200$, we have that, with probability $99.5/100$, $\lVert \mathbf{U}^T \mathbf{S}^T \mathbf{S} \mathbf{U} - \mathbf{I} \rVert_2 \leq \frac{1}{2}$.
Then, by triangle inequality,
\begin{align}
\label{eq:matvec_step1}
    \lVert \tilde{\mathbf{y}}^* - \mathbf{y}^* \rVert_2 &\leq \lVert \mathbf{U}^T \mathbf{S}^T \mathbf{S} \mathbf{U} (\tilde{\mathbf{y}}^* - \mathbf{y}^*) \rVert_2 + \lVert \mathbf{U}^T \mathbf{S}^T \mathbf{S} \mathbf{U} (\tilde{\mathbf{y}}^* - \mathbf{y}^*) - (\tilde{\mathbf{y}}^* - \mathbf{y}^*) \rVert_2 \nonumber\\
    &\leq \lVert \mathbf{U}^T \mathbf{S}^T \mathbf{S} \mathbf{U} (\tilde{\mathbf{y}}^* - \mathbf{y}^*) \rVert_2 + \lVert \mathbf{U}^T \mathbf{S}^T \mathbf{S} \mathbf{U} - \mathbf{I} \rVert_2 \lVert \tilde{\mathbf{y}}^* - \mathbf{y}^* \rVert_2  \nonumber \\
    &\leq \lVert \mathbf{U}^T \mathbf{S}^T \mathbf{S} \mathbf{U} (\tilde{\mathbf{y}}^* - \mathbf{y}^*) \rVert_2 + \frac{1}{2} \lVert \tilde{\mathbf{y}}^* - \mathbf{y}^* \rVert_2.
\end{align}
Rearranging, we conclude that $\lVert \tilde{\mathbf{y}}^* - \mathbf{y}^* \rVert_2^2 \leq 4 \lVert \mathbf{U}^T \mathbf{S}^T \mathbf{S} \mathbf{U} (\tilde{\mathbf{y}}^* - \mathbf{y}^*) \rVert_2^2$. Since $\tilde{\mathbf{y}}^*$ is the minimizer of $\lVert \mathbf{S} \mathbf{U} \mathbf{y} - \mathbf{S} \mathbf{b} \rVert_2^2$, we have $\nabla_{\mathbf{y}} \lVert \mathbf{S}\mathbf{U}\mathbf{y} - \mathbf{S} \mathbf{b} \rVert_2^2 = 2(\mathbf{S}\mathbf{U})^T(\mathbf{S}\mathbf{U}\tilde{\mathbf{y}}^* - \mathbf{S} \mathbf{b}) = \mathbf{0}$. Thus,
\begin{align*}\begin{split}
    \lVert \mathbf{U}^T \mathbf{S}^T \mathbf{S} \mathbf{U} (\tilde{\mathbf{y}}^* - \mathbf{y}^*) \rVert_2^2 &= \lVert \mathbf{U}^T \mathbf{S}^T (\mathbf{S} \mathbf{U} \tilde{\mathbf{y}}^* - \mathbf{S} \mathbf{b} + \mathbf{S} \mathbf{b} - \mathbf{S} \mathbf{U} \mathbf{y}^*) \rVert_2^2\\ &= \lVert \mathbf{U}^T \mathbf{S}^T \mathbf{S} (\mathbf{b} - \mathbf{U} \mathbf{y}^*) \rVert_2^2.
\end{split}\end{align*}
Applying \eqref{eq:matvec_result} with $\delta = 1/200$ and combining with \eqref{eq:matvec_step1} using union bound, we thus have that with probability $99/100$, 
\begin{align}\begin{split}
    \lVert \tilde{\mathbf{y}}^* - \mathbf{y}^* \rVert_2^2 \leq 4 \lVert \mathbf{U}^T \mathbf{S}^T \mathbf{S} \mathbf{U} (\tilde{\mathbf{y}}^* - \mathbf{y}^*) \rVert_2^2 = 4 \lVert \mathbf{U}^T \mathbf{S}^T \mathbf{S} (\mathbf{b} - \mathbf{U} \mathbf{y}^*) \rVert_2^2 \leq 4 \epsilon \lVert \mathbf{U} \mathbf{y}^* - \mathbf{b} \rVert_2^2.
\end{split}\end{align}
Plugging into \eqref{original_split} and adjusting $\epsilon$ by a constant factor completes the proof of \Cref{thm:main}.
\end{proof}

\subsection{Proof of Corollary \ref{corr:main}}
\label{sec:appendix:l_infty} 
We briefly comment on how to derive Corollary \ref{corr:main} from our main result. By definition of the method, we immediately have that our binary-tree-based pivotal sampling is $k$-homogeneous, so we just need to show that it produces a distribution over samples that is $\ell_\infty$-independent with constant parameter $D_{\text{inf}}$. 
This fact can be derived directly from a line of prior work. In particular, \citep{BJ12} proves that binary-tree-based pivotal sampling satisfies negative association, \citep{PP14} proves that negative association implies a stochastic covering property, and \citep{KKS22} shows that any distribution satisfying the stochastic covering property has $\ell_\infty$-independence parameter at most $D = 2$. We also give an arguably more direct alternative proof below based on a natural conditional variant of negative correlation. 

\begin{proof}[Proof of Corollary \ref{corr:main}]
\citep{BJ12} proves that binary-tree-based pivotal sampling is \emph{conditionally negatively associated} (CNA). Given a set $\mathcal{C} \subseteq [n]$ and a vector $\bv{c} \in \{0, 1\}^{|\mathcal{C}|}$, we denote the condition $\xi_c = c_i$ for all $i \in \mathcal{C}$ by $C$.  Conditional negative association asserts that, for all $C$, any disjoint subsets $\mathcal{S}$ and $\mathcal{T}$ of $\{\xi_1, \cdots, \xi_n \}$, and any non-decreasing functions  $f$ and $g$,
\begin{align*}\begin{split}
    \mathbb{E}[f(\mathcal{S}) | C]\cdot \mathbb{E}[g(\mathcal{T}) | C] \geq \mathbb{E}[f(\mathcal{S})g(\mathcal{T}) | C].
\end{split}\end{align*}
When $\mathcal{S}$ and $\mathcal{T}$ are singletons and $f$ and $g$ are the identity functions, we have
\begin{align}\begin{split}
\label{eq:cpnc}
    \mathbb{E}[\xi_i | C] \cdot \mathbb{E}[\xi_j | C] \geq \mathbb{E}[\xi_i \xi_j | C].
\end{split}\end{align}
Since $\E[\xi_i | C] = \Pr[\xi_i = 1 | C]$, we also have 
\begin{align}\begin{split}\label{eq:cnc}
    \Pr[\xi_i = 1 | C] \Pr[\xi_j = 1 | C] &\geq \Pr[\xi_i = 1 \wedge \xi_j = 1 | C] \\
    \Pr[\xi_i = 1 | C] &\geq \frac{\Pr[\xi_i = 1 \wedge \xi_j = 1 | C]}{\Pr[\xi_j = 1 | C]} = \Pr[\xi_i = 1 | \xi_j = 1 \wedge C].
\end{split}\end{align}
In words, the entries of our vector $\bs{\xi}$ are negatively correlated, even conditioned on fixing any subset of entries in the vector.  We will use this fact to show that $\sum_{j \in [n]} | \mathcal{I}_{\mu}^{\mathcal{S}}(i, j) | \leq 2$ for all $i \in [n]$, where $\mathcal{I}_{\mu}^{\mathcal{S}}$ is as defined in Definition \ref{def:l-infty}. For a fixed $i$, let $q_i = \Pr_{\bm{\xi} \sim \mu} [\xi_i=1 | \xi_{\ell} = 1 \forall \ell \in \mathcal{S}]$. Then, we have $| \mathcal{I}_{\mu}^{\mathcal{S}}(i, j) | = 0$ for $j \in \mathcal{S}$, $| \mathcal{I}_{\mu}^{\mathcal{S}}(i, j) | = 1 - q_i$ for $j = i$, and $\sum_{j \in [n] \backslash \mathcal{S} \cup \{ i \}} | \mathcal{I}_{\mu}^{\mathcal{S}}(i, j) | = 1 - q_i$. The last fact follows from $k$-homogeneity, i.e. that $\sum_{i=1}^n q_i = k$, and \eqref{eq:cnc}, which implies that $\mathcal{I}_{\mu}^{\mathcal{S}}(i, j) \leq 0$ for all $j$ in $[n] \backslash \mathcal{S} \cup \{ i \}$, so $\sum_{j \in [n] \backslash \mathcal{S} \cup \{ i \}} | \mathcal{I}_{\mu}^{\mathcal{S}}(i, j) | = \left|\sum_{j \in [n] \backslash \mathcal{S} \cup \{ i \}} \mathcal{I}_{\mu}^{\mathcal{S}}(i, j) \right|$. Thus, we have $\sum_{j \in [n]} | \mathcal{I}_{\mu}^{\mathcal{S}}(i, j) | = 2 - 2q_i \leq 2$.
\end{proof}

\section{Proof of \Cref{thm:poly}}
\label{app:poly_proof}
In this section we prove \Cref{thm:poly}, which shows that pivotal obtains a better sample complexity for polynomial regression on an interval than independent leverage score sampling. Since we can always shift and scale our target function, without loss of generality we can take $[\ell, u]$ to be the interval $[-1,1]$. We will sample from the infinite set of points on this interval. Each point $t\in  [-1,1]$ will correspond to a row in a regression matrix $\bv{A}$ with $d+1$ columns and an infinite number of rows. Such an object is sometimes referred to as a quasimatrix \citep{Trefethen:2009}.\footnote{We refer the reader to \cite{universal_sampling}, \cite{tamas_2020}, or \cite{CP19} for a more in depth treatment of leverage score sampling and active linear regression for quasimatrices.} The row with index $t$ equals $\bv{a}_t = [1,t,t^2, \ldots, t^d]$. The leverage score for the row with index $t$ is defined analogously to \Cref{eq:lev_defs} as:
\begin{align}
	\label{eq:poly_levs}
	\tau(t) = \max_{\bv{x}\in \R^{d+1}}\frac{(\bv{x}^T\bv{a}_t)^2}{\int_{-1}^1 (\bv{x}^T\bv{a}_s)^2 ds}.
\end{align}
Note that $\int_{-1}^1 \tau(t)dt = (d+1)$ since $\bv{A}$ has $d+1$ columns. We will sample $k$ points from $[-1, 1]$ interval with the probability of sampling point $t$ to be proportional to $\tau(t)$.
% For an even integer $m$ to be specified later, define $I_1, \ldots, I_m$ to be intervals that partition $[-1,1]$ so that $\int_{I_i} \tau(t)dt= \frac{d+1}{m}$. $I_1$ has left end point at $-1$, $I_2$ has left end point equal to the right end point of $I_1$, etc., so that $I_m$ ends with its right end point at $1$. 
% We will sample $m$ points from $[-1,1]$ interval with marginal probability density proportional to $\tau(t)$.
To do so, we renormalize $\tau(t)$ and consider sampling point $t$ with probability proportional to $\frac{k\tau(t)}{d+1}$. This is the analog of sampling with probability $\tilde p_i$ in the discrete case, since $\int_{-1}^1 \frac{k \tau(t)}{d+1} \,dt = k$. Then we apply pivotal sampling in the infinite point limit, where we choose the pivotal competition order so that points in $[-1,1]$ compete with each other from left to right across the interval. In this limit, the probability carried along in the competition will, at one point, accumulate to exactly $1$. This defines the interval $I_1$ whose left endpoint is $-1$ and its right endpoint is defined so that $\int_{I_1} \frac{k \tau(t)}{d+1} dt = 1$. This means that the winner in the last competition in $I_1$ will be sampled. A new competition then starts for the next point, until the renormalized leverage score accumulates to $1$ again and the winner is sampled. We can define $I_i (i=1, 2, \dots, k)$ to be adjacent intervals with $\int_{I_i} \frac{k \tau(t)}{d+1} dt = 1$. It can be seen that pivotal sampling in the infinite point limit will always sample \emph{exactly one point} from each of $I_1, \ldots, I_k$. This is actually the only property of pivotal sampling we will need to prove \Cref{thm:poly}, although we note that, within an interval $I_i$, point $t$ is selected with probability proportional to $\tau(t)/\int_{I_i} \tau(s)ds$.

To prove a sample complexity bound on $O(d/\epsilon)$ for polynomial regression, it suffices to  prove that a constant factor subspace embedding guarantee holds when selecting $O(d)$ samples from $[-1,1]$ using the pivotal sampling process above. Observe that the required $\epsilon$-error approximate matrix multiplication guarantee from \Cref{sec:analysis} already holds with $O(d/\epsilon)$ samples collected via the pivotal method with leverage score marginals, and this analysis extends to the infinite quasi-matrix setting (see \Cref{sec:matvec_poly}). It is only the subspace embedding guarantee that adds an extra $\log d$ factor to the sample complexity. Since this factor is inherent to our use of a matrix Chernoff bound in analyzing the general case, to prove a tighter bound for polynomial regression we use a direct analysis that avoids matrix Chernoff entirely. In particular, our main result of this section is as follows:
\begin{theorem}
	\label{thm:mainerrbound}
	Let $k = O\left(\frac{d}{\alpha}\right)$ points $t_1, \ldots, t_k$ be selected from $[-1,1]$ via pivotal sampling. Specifically, exactly one point $t_i$ is sampled from each  interval $I_1, \ldots, I_k$ with probability proportional to its leverage score. We have (deterministically) that for any degree $d$ polynomial $p$, with $w(t) = \frac{\tau(t)}{d+1}$,
	\begin{align}
\label{eq:polyembed}
	\left|\int_{-1}^1 p(t)^2 dt - \frac{1}{k}\sum_{i=1}^k \frac{p(t_i)^2}{w(t_i)}\right| \leq \alpha \int_{-1}^1 p(t)^2 dt.
	\end{align}
\end{theorem}
To translate from the notation above to the notation used in \eqref{subspacegoal}, note that $\int_{-1}^1 p(t)^2 dt = \|\bv{U}\bv{x}\|_2^2$ where $\bv{U}$ is an orthogonal span for the quasi-matrix $\bv{A}$ defined above, and $\bv{x}$ is a $d+1$ dimensional vector containing coefficients of the polynomial $p$ in the basis $\bv{U}$. Correspondingly,   $\frac{1}{k}\sum_{i=1}^k \frac{p(t_i)^2}{w(t_i)} = \|\bv{S}\bv{U}\bv{x}\|_2^2$.
\begin{proof}[Proof of \Cref{thm:poly}]
	With \Cref{thm:mainerrbound} providing the required  subspace embedding guarantee from \eqref{subspacegoal}, the proof of \Cref{thm:poly} is essentially identical to the proof for the discrete case (\Cref{thm:main}). The only remaining requirement is the required approximate matrix-vector multiplication guarantee from \eqref{eq:matvec_result}. For completeness, we provide a direct proof in the polynomial regression case in \Cref{sec:matvec_poly}. 
 % Since $\alpha$ can be set to a constant in the active regression analysis, we conclude that \Cref{thm:poly} holds with $O(d + d/\epsilon)$ leverage score samples. 
 % [add statement of the approximate mat-vec multiplication in this case] [Not sure how to extend the $\ell_\infty$ independence rigorously since we no longer can talk about the probability of choosing a point]
    % We can also prove the approximate matrix-vector multiplication bound \eqref{eq:matvec_result} directly in the polynomial regression case. This also helps to illustrate how this bound looks like in terms of polynomials. The proof is deferred to the end of this section when we have learned more about properties of the leverage score in the polynomial regression case. [maybe put this part outside of this proof?]
\end{proof}

We build up to the proof of \Cref{thm:mainerrbound} by introducing several intermediate results. Our approach is inspired by a result of \cite{KaneKarmalkarPrice:2017}, which also obtains  $O(d/\epsilon)$ sample complexity for active degree-$d$ polynomial regression, albeit with a sampling distribution not based on leverage scores. In particular, they prove the same guarantees as \Cref{thm:mainerrbound}, but where points are selected uniformly from $k$ intervals that evenly partition the Chebyshev polynomial weight function $1/\sqrt{1-t^2}$. As we will see, the leverage scores for polynomials closely approximate this weight function. This connection is well known, as the leverage score function from \eqref{eq:poly_levs} is exactly proportional to the inverse of the polynomial ``Christoffel function'' under the uniform measure, a well-studied function in approximation theory.

For simplicity, from now on we denote $f(t) = p(t)^2$ and describe the main structure of the proof of Theorem \ref{thm:mainerrbound}. The left hand side of \eqref{eq:polyembed} can be decomposed into a sum of errors in individual intervals as
\begin{align} \label{eq: decomposeL1intointervals}
	\left| \int_{-1}^1 f(t) dt - \sum_{i=1}^k \frac{1}{k} \frac{f(t_i)}{w(t_i)} \right| \leq \sum_{i=1}^k \int_{I_i} \left| \frac{f(t)}{w(t)} - \frac{f(t_i)}{w(t_i)} \right| w(t) dt.
 \end{align}

	In each interval $I_i$, noting that by the definition of $I_i$, $\int_{I_i} w(t) dt = \frac{1}{k}$. So we have
	\begin{align}
 \label{eq:structureofpolybound}
	\begin{split}
		\int_{I_i} \left| \frac{f(t)}{w(t)} - \frac{f(t_i)}{w(t_i)} \right| w(t) dt & = \int_{I_i} \left| \int_{t_i}^t \frac{f'(s) w(s) - f(s) w'(s)}{w^2(s)} ds \right| w(t) dt \\
		& \leq \int_{I_i} \left( \int_{I_i} \frac{|f'(s)| w(s) + f(s) |w'(s)|}{w^2(s)} ds \right) w(t) dt \\
		& = \frac 1k \int_{I_i} \frac{|f'(t)|}{w(t)} dt + \frac 1k \int_{I_i} \frac{f(t) |w'(t)|}{w^2(t)} dt.
	\end{split}
	\end{align}
Therefore, 
\begin{align}
\label{eq: polyboundmain}
\left| \int_{-1}^1 f(t) dt - \sum_{i=1}^k \frac{1}{k} \frac{f(t_i)}{w(t_i)} \right| \leq \frac 1k \int_{-1}^{1} \frac{|f'(t)|}{w(t)} dt + \frac 1k \int_{-1}^{1} \frac{f(t) |w'(t)|}{w^2(t)} dt.
\end{align}
Then, proving the upper bound in Theorem \ref{thm:mainerrbound} boils down to establishing a \emph{lower bound} for $w(t)$, an \emph{upper bound} for $|w'(t)|$, and the connection between  $\int_{-1}^1 \frac{|f'(t)|}{w(t)}dt$ and $\int_{-1}^1 f(t) dt$, which is  related to the lower bound for $w(t)$. 

Both $f(t)$ and $w(t)$ are polynomials. This is key to providing both pointwise and (weighted) integral upper bounds of their derivatives by their function values and integrals. In polynomial approximation theory, Markov-Bernstein inequalities address exactly this point. 

We begin with the integral form of the Markov-Bernstein inequalities from \cite{Borwein_1995} and \cite{nevai_1979}.
\begin{proposition}[$L^1$ Bernstein's inequality]
	\label{prop: L1Bernstein}
	For any degree $d$ polynomial $p(t)$, with a universal constant $C_0$, 
	\begin{align}\label{eq: L1Bernstein}
	\int_{-1}^1 \left\lvert\sqrt{1-t^2}\,p'(t)\right\rvert dt \leq C_0\, d \int_{-1}^1 |p(t)| dt.
	\end{align}
\end{proposition}

\begin{proposition}[$L^1$ Markov's inequality]
	\label{prop: L1Markov} For any degree $d$ polynomial $p(t)$, with a universal constant $C_1$, 
	\begin{align}\label{eq: L1Markov}
	\int_{-1}^1 \left\lvert p'(t)\right\rvert  dt \leq C_1\, d^2 \int_{-1}^1 |p(t)| dt.
	\end{align}
\end{proposition}
The integral on the left hand side of \eqref{eq: L1Bernstein} is weighted by $\sqrt{1-t^2}$. This weight is negligible near the boundary of the interval $[-1, 1]$, which, to some extent, explains the milder $O(d)$ dependence  on the right hand side of \eqref{eq: L1Bernstein} compared to the $O(d^2)$ dependence in \eqref{eq: L1Markov}. 

\eqref{eq: L1Bernstein} and \eqref{eq: L1Markov} also indicate that a lower bound on $w(t)$ of the form $\frac{1}{\sqrt{1-t^2}}$ or a constant can cooperate with $|f'(t)|$ to yield an integral bound. This is indeed realizable, and we will prove that the lower bound on $w(t)$ is of the two different forms in different parts of the interval $[-1, 1]$. It is roughly the minimum of these two, as $\frac{1}{\sqrt{1-t^2}}$ blows up near the boundary. We define the \emph{middle region} to be a centered subinterval in $[-1, 1]$ that is $\sim \frac{1}{d^2}$ away from the boundary, and the \emph{boundary region} to be the two subintervals of length $\sim \frac{1}{d^2}$ near the boundary. The exact expression of these two regions will be clear when we state different lower bounds of $w(t)$. There we note that these two regions are overlapping so that we get a lower bound on $w(t)$ on the whole interval. Matching this pattern of lower bounds, we will also apply different upper bounds of $w'(t)$ in these two regions. 

 \noindent{\textbf{Lower bound on $w(t)$ in the middle region.} }
 \cite{erdelyi_nevai_1992} states the following bound on $\tau(t)$, which shows the relation between the leverage score and the Chebyshev polynomial weight function. See Section 4.3 in \cite{MeyerMuscoMusco:2023} for further discussion of this result.

\begin{proposition}
    With a constant $C_2>0$,
	\label{prop: middlew}
	\begin{align}
	\tau(t) \geq \frac{C_2 (d+1)}{\pi\sqrt{1-t^2}}, \quad w(t) \geq \frac{C_2 }{\pi\sqrt{1-t^2}}, \quad \text{for } |t| \leq \sqrt{1 - \frac{9}{(d-1)^2}}.
	\end{align}
\end{proposition}

\noindent{\textbf{Lower bound on $w(t)$ in the boundary region.}} This requires a bit more effort than the case of the middle region. Inspired by \eqref{eq: L1Markov}, we aim for the following bound on $\tau(t)$ and $w(t)$.

\begin{proposition}\label{prop: encapw}
With constants $C_3>0, c>\frac{9}{2}$,
	\begin{align}
	\tau(t) \geq C_3 d (d+1), \quad w(t) \geq C_3 d,  \quad \text{for } |t| > 1 - \frac{c}{d^2}.
	\end{align}
\end{proposition}

We specify $c> \frac{9}{2}$ because this implies that, at least for $d$ large enough, the middle region, defined as $\left[-\sqrt{1 - \frac{9}{(d-1)^2}}, \sqrt{1 - \frac{9}{(d-1)^2}}\right]$, and the boundary region, defined as $\left[-1, -1 + \frac{c}{d^2}\right]  \cup \left[1 - \frac{c}{d^2}, 1\right]$, have overlap.

To prove Proposition \ref{prop: encapw}, we need to employ the explicit expression for $\tau(t)$ in terms of Legendre polynomials and apply properties of Legendre polynomials. We denote the unnormalized degree $d$ Legendre polynomial by $P_d(t)$ and we fix $P_d(1) = 1$. We use $L_d(t)$ to denote the normalized Legendre polynomials satisfying $\int_{-1}^1 L_d^2(t) dt = 1$. There are classical, explicit bounds for $L_d(t)$ and its derivative on the interval $[-1, 1]$.

\begin{lemma}[See e.g. \cite{seigel_1955}]
	\label{lem: Ln}
	For $|t| \leq 1$, the normalized degree $d$ Legendre polynomials $L_d$, satisfy
	\begin{align}
	|L_d(t)| \leq L_d(1) = \sqrt{d + \frac 12}, \quad |L_d'(t)| \leq \frac{d(d+1)}{2}\sqrt{d+\frac12}.
	\end{align}
\end{lemma}

\begin{proof}[Proof of Proposition \ref{prop: encapw}:]
 Since $\tau(t)$ and $w(t)$ are related as $w(t) = \frac{\tau(t)}{d+1}$, we only need to prove the bound for $\tau(t)$. The proof follows from the fact that, $\tau(t)$ can be equivalently written as the squared norm of the $t$-th row in any orthonormal basis for $\bv{A}$, 
	\begin{align}\label{eq: tauinLn}
		\tau(t) = \sum_{i=0}^d L_i^2(t).
	\end{align}
See e.g. Equation (4) in \cite{MeyerMuscoMusco:2023} for this equivalency. We will then lower bound $\tau(t)$ by deriving lower bounds for the individual summands, ultimately arguing that for $|t| > 1-\frac{c}{d^2}$, $\tau(t) \gtrsim d^2$. 
%[the idea is that Legendre polynomials are almost tight for its derivative...] 
 Lower bounds on individual $L_i^2(t)$ can be derived with an upper bound on the derivative, since we know at the boundary $L_i(1) = \sqrt{i+\frac{1}{2}}$ exactly and $L_i^2(t)$ is an even function. In particular, by Lemma \ref{lem: Ln},
	\[
	|(L_i^2)'(t)| = |2L_i(t) L_i'(t)| \leq i(i+\frac12)(i+1).
	\]
	Therefore, for any $c > 0$ and $|t| > 1 - \frac{ c}{d^2}$,
	\[
	L_i^2(t) \geq (i+\frac12) - i(i+\frac12)(i+1)\frac{ c}{d^2}.
	\]
	For a given $c$, $L_i^2(t)$ is thus positive and of order $\gtrsim i$ as long as $i \leq d/c'$ for a sufficiently large constant $c' > 1$. A careful algebraic manipulation will show that we can choose $c>\frac{9}{2}$. With the classical formula for summing a linear growth sequence, we have for some constant $C_3$, 
	\[
	\tau(t) = \sum_{i=0}^d L_i^2(t) \geq \sum_{i=0}^{d/c'} L_i^2(t) \geq C_3 d(d+1).\qedhere
	\]
\end{proof}

Now we move on to upper bounds on $w'(t)$ in the two regions matching the form of $w(t)$ so that some cancellation can occur to leave a clean integral $\int_{I_i} f(t) dt$ in \eqref{eq: polyboundmain}. The classical Markov-Bernstein inequalities in \cite{Borwein_1995} bound the pointwise value of the derivative of a polynomial by its maximum function value.

\begin{lemma}[Bernstein's inequality]\label{lem: bernstein}
   For any degree $d$ real polynomial $p(t)$, 
    \begin{align}
        |p'(t)| \leq \frac{d}{\sqrt{1-t^2}}\sup_{t \in [-1, 1]}|p(t)|,\quad \text{for } -1< t <1.
    \end{align}
\end{lemma}

\begin{lemma}[Markov's inequality]\label{lem: markov}
For any degree $d$ real polynomial $p(t)$, 
\begin{align}
        |p'(t)| \leq d^2\sup_{t \in [-1, 1]}|p(t)|,\quad \text{for }-1< t <1.
    \end{align}
\end{lemma}
Both Lemma \ref{lem: bernstein} and Lemma \ref{lem: markov} are valid on the whole interval $[-1, 1]$, but one of them is tighter than the other depending on where $t$ sits. Although $w'(t)$ itself is a degree $2d-1$ polynomial, it turns out that the bounds directly given by Lemma \ref{lem: bernstein} and Lemma \ref{lem: markov} are too crude.  However, we can unwrap more structure in $\tau'(t)$ by applying the recurrence relation of Legendre polynomials.
\begin{proposition}[Neat expression of $\tau'(t)$]
	\label{prop: derivativeoftau}
	\begin{align}
	\tau'(t) = P_{d+1}'(t) P_d'(t).
	\end{align}
\end{proposition}
\begin{proof}
We write $\tau(t)$ in terms of the unnormalized Legendre polynomials $P_n$ as
	\begin{align*}
	\tau(t) = \sum_{i=0}^d L_i^2(t) 
	=  \sum_{i=0}^d (i + \frac 12) P_i^2.
	\end{align*}
	$P_n$ and $P_n'$ are related by the recurrence relation $
	(2n+1) P_n = P_{n+1}' - P_{n-1}'$ with $P_0' = 0$, so
	\[
	\tau'(t) = \sum_{i=0}^d (2i + 1) P_i P_i' = \sum_{i=1}^d P_i' (P_{i+1}' - P_{i-1}') = P_{d+1}' P_d'.   \qedhere
	\] 
\end{proof}
\noindent{\textbf{Upper bound on $w'(t)$ in the middle region.}} 

\begin{proposition}
	\label{prop: boundofdtau}
	\begin{align}
	|\tau'(t)| \leq \frac{d(d+1)}{1-t^2}, \quad |w'(t)| \leq \frac{d}{1-t^2}, \quad \text{for } -1 < t < 1.
	\end{align}
\end{proposition}
\begin{proof}
	This follows from Proposition \ref{prop: derivativeoftau} and Bernstein's inequality in Lemma \ref{lem: bernstein} for Legendre polynomials $P_{d}$ and $P_{d+1}$ by noting that the maximum of these polynomials are $P_i(1) = 1$, $\forall i$.
\end{proof}
\noindent{\textbf{Upper bound on $w'(t)$ in the boundary region.}} The bound in Proposition \ref{prop: boundofdtau} will blow up near the boundary. This is caused by the blowing up of Bernstein's inequality in Lemma \ref{lem: bernstein}. We should use the tighter upper bound offered by Markov's inequality in Lemma \ref{lem: markov} and a similar proof will give the following bound.

\begin{proposition}
    \label{prop: boundofdtauboundary}
    \begin{align}
        |\tau'(t)| \leq d^2(d+1)^2, \quad |w'(t)| \leq d^2(d+1), \quad \text{for } -1 \leq t \leq 1.
    \end{align}
\end{proposition}

As mentioned, Proposition \ref{prop: boundofdtau} and Proposition \ref{prop: boundofdtauboundary} cannot be obtained by treating $\tau(t)$ as a general $2d-1$ polynomial. In fact, directly applying Lemma \ref{lem: bernstein} and Lemma \ref{lem: markov} to $\tau(t)$, we will get instead
\begin{align*}
|\tau'(t)| \leq \frac{(2d-1)(d+1)^2}{2\sqrt{1-t^2}} \quad \text{and} \quad |\tau'(t)| \leq \frac12(2d-1)^2 (d+1)^2,
\end{align*}
which is significantly worse in the middle region.

With the previous intermediate results in place, we are now ready to prove our main claim, following the proof of Lemma 2.1 in \cite{KaneKarmalkarPrice:2017} for intervals defined by the Chebyshev measure. 
\begin{proof}[Proof of \Cref{thm:mainerrbound}] Recall we start with \eqref{eq: polyboundmain} as
	\begin{align}
        \label{eq: polyboundmainrepeat}
	\left| \int_{-1}^1 f(t) dt - \sum_{i=1}^k \frac{1}{k} \frac{f(t_i)}{w(t_i)} \right| \leq \frac 1k \int_{-1}^{1} \frac{|f'(t)|}{w(t)} dt + \frac 1k \int_{-1}^{1} \frac{f(t) |w'(t)|}{w^2(t)} dt.
	\end{align}

	We will use different pointwise bounds of $w(t)$ and $w'(t)$ in different parts of the interval. For the \emph{middle region}, by \Cref{prop: middlew} and \Cref{prop: L1Bernstein}  we have
	\begin{align}
            \label{eq: dfwmiddle}
	\frac{1}{k} \int_{-\sqrt{1 - \frac{9}{(d-1)^2}}}^{\sqrt{1 - \frac{9}{(d-1)^2}}} \frac{|f'(t)|}{w(t)} dt \leq \frac{\pi}{C_2 k} \int_{-1}^1 \sqrt{1 - t^2} |f'(t)| dt \leq \frac{2\pi C_0 d}{C_2 k} \int_{-1}^1 f(t) dt.
	\end{align}
	
	Also, by Proposition \ref{prop: middlew} and Proposition \ref{prop: boundofdtau}, 
	$
	\frac{|w'(t)|}{w^2(t)} \leq \frac{\frac{d}{1-t^2}}{\frac{C_2^2}{\pi^2 (1-t^2)}} = \frac{\pi^2 d}{C_2^2}$.
	Therefore,
	\begin{align}
            \label{eq: fdwwmiddle}
	\frac{1}{k} \int_{-\sqrt{1 - \frac{9}{(d-1)^2}}}^{\sqrt{1 - \frac{9}{(d-1)^2}}} \frac{f(t) |w'(t)|}{w^2(t)} dt \leq \frac{\pi^2 d}{C_2^2 k} \int_{-1}^1 f(t) dt.
	\end{align}

	For the \emph{boundary region}, by \Cref{prop: encapw} and \Cref{prop: L1Markov} we have
	\begin{align}
        \label{eq: dfwboundary}
	\frac{1}{k} \int_{[-1, -1 + \frac{c}{d^2}] \cup[1 - \frac{c}{d^2}, 1] }  \frac{|f'(t)|}{w(t)} dt \leq \frac{1}{C_3kd} \int_{-1}^1 |f'(t)| dt \leq \frac{4 C_1d}{C_3 k}\int_{-1}^1 |f(t)| dt.
	\end{align}
	
	Also, by Proposition \ref{prop: encapw} and Proposition \ref{prop: boundofdtauboundary}, we have
	$
	\frac{|w'(t)|}{w^2(t)} \leq \frac{(d+1)d^2}{C_3^2 d^2}  \leq \frac{2d}{C_3^2}$.
	Therefore,
	\begin{align}
            \label{eq: fdwwboundary}
	\frac{1}{k} \int_{[-1, -1 + \frac{c}{d^2}] \cup[1 - \frac{c}{d^2}, 1]} \frac{f(t) |w'(t)|}{w^2(t)} dt \leq \frac{2d}{C_3^2 k}\int_{-1}^1 f(t) dt.
	\end{align}
	Finally, noticing that the middle region $\left[ -\sqrt{1 - \frac{9}{(d-1)^2}}, \sqrt{1 - \frac{9}{(d-1)^2}}\right]$ and the boundary region $\left[-1, -1 + \frac{c}{d^2}\right] \cup \left[1 - \frac{c}{d^2}, 1\right]$ are overlapping, we can further upper bound the right hand side of \eqref{eq: polyboundmainrepeat} by adding up \eqref{eq: dfwmiddle}, \eqref{eq: fdwwmiddle}, \eqref{eq: dfwboundary}, and \eqref{eq: fdwwboundary},  so that
	\begin{align}
		\left| \int_{-1}^1 f(t) dt - \sum_{i=1}^k \frac{1}{k} \frac{f(t_i)}{w(t_i)} \right| \leq \left(\frac{2\pi C_0}{C_2} + \frac{\pi^2}{C_2^2} + \frac{4 C_1}{C_3} + \frac{2}{C_3^2}\right) \frac{d}{k} \int_{-1}^1 f(t) dt.
	\end{align}
	
	Then, by taking $k = O(\frac{d}{\alpha})$, we prove the statement.
\end{proof}
\subsection{Direct Proof of Approximate Matrix-Vector Multiplication Bound}
\label{sec:matvec_poly}
\begin{proof}[Approximate matrix-vector multiplication bound for \Cref{thm:poly}]
Let $\mathbf{U}$ be an orthogonal basis of the column span of the polynomial regression quasi-matrix $\mathbf{A}$, as defined at the beginning of \Cref{app:poly_proof}. 
As in the finite dimensional setting, our goal is to prove that, if we take $k = O\left(\frac{d}{\epsilon\delta}\right)$ samples via pivotal sampling, using those samples to construct a sampling matrix $\bv{S}$, then with probability $1-\delta$,
\begin{align}
\label{eq:matvec_result_poly}
    \| \mathbf{U}^T \mathbf{S}^T \mathbf{S} (\mathbf{b} - \mathbf{U} \mathbf{y}^*) \|_2^2 \leq \epsilon \| \mathbf{b} - \mathbf{U} \mathbf{y}^* \|_2^2.
\end{align}
We start with the key step \eqref{eq:markov} in the matrix case as
\begin{align}\label{eq:markovrepeat}
    \Pr \left[ \lVert \mathbf{U}^T \mathbf{S}^T \mathbf{S} (\mathbf{b} - \mathbf{U} \mathbf{y}^*) \rVert_2^2 \geq \epsilon \lVert \mathbf{b} - \mathbf{U} \mathbf{y}^* \rVert_2^2 \right] \leq \frac{\mathbb{E}\big[\lVert \mathbf{U}^T \mathbf{S}^T \mathbf{S} (\mathbf{b} - \mathbf{U} \mathbf{y}^*) \rVert_2^2 \big]}{\epsilon \lVert \mathbf{b} - \mathbf{U} \mathbf{y}^* \rVert_2^2} .
\end{align}
Our goal is to prove $\mathbb{E}\big[\lVert \mathbf{U}^T \mathbf{S}^T \mathbf{S} (\mathbf{b} - \mathbf{U} \mathbf{y}^*) \rVert_2^2 \big] = O(\frac{d}{k})  \lVert \mathbf{b} - \mathbf{U} \mathbf{y}^* \rVert_2^2$.
% without resorting to $\ell_{\infty}$-independence, and then we can follow the same proof as in the matrix case to prove the approximate matrix-vector multiplication bound.
We have that $\mathbf{U}$'s $i$-th column is the normalized Legendre polynomial $L_i$. Let $\mathbf{p}^* = \argmin_{\text{degree $d$ polynomial $\mathbf{p}$}} \|\mathbf{p} - \mathbf{b}\|_2^2$. Then $\mathbf{Uy}^* = \mathbf{p}^*$, where $\mathbf{y}^*$ is a length $d+1$ vector holding the coefficients of $\mathbf{p}^*$ when expanded in the basis of Legendre polynomial $L_i$'s. So the expectation can be written as 
\begin{align}
\label{eq: polymatvecinitial}
    \mathbb{E}\big[\lVert \mathbf{U}^T \mathbf{S}^T \mathbf{S} (\mathbf{b} - \mathbf{U} \mathbf{y}^*) \rVert_2^2 \big] = \mathbb{E}\left[\sum_{j=0}^d \left(\sum_{i=1}^k  \frac{L_j(t_i) (p^*(t_i) - b(t_i))}{k w(t_i)}\right)^2\right].
\end{align}
Note that each $t_i$ is chosen independently from $I_i$ with probability density function in $I_i$ as $\tau(t) / \int_{I_i} \tau(t) dt = k w(t)$. So we can calculate the expectation in \eqref{eq: polymatvecinitial} to get
\begin{align}
\label{eq: polymatveceq}
\begin{split}
    \mathbb{E}\big[\lVert \mathbf{U}^T \mathbf{S}^T \mathbf{S} (\mathbf{b} - \mathbf{U} \mathbf{y}^*) \rVert_2^2 \big] & =\mathbb{E}\left[\sum_{j=0}^d \sum_{i=1}^k \left( \frac{L_j(t_i) (p^*(t_i) - b(t_i))}{k w(t_i)}\right)^2\right] \\
    & = \sum_{j=0}^d\sum_{i=1}^k \int_{I_i}\frac{\left(L_j(t)(p^*(t) - b(t))\right)^2}{kw(t)} dt\\
    & = \sum_{i=1}^d \int_{I_i} \frac{\tau(t)(p^*(t) - b(t))^2}{kw(t)}dt\\
    & = \frac{d+1}{k}\int_{-1}^1 (p^*(t) - b(t))^2 dt,
\end{split}
\end{align}
where we use \eqref{eq: tauinLn} in the third equality. The integral in the last line is exactly $\lVert \mathbf{b} - \mathbf{U} \mathbf{y}^* \rVert_2^2$, so we prove that $\mathbb{E}\big[\lVert \mathbf{U}^T \mathbf{S}^T \mathbf{S} (\mathbf{b} - \mathbf{U} \mathbf{y}^*) \rVert_2^2 \big] = \frac{d+1}{k}  \lVert \mathbf{b} - \mathbf{U} \mathbf{y}^* \rVert_2^2$.
Plugging this back into \eqref{eq:markovrepeat}, we get
\begin{align}
     \Pr \left[ \lVert \mathbf{U}^T \mathbf{S}^T \mathbf{S} (\mathbf{b} - \mathbf{U} \mathbf{y}^*) \rVert_2^2 \geq \epsilon \lVert \mathbf{b} - \mathbf{U} \mathbf{y}^* \rVert_2^2 \right] \leq \frac{d+1}{\epsilon k}.
\end{align}
This shows that if we take $k = O(\frac{d}{\epsilon \delta})$ samples via pivotal sampling, then with probability $1-\delta$, $\| \mathbf{U}^T \mathbf{S}^T \mathbf{S} (\mathbf{b} - \mathbf{U} \mathbf{y}^*) \|_2^2 \leq \epsilon \| \mathbf{b} - \mathbf{U} \mathbf{y}^* \|_2^2$.
\end{proof}
% Comparing this bound with the bound obtained from $\ell_\infty$-independence, we find that one interesting point is that this bound will be matched by $\ell_\infty$-independence approach if $D_{\text{inf}} = 1$. And in some sense, we do have independence here -- points in different subintervals are chosen independently. This might suggest that the concept of $\ell_\infty$-independence can be generalized.

\section{Complementary Experiments}\label{sec:appendix:experiments}

In Section \ref{sec:experiments}, we conduct experiments using three targets; a damped harmonic oscillator, a heat equation, and a chemical surface reaction. Here, we show the results of additional simulations. Thus far, the original domain of the experiments is 2D for visualization purposes. Here, we consider the 3D original domain by freeing one more parameter in the damped harmonic oscillator model. We also summarize the number of samples required to achieve a given target error for all four test problems. To corroborate the discussion in Section \ref{sec:lev_score}, we provide a simulation result showing that leverage score sampling outperforms uniform sampling. Lastly, the performance of the randomized BSS algorithm from \cite{CP19} on the 2D damped harmonic oscillator target is displayed which supports our discussion in Section \ref{subsec:relatedwork}. Before showing the simulation results, we begin this section with a deferred detail of the chemical surface coverage target in Section \ref{sec:experiments}.

\noindent \textbf{Chemical Surface Coverage.} As explained in \citep{HamptonDoostan:2015}, the target function models the surface coverage of certain chemical species and considers the uncertainty $\rho$ which is parameterized by absorption $\alpha$, desorption $\gamma$, the reaction rate constant $\kappa$, and time $t$. Given a position $(x, y)$ in 2D space, our quantity of interest $\rho$ is modeled by the non-linear evolution equation:
\begin{align}\begin{split}
    \frac{d \rho}{d t} = \alpha (1-\rho) - \gamma \rho - \kappa (1-\rho)^2 \rho, \quad \rho(t=0)=0.9, \\ \alpha = 0.1 + \exp(0.05 x), \quad \gamma = 0.001 + 0.01 \exp(0.05 y)
\end{split}\end{align}
In this experiment, we set $\kappa=10$ and focus on $\rho$ after $t=4$ seconds. 

\noindent \textbf{3D Damped Harmonic Oscillator.} The damped harmonic oscillator model is given as (restated),
\begin{align}\begin{split}
    \frac{d^2 x}{d t^2}(t) + c \frac{dx}{dt}(t) + k x(t) = f \cos(\omega t), \quad x(0)=x_0, \quad \frac{dx}{dt}(0)=x_1
\end{split}\end{align}
In \Cref{sec:experiments}, we define the domain as $k \times \omega = [1,3] \times [0,2]$ while $f$ was fixed as $f=0.5$. This time, we extend it to 3D space by setting the domain as $k \times f \times \omega = [1,3] \times [0,2] \times [0,2]$.
$3$D plot of the target is given in Figure \ref{fig:appendix:leverage} (a). Note that if we slice the cube at $f'=0.5$, we obtain the target in Section \ref{sec:experiments}.
This time, we create the base data matrix $\mathbf{A}' \in \R^{n \times d'}$ by constructing a $n = 51^3$ fine grid, and the polynomial degree is set to $12$. \Cref{fig:appendix:leverage} (b) and (c) give examples of the leverage score of this data matrix.

\begin{figure}[h]
    \vspace{-1em}
    \centering
    \begin{subfigure}[b]{0.32\textwidth}
        \centering
        \includegraphics[width=\linewidth]{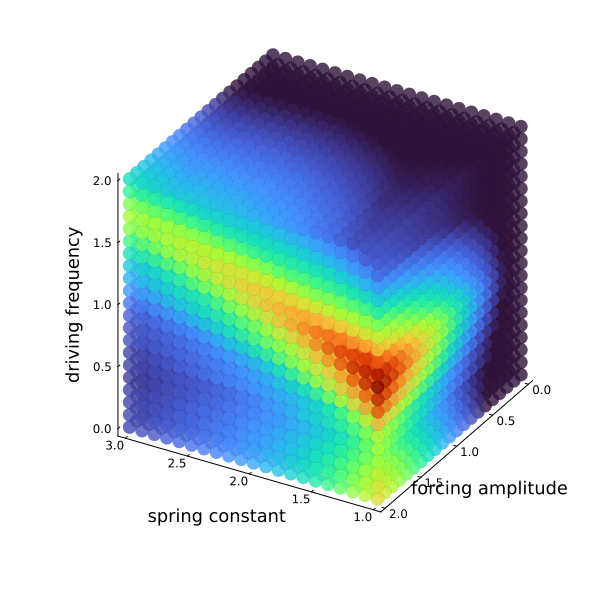}
                \vspace{-.5em}
        \caption{Target Function.}   
    \end{subfigure}
    \hfill
    \centering
    \begin{subfigure}[b]{0.32\textwidth}
        \centering
        \includegraphics[width=\linewidth]{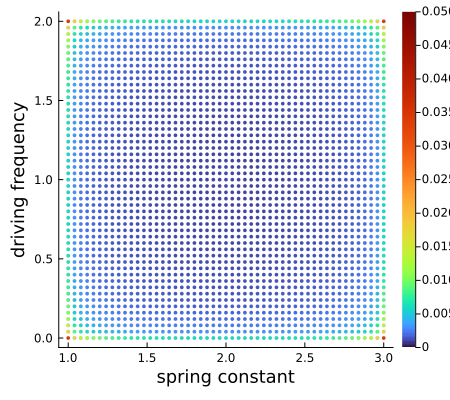}
                \vspace{-.5em}
        \caption{Leverage Score at $f'=0.0$.}     
    \end{subfigure}
    \hfill
    \centering
    \begin{subfigure}[b]{0.32\textwidth}
        \centering
        \includegraphics[width=\linewidth]{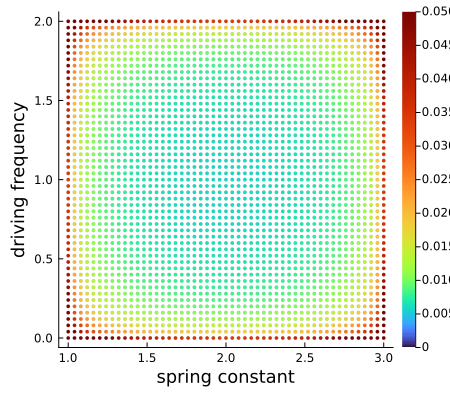}
                \vspace{-.5em}
        \caption{Leverage Score at $f'=1.0$.}    
    \end{subfigure}
    \caption{(a): The target value of the 3D damped harmonic oscillator model. (b) and (c): Its leverage score.We have $f \in [0,2]$, and the grid is sliced at $f=0.5$ in (b) and at $f=1.0$ in (c).}
    \label{fig:appendix:leverage}
    \vspace{-1em}
\end{figure}

Figure \ref{fig:appendix:3plots} (a) shows the relative error of Bernoulli sampling and our pivotal sampling both using the leverage score. As expected, our method shows a better fit than Bernoulli sampling again. 

\noindent \textbf{Samples Needed for a Certain Error.} Table \ref{table:2opt} and \ref{table:1.1opt} summarize the number of samples required to achieve $2 \times \text{OPT}$ error and $1.1 \times \text{OPT}$ error where OPT is the error we could obtain when all the data are labeled. In all four test problems, pivotal sampling achieves the target error with fewer samples than the existing method (independent leverlage scores sampling), which is denoted by Bernoulli in the tables. Our method is especially efficient when we aim at the target error close to OPT. For instance, to achieve the $1.1 \times \text{OPT}$ error in the 2D damped harmonic oscillator model, our method requires less than half samples that Bernoulli sampling requires, showing a significant reduction in terms of the number of samples needed.

\begin{table}[b]
\vspace{-1em}
\centering
\begin{tabular}{c|cccc}
     & Oscillator 2D & Heat Eq. & Surface Reaction & Oscillator 3D \\
     \hline
    n & $10000$ & $10000$ & $10000$ & $51^3$ \\
    poly. deg. & 20 & 20 & 20 & 10 \\
    \hline
    Bernoulli (a) & 574 & 554 & 545 & 671 \\
    Pivotal (b) & 398 & 395 & 390 & 533 \\
    \hline
    Efficiency (b / a) & 0.693 & 0.713 & 0.716 & 0.794
\end{tabular}
\caption{Number of samples needed to achieve $2 \times \text{OPT}$ error.}
\label{table:2opt}
\vspace{-1em}
\end{table}

\begin{table}[h!]
\vspace{-1em}
\centering
\begin{tabular}{c|cccc}
     & Oscillator 2D & Heat Eq. & Surface Reaction & Oscillator 3D \\
     \hline
    n & $10000$ & $10000$ & $10000$ & $51^3$ \\
    poly. deg. & 12 & 12 & 12 & 10 \\
    \hline
    Bernoulli (a) & 924 & 814 & 903 & 3121 \\
    Pivotal (b) & 450 & 442 & 492 & 1943 \\
    \hline
    Efficiency (b / a) & 0.487 & 0.523 & 0.545 & 0.623
\end{tabular}
\caption{Number of samples needed to achieve $1.1 \times \text{OPT}$ error.}
\label{table:1.1opt}
\vspace{-1em}
\end{table}

\begin{figure}[t]
    \centering
    \begin{subfigure}[t]{0.32\textwidth}
        \centering
        \includegraphics[width=.8\linewidth]{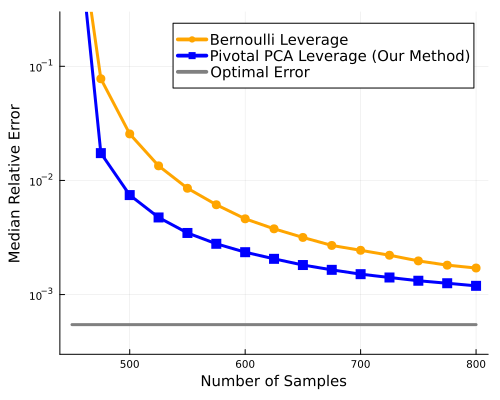}
        \vspace{-.5em}
        \caption{3D Damped Harmonic Oscillator.}    
    \end{subfigure}
    \hfill
    \centering
    \begin{subfigure}[t]{0.32\textwidth}
        \centering
        \includegraphics[width=.8\linewidth]{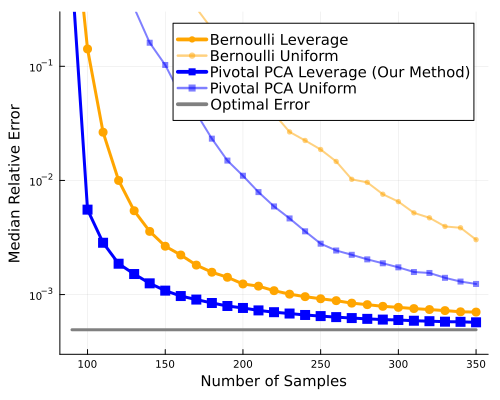} 
		% I named the file incorrectly. It says p15 but the actual polynomial degree is 12. (Atsushi)
        \vspace{-.5em}
        \caption{Leverage Score v.s. Uniform Sampling.}    
    \end{subfigure}
    \hfill
    \centering
    \begin{subfigure}[t]{0.32\textwidth}
        \centering
        \includegraphics[width=.8\linewidth]{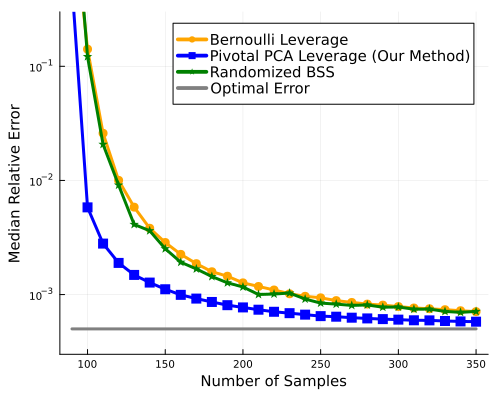}
                \vspace{-.5em}
        \caption{Randomized BSS.}    
    \end{subfigure}
    \caption{(a): Results for degree $12$ active polynomial regression for the damped harmonic oscillator QoI in 3D space. (b), (c): Degree $12$ active polynomial regression for the damped harmonic oscillator QoI in 2D space. (b) includes the performance of the Bernoulli sampling and our pivotal sampling both using uniform inclusion probability which is given with the thin lines. (c) demonstrates the approximation power of the randomized BSS algorithm. We run the sampling $2000$ times with different parameters, round the number of samples to the tens place, and take the median error.}
    \label{fig:appendix:3plots}
\end{figure}

\noindent \textbf{Leverage Score Sampling vs. Uniform Sampling.} We also conduct a complementary experiment to show empirically that leverage score sampling is much more powerful than uniform sampling. We extend the simulation in Figure \ref{fig:sec1:comparison} to a uniform inclusion probability setting. This time, we draw $350$ samples, repeat the simulation $100$ times, and report the approximation with a median error in Figure \ref{fig:appendix:approximationexample}. The result shows poor performance of the uniform sampling. As they draw fewer samples near the boundaries compared to leverage score sampling, they are not able to pin down the polynomial function near the edges, resulting in suffering large errors in these areas. The relative error plot is given in Figure \ref{fig:appendix:3plots} (b). We point to three observations. 1) Comparing the thick lines and the thin lines, one can tell that the use of leverage score significantly improves performance. 2) Comparing the orange lines and the blue lines, one can see that our \textit{spatially-aware} pivotal sampling outperforms the Bernoulli sampling. 3) By combining the leverage score and \textit{spatially aware} pivotal sampling (our method), one can attain the best approximation among the four sampling strategies.

\begin{figure}[h!]
    \vspace{-1em}
    \centering
    \begin{subfigure}[b]{0.24\textwidth}
        \centering
        \includegraphics[width=.9\linewidth]{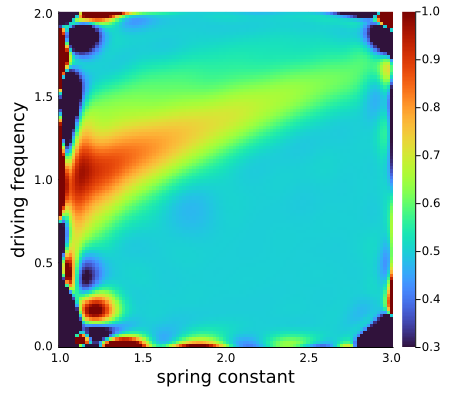}
                \vspace{-.5em}
        \caption{Bernoulli Uniform.}    
    \end{subfigure}
    \hfill
    \centering
    \begin{subfigure}[b]{0.24\textwidth}
        \centering
        \includegraphics[width=.9\linewidth]{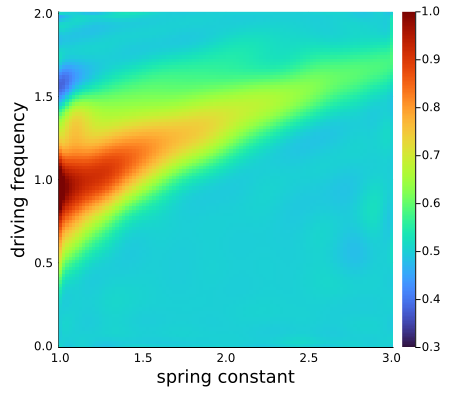}
                \vspace{-.5em}
        \caption{Bernoulli Leverage.}    
    \end{subfigure}
    \hfill
    \centering
    \begin{subfigure}[b]{0.24\textwidth}
        \centering
        \includegraphics[width=.9\linewidth]{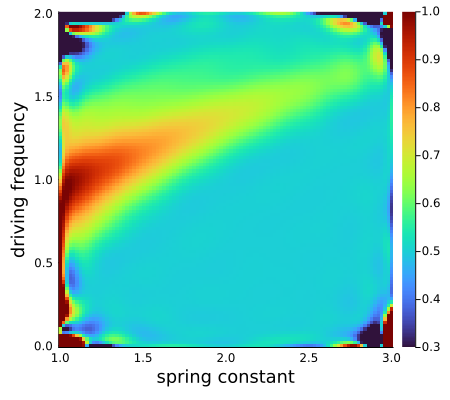}
                \vspace{-.5em}
        \caption{Pivotal Uniform.}    
    \end{subfigure}
    \hfill
    \centering
    \begin{subfigure}[b]{0.24\textwidth}
        \centering
        \includegraphics[width=.9\linewidth]{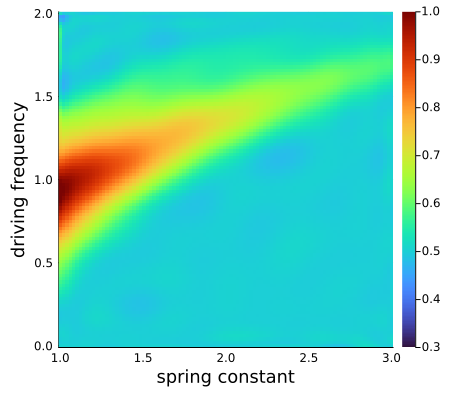}
                \vspace{-.5em}
        \caption{Pivotal Leverage.}    
    \end{subfigure}
    \caption{
    Visualizations of a polynomial approximation to the maximum displacement of a damped harmonic oscillator, as a function of driving frequency and spring constant. $\mathcal{X}$ is a uniform distribution over a box. All four draw $350$ samples but by different sampling methods; (a) and (b) use Bernoulli sampling but (c) and (d) use pivotal sampling. Also, (a) and (c) select samples with uniform probability while (b) and (d) employs the leverage score. Clearly, the leverage score successfully pins down the polynomial function near the boundaries, resulting in better approximation.}
    \label{fig:appendix:approximationexample}
    \vspace{-1em}
\end{figure}

\noindent \textbf{Randomized BSS Algorithm.} Finally, we run the randomized BSS algorithm \citep{CP19} on the 2D damped harmonic oscillator target function, and report the relative error together with Bernoulli leverage score sampling and our pivotal sampling in Figure \ref{fig:appendix:3plots} (c). The setting is the same as the Section \ref{sec:experiments} that the initial data points are drawn uniformly at random from $[1,3] \times [0,2]$, and the polynomial degree is set to $12$. Even though this algorithm also has a theoretical sample complexity of $O(d/\epsilon)$ for the regression guarantee in (\ref{eq:main_thm_gaur}), our sampling method achieves a certain relative error with fewer samples.
\end{document}

\maketitle

\begin{abstract}
    We present an active data collection method for sample efficient learning of high dimensional linear functions with spatial structure, such as high degree polynomials over low-dimensional surfaces. Our method uses leverage scores as a measure of data importance, and a non-i.i.d. pivotal sampling method to collect samples that are well-scattered across the domain. For $d$-dimensional linear function families, we prove that our approach solves the challenging agnostic learning problem with $O(d\log d)$ samples, matching leverage score based methods that use i.i.d. sampling. At the same time, our method performs far better empirically. On a set of test problems motivated by learning-based methods for solving parametric PDEs and uncertainty quantification, we reduce the number of samples required to achieve a given target accuracy by up to $50\%$ in comparison to existing methods. 
\end{abstract}

\section{Introduction}\label{sec:introduction}
In the active linear regression problem, we are given a data matrix $\mathbf{A} \in \R^{n \times d}$ with $n\gg d$ rows and query access to a target vector $\bv{b}\in \R^n$. The goal is to learn parameters $\bv{x}\in \R^d$ such that $\bv{A}\bv{x} \approx \bv{b}$ while observing as few entries in $\bv{b}$ as possible. 
% limited access to the target values in $\mathbf{b}$. That is, we choose samples $\mathcal{S} \subseteq [n]$, and $\bv{x}$ is learned using only the target values $\{b_i: i \in \mathcal{S}\}$ associated with those samples. The goal is to minimize $|\mathcal{S}|$ while still 
In this paper, we study this problem in the challenging agnostic learning or ``adversarial noise'' setting, where we do not assume any underlying relationship between $\mathbf{A}$ and $\mathbf{b}$. Instead, our goal is to find parameters competitive with the best possible fit, good or bad. Specifically, considering $\ell_2$ loss, let $\mathbf{x}^* = \argmin_{\mathbf{x}} \lVert \mathbf{A} \mathbf{x} - \mathbf{b} \rVert_2^2$ be optimal model parameters. We want to find $\tilde{\mathbf{x}}^*$ using a small number of queried target values in $\bv{b}$ such that 
\begin{align}
\label{main:gaur}
\begin{split}
     \lVert \mathbf{A} \tilde{\mathbf{x}}^* - \mathbf{b} \rVert_2^2 \leq (1 + \epsilon) \lVert \mathbf{A} \mathbf{x}^* - \mathbf{b} \rVert_2^2,
\end{split}
\end{align}
 for some error parameter $\epsilon > 0$.
Beyond being a fundamental learning problem, this active regression problem has emerged as a fundamental tool in learning based methods for the solution and uncertainty analysis of parametric partial differential equations (PDEs) \cite{esiam_2015,siam_review_2022}. For such applications, the agnostic setting is crucial, as a potentially complex quantity of interest is approximated by a simple surrogate function class (e.g. polynomials, sparse polynomials, single layer neural networks, etc.) \cite{sparse_review_2021,ridge_functions_2018}. Additionally, reducing the number of labels used for learning is crucial, as each label usually requires the computationally expensive numerical solution of a PDE for a new set of parameters \cite{cohen_devore_2015}.

\subsection{Leverage Score Sampling}\label{sec:lev_score}
Active linear regression has been studied for decades in the statistical model where $\bv{b}$ is assumed to equal $\bv{A}\bv{x}^*$ plus i.i.d random noise. In this case, the problem can be addressed using tools from optimal experimental design \cite{Pukelsheim:2006}. In the agnostic case, near-optimal sample complexity results were only obtained relatively recently using tools from non-asymptotic matrix concentration \cite{Tropp11}. In particular, it was shown independently in several papers that collecting entries from $\bv{b}$ \emph{randomly} with probability proportional to the \emph{statistical leverage scores} of rows in $\bv{A}$ can achieve the guarantee of \eqref{main:gaur} with $O(d\log d + d/\epsilon)$ samples \cite{Sarlos:2006,RauhutWard:2012,HamptonDoostan:2015,CohenMigliorati:2017}. The leverage scores are defined as follows:
\begin{definition}[Leverage Score]\label{def:leveragescore}
    Let $\bv{U}\in \R^{n \times r}$ be any orthogonal basis for the column span of a matrix $\mathbf{A} \in \R^{n \times d}$. Let $\mathbf{a}_i$ and $\mathbf{u}_i$ be the $i$-th rows of $\mathbf{A}$ and $\bv{U}$, respectively. The \textit{leverage score} $\tau_i$ of the $i$-th row in $\mathbf{A}$ can be equivalently written as:
    \begin{align}\begin{split}
    \label{eq:lev_defs}
        \tau_i = \lVert \mathbf{u}_i \rVert_2^2 = \mathbf{a}_i^T (\mathbf{A}^T \mathbf{A})^{-1} \mathbf{a}_i = \max_{\mathbf{x} \in \R^d} \frac{(\mathbf{a}_i^T \mathbf{x})^2}{\lVert \mathbf{A} \mathbf{x} \rVert_2^2}
    \end{split}\end{align}
    Notice that since $\tau_i = \lVert \mathbf{u}_i \rVert_2^2$, $\sum_{i=1}^n \tau_i = d$ when $\mathbf{A}$ is full-rank and thus $\bv{U}$ has $r =d$ columns. 
    \end{definition}

We refer the reader to \cite{pmlr-v70-avron17a} for a short proof of the final equality in \eqref{eq:lev_defs}. This last definition, based on a maximization problem, gives an intuitive understanding of the leverage scores. The score of a row $\bv{a}_i$ is higher if it is more ``unique'', meaning that it is possible to find a vector $\bv{x}$ that has large inner product with $\bv{a}_i$, but relatively low average inner product (captured by $\|\bv{A}\bv{x}\|_2^2$) with all other rows in the matrix. More unique rows need to be sampled with higher probability.

Prior work considers independent leverage score sampling, either with or without replacement. The typical approach for sampling without replacement, which we call ``Bernouilli sampling'' is as follows: Each row $\bv{a}_i$ is assigned a probability $p_i = \min(1,c\cdot \tau_i)$ for an oversampling parameter $c \geq 1$. Then each row is sampled independently with probability $p_i$. We construct a subsampled data matrix $\tilde{\bv{A}}$ and subsampled target vector $\tilde{\bv{b}}$ by adding $\bv{a}_i/\sqrt{p_i}$ to $\tilde{\bv{A}}$ and $b_i/\sqrt{p_i}$ to $\tilde{\bv{b}}$ for any index $i$ that is sampled. To solve the active regression problem, we return $\tilde{\bv{x}}^* = \argmin_{\bv{x}} \|\tilde{\bv{A}}\bv{x} - \tilde{\bv{b}}\|_2^2$.

\subsection{Our Contributions}\label{subsec:contributions}
In applications to PDEs, our goal is often to approximate a function over a collection of points drawn from a low dimensional distribution $\mathcal{X}$. E.g. $\mathcal{X}$ might be uniform over an interval $[-1,1] \subset \R$ or over a box $[-1,1]\times\ldots \times [-1,1] \subset \R^q$. In this setting, the length $d$ rows of $\bv{A}$ correspond to feature transformations of samples from $\mathcal{X}$. For example, in the ubiquitous task of polynomial regression, we start with  $\bv{x}\sim \mathcal{X}$ and add to $\bv{A}$ a row containing all combinations of entries in $\bv{x}$ with total degree $p$, i.e., $x_1^{\ell_1}x_2^{\ell_2}, \ldots x_q^{\ell_q}$ for all choices of non-negative integers $\ell_1, \ldots, \ell_q$ such that $\sum_{i=1}^q \ell_i \leq p$.

For such problems, random sampling methods compete with ``grid'' based interpolation, where the target 
$\bv{b}$ is queried on a deterministic structured grid tailored to $\mathcal{X}$. For example, when $\mathcal{X}$ is uniform on a box, the standard approach is to use a grid based on the Chebyshev nodes \cite{Xiu2016}. Pictured in Figure \ref{fig:sampledist}, the Cheybshev grid concentrates samples near the boundaries of the box, avoiding the well known issue of Runge's phenomenon for uniform grids. As shown in the same figure, leverage score sampling also leads to more samples near the boundaries. In fact, the methods are closely related: it can be shown that, in the high degree limit, the leverage score for polynomial regression over the box matches the asymptotic density of the Chebyshev nodes \cite{sparse_review_2021}.

So how do the deterministic and randomized methods compare? The advantage of randomized methods based on leverage score sampling is that they yield strong provable approximation guarantees, and easily generalize to any distribution $\mathcal{X}$. Deterministic methods are less flexible on the other hand, and do not yield provably guarantees\footnote{As discussed in \cite{CP19}, no deterministic method can provably solve the agnostic regression problem with few samples. Since we make no assumptions on $\bv{b}$, all error in $\bv{A}\bv{x}^* - \bv{b}$ could be concentrated only at the deterministic indices to be selected. Randomization is needed to avoid high error outliers in $\bv{b}$}. However, when they can be applied, the advantage of grid based methods is that they more ``evenly'' distribute samples over the original data domain, which can lead to better performance in practice. Randomized methods are prone to ``missing'' larger regions of $\mathcal{X}$'s support, as shown in Figure \ref{fig:sampledist}. The driving question behind our work is:

\begin{myquote}{2em}
    Is it possible to obtain the ``best of both worlds'' for fitting functions over low-dimensional domains? I.e., can we achieve the strong theoretic guarantees of leverage score sampling in the agnostic setting with a method that produces spatially well-distributed samples? 
\end{myquote}

\begin{figure}[t]
    \centering
    \begin{subfigure}[b]{0.32\textwidth}
        \centering
        \includegraphics[width=.7\linewidth]{NeurIPS/plots/sec1_bernoulli_leverage_sampledist.png}
        \caption{Bernoulli Sampling}    
    \end{subfigure}
    \hfill
    \centering
    \begin{subfigure}[b]{0.32\textwidth}
        \centering
        \includegraphics[width=.7\linewidth]{NeurIPS/plots/sec1_PCA_leverage_sampledist.png}
        \caption{Pivotal Sampling (our method)}    
    \end{subfigure}
    \hfill
    \centering
    \begin{subfigure}[b]{0.32\textwidth}
        \centering
        \includegraphics[width=.7\linewidth]{NeurIPS/plots/sec1_chebyshev_sampledist.png}
        \caption{Chebyshev Grid}    
    \end{subfigure}
    \caption{The results of three different active learning methods used to collect samples to fit a polynomial over $[-1,1]\times [-1,1]$. The image on the left was obtained by collecting points independently at random with probability according to their statistical leverage scores. The image on the right was obtained by collecting samples at the 2-dimensional Chebyshev nodes. The image in the middle shows our method, which collects samples according to leverage scores, but using a negatively associated pivotal sampling strategy that ensures samples are evenly spread in spatially.}   
    \vspace{-1em}
    \label{fig:sampledist}
\end{figure}

In this paper, we take a step towards answering this question in the affirmative. In particular, instead of sampling rows from $\bv{A}$ \emph{independently} with probability proportional to the leverage scores, we propose to use a well developed tool from the survey sampling literature known as
pivotal sampling \cite{deville1998unequal}.
% \cite{GLS12}. 
Our specific version of pivotal sampling is designed to be \textit{spatially-aware}, meaning that it covers the entire domain in a well-balanced manner, while the marginal probabilities remain proportional to the leverage score. Details are discussed in Section \ref{sec:preliminaries} and \ref{sec:analysis}. At a high-level, the pivotal method is a ``competition'' based sampling approach, where candidate rows to be sampled compete in a binary tree tournament. By structuring the tournament so that spatially close points compete at lower levels of the tree, we can ensure a better spatial spread than Bernoulli leverage score sampling. Our main contribution is to show that this sampling scheme performs significantly better in practice, while matching the complexity of independent Bernoulli leverage score sampling in theory. 

On the practice side, we offer Figure \ref{fig:sec1:comparison} as an example from a simple PDE test problem. In comparison to independent leverage score sampling, our spatially-aware method obtains a much better approximation to the target function for a fixed number of examples. More details about this simulation and others are provided in Section \ref{sec:experiments}.
We consider two different spatially aware methods, one based on PCA that is adaptable to a wide variety of distributions $\mathcal{X}$, and another based on a coordinate-wise splitting that is easier to implement and also performs very well, but is tailored to uniform domains.  

On the theory side, we prove the following general bound:

\begin{theorem}[Main Result]\label{thm:main}
    Let $\mathbf{A} \in \R^{n \times d}$ be a data matrix and $\mathbf{b} \in \R^n$ be a target vector. Consider any algorithm which samples exactly $k$ rows from $\bv{A}$ (and observes the corresponding entries in $\bv{b}$) from a distribution that 1) satisfies one-sided $\ell_{\infty}$ independence (Defn. \ref{def:l-infty}) with constant parameter and 2) the marginal probability of sampling any row $\bv{a}_i$ is proportional to $\tau_i$.\footnote{Formally, we require that the marginal probabilities equal $\min(1,c\tau_i)$ for a fixed constant $c \geq 1$.} Let $\tilde{\mathbf{A}}$ and $\tilde{\mathbf{b}}$ be the scaled sampled data and target, as defined in Section \ref{sec:lev_score}, and let $\tilde{\mathbf{x}}^* = \argmin_{\mathbf{x} \in \R^d} \lVert \tilde{\mathbf{A}} \mathbf{x} - \tilde{\mathbf{b}} \rVert_2^2$.
  As long as $k \geq c\cdot \left(d\log d + \frac{d}{\epsilon}\right)$ for a fixed positive constant $c$, then with probability $99/100$, 
    \begin{align}\begin{split}
    \label{eq:main_thm_gaur}
        \lVert \mathbf{A} \tilde{\mathbf{x}}^* - \mathbf{b} \rVert_2^2 \leq (1 + \epsilon) \lVert \mathbf{A} \mathbf{x}^* - \mathbf{b} \rVert_2^2
    \end{split}\end{align}
\end{theorem}
The definition of one-sided $\ell_{\infty}$ independence is given in Section \ref{sec:analysis}. This is a particularly weak notation introduced in \cite{KKS22} that is implied by many standard notations of negative dependence between random variables, including conditional negative association (CNA) and the strongly Rayleigh property \cite{permantle_survey}. We can show that binary-tree-based pivotal  sampling satisfies the property, either indirectly by using that pivotal sampling is strongly Rayleigh, or with a more direct proof from first principals. 
% The definition of the Conditional Pairwise Negative Correlation is given in . This is a particular notation of negative dependence between random variables that is closely related to other such definitions \cite{permantle_survey}. For example, as will be discussed, any distribution that satisfies conditional negative association (CNA), which includes all Strongly Rayleigh distributions, satisfies the condition. Later in Section \ref{sec:analysis}, we show that any fixed-order binary-tree-based pivotal sampling algorithm 
%   pairwise negative correlation. 
  So, as a corollary, we obtain:

\begin{corollary}
\label{corr:main}
    The spatially-aware pivotal sampling methods introduced in Section \ref{sec:methods} (which use a fixed binary tree) return with probability $99/100$ a vector $\tilde{\mathbf{x}}^*$ satisfying $\lVert \mathbf{A} \tilde{\mathbf{x}}^* - \mathbf{b} \rVert_2^2 \leq (1 + \epsilon) \lVert \mathbf{A} \mathbf{x}^* - \mathbf{b} \rVert_2^2$ while only observing $O\left(d\log d + \frac{d}{\epsilon}\right)$ entries in $\bv{b}$.  
\end{corollary}

% The details of spatially-aware pivotal sampling are discussed in Sections \ref{sec:preliminaries} and \ref{sec:analysis}. At a high-level, the pivotal method is a ``competition'' based sampling approach, where candidates rows to be sampled compete in a binary tree tournament. By structuring the tournament so that spatial close points compete at lower levels of the tree, we can ensure a better spatial spread than Bernoulli leverage score sampling. The result is a method that tends to work significantly better in practice. 
% Figure \ref{fig:sec1:comparison} shows an example of the effectiveness of our sampling. Both (b) and (c) are approximations of (a) with the same polynomial degree and the same number of samples using leverage score but by different sampling methods. Clearly, our method shown in (c) outperforms the Bernoulli sampling in (b). More details about the simulations are provided in Section \ref{sec:experiments}.

\begin{figure}[t]
    \centering
    \begin{subfigure}[b]{0.32\textwidth}
        \centering
        \includegraphics[width=.9\linewidth]{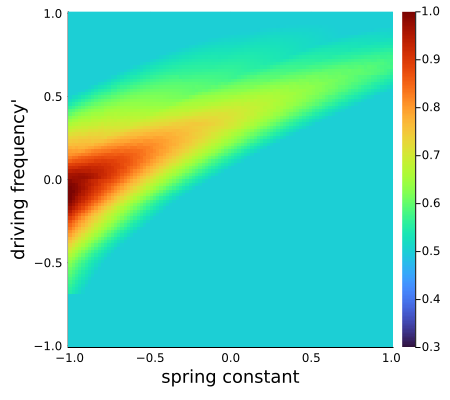}
        \caption{Target Function}    
    \end{subfigure}
    \hfill
    \centering
    \begin{subfigure}[b]{0.32\textwidth}
        \centering
        \includegraphics[width=.9\linewidth]{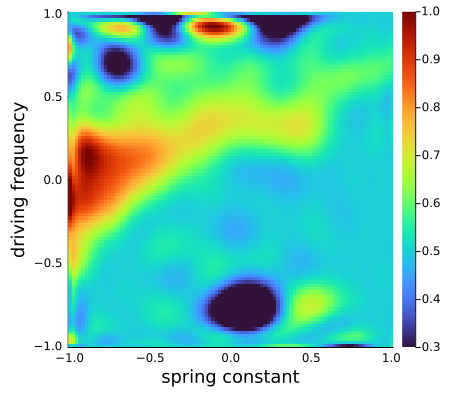}
        
        \caption{Bernoulli Sampling}    
    \end{subfigure}
    \hfill
    \centering
    \begin{subfigure}[b]{0.32\textwidth}
        \centering
        \includegraphics[width=.9\linewidth]{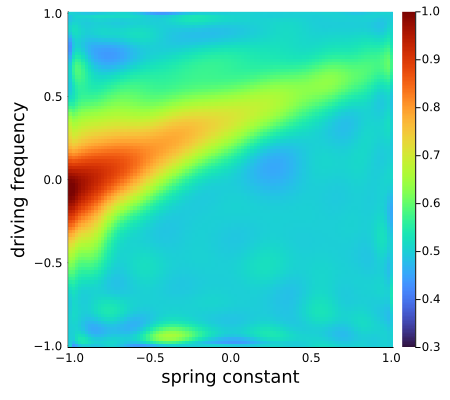}
        \caption{Pivotal Sampling (our method)}    
    \end{subfigure}
    \caption{Visualizations of a polynomial approximation to the maximum displacement of a damped harmonic oscillator, as a function of driving frequency and spring constant. $\mathcal{X}$ is a uniform distribution over a box. (a) is the target value, and samples can be obtained through the numerical solution of the differential equation governing the oscillator. Both (b) and (c) draw $250$ samples using leverage score sampling and perform polynomial regression of degree $20$. However, (b) uses Bernoulli sampling while (c) employs our binary-tree-based pivotal sampling.  We run the sampling 100 times and show the result with a median error. Clearly, the pivotal sampling results in a better approximation, avoiding the significant artifacts that result from gaps between samples in the Bernoulli based solution.}
            \vspace{-1.5em}
    \label{fig:sec1:comparison}
\end{figure}

% Pivotal sampling is a sequence of comparisons of two data points. In our method, the order of comparisons is not adaptive but decided deterministically given the data matrix $\mathbf{A}$, thus, we call it fixed-order pivotal sampling. We want the inclusion probability to be proportional to the leverage score, but as we want to select $k > d = \sum_{i=1}^n \tau_i$ samples where $\tau_i \leq 1$, the inclusion probability could exceed $1$ which causes an issue in pivotal sampling. To address this issue, we perform a probability ceiling. The detail is provided in Section \ref{sec:preliminaries}.

Our proof of Theorem \ref{thm:main} is discussed in Section \ref{sec:analysis}. It requires two main ingredients: 1) a recent result on matrix Chernoff bounds for distributions satisfying an $\ell_{\infty}$-independence property \cite{KKS22}, and 2) a novel analysis of an approximate matrix-multiplication method from \cite{DrineasKannanMahoney:2006} for distributions with this property. 

% In Section \ref{sec:experiments}, we show a couple of experimental results where we consider leverage score and uniform distribution as two candidates for marginal probability, and Bernoulli sampling and pivotal sampling as the sampling scheme. In all the simulations, the leverage score outperforms uniform distribution, and the pivotal sampling beats the Bernoulli sampling. Moreover, one can obtain the strongest result by combining the leverage score and our pivotal sampling.

\subsection{Related Work}\label{subsec:relatedwork}
\vspace{-.5em}
The application of leverage score sampling to the agnostic active regression problem has received significant recent attention. Beyond the results discussed above, extensions of leverage score sampling have been studied for norms beyond $\ell_2$ \cite{pmlr-v134-chen21d,focs2022,MeyerMuscoMusco:2023,parulekar_et_al}, in the context where the sample space is infinite (i.e. $\bv{A}$ is an operator with infinite rows) \cite{tamas_2020,universal_sampling}, and for functions that involve non-linear transformations \cite{GHM22,MaiRaoMusco:2021,MunteanuSchwiegelshohnSohler:2018}.

Theoretical improvements on leverage score sampling have also been studied. Notable is a recent result that improves on the $O(d\log d + d/\epsilon)$ bound for leverage score sampling by a $\log d$ factor, showing that the active least squares regression problem can be solved with $O(d/\epsilon)$ samples \cite{CP19}. This is provably optimal. However, the algorithm in \cite{CP19} relies on heavy tools from the graph sparsification literature \cite{BSS12,LS18} and appears to involve large constant factors. In our initial experiments, it was not competitive with leverage score sampling in practice. 
There have also been a few efforts, that like ours, seek to develop practical improvements on leverage score sampling.  \cite{lev_volume_neurips2018} studies a variant of volume sampling that matches the theoretical guarantees of leverage score sampling, but appears to perform better experimentally. Unlike our approach, however, this method does not explicitly take into account spatial-structure in the underlying regression problem. 

While pivotal sampling has not been studied in the context of agnostic active regression, it is widely used in other applications, and its negative dependence properties have been studied extensively. \cite{DJR05} proves that a limited sequential variant of pivotal sampling satisfies the negative association (NA) property. \cite{BBL08} introduced the notion of a strongly Rayleigh distribution and proved that it implies a stronger notion of conditional negative association (CNA), and \cite{BJ12} showed that pivotal sampling run with an arbitrary binary tree is strongly Rayleigh. It follows that the method satisfies CNA.
\cite{GWBW22} discusses an efficient algorithm for pivotal sampling by parallelization and careful manipulation of inclusion probabilities. Another variant of pivotal sampling that is \textit{spatially-aware} is proposed in \cite{GLS12}. Though their approach is out of the scope of our analysis as it involves randomness in the competition order used during sampling, our method is inspired by their work.
 
Both scalar and matrix concentration bounds for negatively dependent distributions have been studied extensively \cite{PP14,BC19}. Most relevant to our work are matrix Chernoff bounds, which have been studied e.g. in the context of understanding graph sparsification methods based on sampling random spanning trees \cite{KS18}. A recent result of \cite{KKS22} proves a matrix-Chernoff bound that matches the bounds available from i.i.d. sampling \cite{Tropp11}, and only requires the sampling distribution to obey the relatively weak $\ell_{\infty}$-independence property, which holds e.g. for all strongly Rayleigh distributions.
% , and as we show, for all distributions satisfying our conditional negative correlation property. 
% proven in  and discussed its application to random spanning tree problems. Their matrix bound was actually a bit weaker than the matrix bound for independent distributions by \cite{Tropp11}, but the later work \cite{KKS22} for $\ell_{\infty}$-independent distribution improves it to match the tightness of the independent version. 

% Negative association and the newly introduced notion of $\ell_{\infty}$-independence were bridged by \cite{PP14} where a McDiarmid-type bound was proved for distributions satisfying the stochastic covering property. \cite{BC19} proved a scalar Bernstein inequality for negatively associated distributions, which is as tight as the one for independent distributions, with an application to Poisson survey scheme. Together with an $\epsilon$-net argument, this bound would enable us to establish $O(d^2)$ sampling complexity in our method. Instead,  in the following, we use the matrix Chernoff-type bound in \cite{KKS22} to obtain a $O(d \log d)$ bound. .

\section{Preliminaries}\label{sec:preliminaries}

\textbf{Notation.} We let $[n]$ denote $\{1, \cdots, n\}$. $\E[X]$ denotes the expectation of a random variable $X$. We use bold lower-case letters for vectors and bold upper-case letters for matrices. For a vector $\mathbf{z} \in \R^n$ with entries $z_1, \cdots, z_n$, $\lVert \mathbf{z} \rVert_2 = (\sum_{i=1}^n z_i^2)^{1/2}$ denotes the Euclidean norm of $\mathbf{z}$.
Given a matrix $\mathbf{A} \in \R^{n \times d}$, we let $\mathbf{a}_i$ denote the $i$-th row, and $a_{ij}$ denote the entry in the $i$-th row and $j$-th column. 
% The spectral norm of $\mathbf{A}$ is given as $\lVert \mathbf{A} \rVert_2 = \max_{\mathbf{x} \in \R^d} \frac{\lVert \mathbf{A} \mathbf{x} \rVert_2}{\lVert \mathbf{x} \rVert_2}$. 

\noindent \textbf{Importance sampling.} All of the methods studied in this paper solve the active regression problem by collecting a single random sample of rows in $\bv{A}$ and corresponding entries in $\bv{b}$. We introduce a vector of binary random variables $\bm{\xi} = \{\xi_1, \cdots, \xi_n\}$, where $\xi_i$ is $1$ is $\bv{a}_i$ (and thus $b_i$) is selected, and $0$ otherwise. $\xi_1, \cdots, \xi_n$ will not necessarily be independent depending on our sampling method. 

Given a sampling method, let $p_i = \E[\xi_i]$ denote the marginal probability that row $i$ is selected. We return an approximate regression solution as follows: let $\tilde{\bv{A}}\in \R^{k\times d}$  contain $\bv{a}_i/\sqrt{p_i}$ for all $i$ such that $\xi_i = 1$, and similarly let $\tilde{\bv{b}}\in \R^k$ contain $b_i/\sqrt{p_i}$ for the same values of $i$. This scaling ensures that, for any fixed $\bv{x}$, $\E\|\tilde{\bv{A}}\bv{x} - \tilde{\bv{b}}\|_2^2 = \|{\bv{A}}\bv{x} - {\bv{b}}\|_2^2$. To solve the active regression problem, we return $\tilde{\mathbf{x}}^* = \argmin_{\mathbf{x} \in \R^d} \lVert \tilde{\mathbf{A}} \mathbf{x} - \tilde{\mathbf{b}} \rVert_2^2$. Computing this estimate only requires querying $k$ target values in $\bv{b}$. 

\noindent \textbf{Leverage Score Sampling.} 
We consider methods that choose the marginal probabilities proportional to $\bv{A}$'s leverage scores. Specifically, $p_i = \min(1, c\cdot \tau_i)$. The constant  $c > 1$ is a chosen oversampling parameter that controls the number of samples taken. If we set $c = k/d$, then we can bound the expected number of samples by $\sum_{i=1}^n p_i \leq \sum_{i=1}^n c \cdot\tau_i \leq \frac{k}{d}\cdot d = k$. The last step uses that the leverage scores sum to at most $d$. However, in practice, it is not uncommon that $p_i = \min(1, c\cdot \tau_i)$ will take value $1$, in which case the probabilities will sum to much less than $k$. To avoid this issue, we perform a probability preprocessing step, which finds $c_k$ such that
$\sum_{i=1}^n \min(1, c_k\cdot \tau_i) = k$ exactly.

We let $\tilde{p}_i$ denote $\tilde{p}_i = \min(1, c_k\cdot \tau_i)$ the probabilities after scaling by $c_k$ and thresholding at $1$. These are the final marginal probabilities used in our sampling methods. Using these probabilities ensures that we always collect $k$ samples in expectation, and actually, for our pivotal sampling methods, we always take exactly $k$ samples. Details of the preprocessing step are included in Appendix \ref{app:preprocess}

\begin{figure}[b]
\vspace{-2em}
\begin{algorithm}[H]\caption{Binary Tree Based Pivotal Sampling \cite{deville1998unequal}}\label{algo:pivotal}
    \begin{algorithmic}[1]
        \Require Depth $t$ full binary tree $T$ with $n$ leaves, inclusion probabilities $\{\tilde{p}_1, \cdots, \tilde{p}_n\}$ for each leaf.
        \Ensure Set of $k$ sampled indices $\mathcal{S}$
        \State Initialize $\mathcal{S} = \emptyset$.
        \While{$T$ has at least two remaining children nodes}
            \State Select any pair of sibling nodes $S_1,S_2$ with parent $P$. Let $i,j$ be the indices stored at $S_1,S_2$.
                \If{$\tilde{p}_i + \tilde{p}_j \leq 1$}
                    \State With probability $\frac{\tilde{p}_i}{\tilde{p}_i + \tilde{p}_j}$, set $\tilde{p}_i \gets \tilde{p}_i + \tilde{p}_j$, $\tilde{p}_j \gets 0$. Store $i$ at $P$.
                    \State Otherwise, set $\tilde{p}_j \gets \tilde{p}_i + \tilde{p}_j$, $\tilde{p}_i \gets 0$. Store $j$ at $P$.
                \ElsIf{$\tilde{p}_i + \tilde{p}_j > 1$}
                    \State With probability $\frac{1-\tilde{p}_i}{2-\tilde{p}_i-\tilde{p}_j}$, set $\tilde{p}_i \gets \tilde{p}_i+\tilde{p}_j-1$, $\tilde{p}_j=1$. Store $i$ at $P$ and set $\mathcal{S} \gets \mathcal{S} \cup \{j\}$.
                    \State Otherwise, set $\tilde{p}_j \gets \tilde{p}_i+\tilde{p}_j-1$, $\tilde{p}_i \gets 1$. Store $j$ at $P$ and set $\mathcal{S} \gets \mathcal{S} \cup \{i\}$.
                \EndIf
                \State Remove $S_1,S_2$ from $T$.
        \EndWhile
        \Return $\mathcal{S}$
    \end{algorithmic}
\end{algorithm}
\vspace{-2em}
\end{figure}

\vspace{-.5em}
\section{Our Methods}\label{sec:methods}
\vspace{-.5em}

In this section, we present our sampling scheme which consists of two steps; deterministically constructing a binary tree, and choosing samples by running the pivotal method on this tree. The pivotal method is described in Algorithm \ref{algo:pivotal}. It takes as input a binary tree with $n$ leaf nodes, each corresponding to a single index $i$ to be sampled. For each index, we also have an associated probability $\tilde{p}_i$. The algorithm collects a set of exactly $k$ samples $\mathcal{S}$ where $k = \sum_{i=1}^n \tilde{p}_i$. It does so by percolating up the tree and performing repeated head-to-head comparisons of the indices at sibling nodes in the tree. After each comparison, one node promotes to the parent node with updated inclusion probability, and the other node is determined to be sampled or not to be sampled.

It can be checked that, after running Algorithm \ref{algo:pivotal}, index $i$ is always sampled with probability $\tilde{p}_i$, regardless of the choice of $T$. However, the samples collected by the pivotal method are not independent, but rather negatively correlated: siblings in $T$ are unlikely to both be sampled, and in general, the events that close neighbors in the tree are both sampled are negatively correlated. In particular, if index $i$ could at some point compete with an index $j$ in the pivotal process, the chance of selecting $j$ decreases if we condition on $i$ being selected. We take advantage of this property to generate spatially distributed samples by constructing a binary tree that matches the underlying geometry of our data. In particular, assume we are given a set of points $\bv{X}\in \R^{n\times d'}$. $\bv{X}$  will eventually be used to construct a regression matrix $\bv{A}\in \R^{n\times d}$ via feature transformation (e.g. by adding polynomial features). However, we construct the sampling tree using  $\bv{X}$ alone.

\begin{figure}[t]
\vspace{-2em}
\begin{algorithm}[H]\caption{Binary Tree Construction by Coordinate or PCA Splitting}\label{algo:btree}
    \begin{algorithmic}[1]
        \Require Matrix $\mathbf{X} \in \R^{n \times d'}$, split method $\in \{\text{PCA, coordinate}\}$, inclusion probabilities $\tilde{p}_1, \cdots, \tilde{p}_n$.
        \State $\mathcal{R} = \{i \in [n]; \tilde{p}_i<1\}$.
        \State Create a tree $T$ with a single root node. Assign set $\mathcal{R}$ to the root.
        \While{There exists a node in $T$ that holds set $\mathcal{K}$ such that $|\mathcal{K}|>1$}
        \State Select any such node $N$ and let $t$ be its level in the tree. Construct $\mathbf{X}_{(\mathcal{K})} \in \R^{|\mathcal{K}| \times d'}$.
                \If{split method $=$ PCA}
                    \State Sort $\mathbf{X}_{(\mathcal{K})}$ according to the direction of the maximum variance.
                \ElsIf{split method $=$ coordinate}
                    \State Sort $\mathbf{X}_{(\mathcal{K})}$ according to values in its $((t \,\, \mathrm{mod} \,\, d')+1)$-th column.
                \EndIf
                \State Create a left child of $N$. Assign to it all indices associated with the first $\lfloor\frac{|\mathcal{K}|}{2}\rfloor$ rows of $\mathbf{X}_{(\mathcal{K})}$.
                \State Create a right child of $N$. Assign to it the all remaining indices associated with rows in $\mathbf{X}_{(\mathcal{K})}$.
                \State Delete $\mathcal{K}$ from $N$.
        \EndWhile
        \Return $T$
    \end{algorithmic}
\end{algorithm}
\vspace{-1.5em}
\end{figure}

% The deletion in line $12$ is for Algorithm \ref{algo:pivotal}, which requires all the non-leaf nodes to be an empty. 
\begin{figure}[tb]
    \centering
    \includegraphics[width=.7\linewidth]{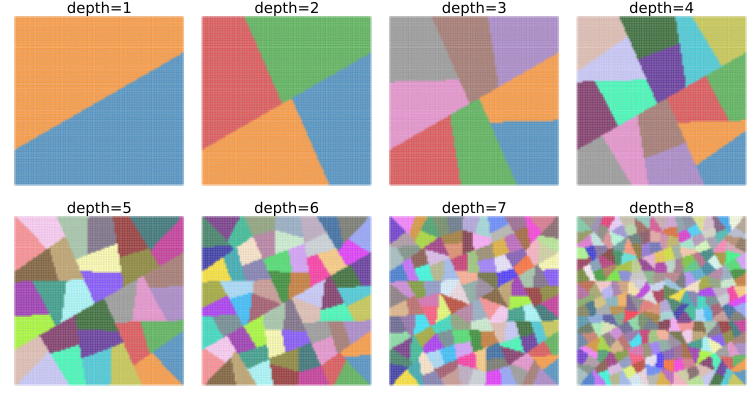}
        \vspace{-.7em}
    \caption{Visualization of a binary tree constructed via our Algorithm \ref{algo:btree} using the PCA method for a matrix $\bv{X}\in\R^{n\times 2}$ containing points on a uniform square grid. For each depth, data points are given the same color if they compose a subtree with root at that depth. As we can see, the method produces uniform recursive spatial partitions, which encourage spatially separated samples.}
    \label{fig:sec3:binarytree}
    \vspace{-1em}
\end{figure}

Our tree construction method is given as Algorithm \ref{algo:btree}. $\bv{X}_{\mathcal{K}}$ denotes the subset of rows of $\bv{X}$ with indices in the set $\mathcal{K}$. First, the algorithm eliminates all data points with inclusion probability $\tilde{p}_i = 1$. Next, it recursively partitions the remaining data points into two subgroups of the same size until all the subgroups have only one data point. Our two methods, PCA-based and coordinate-wise, only differ in how to partition. The PCA-based method performs principal component analysis to find the direction of the maximum variance and splits the space by a hyperplane orthogonal to the direction so that the numbers of data points on both sides are equal. The coordinate-wise version takes a coordinate (corresponding to a column in $\bv{X}$) in cyclic order and divides the space by a hyperplane orthogonal to the chosen coordinate.
An illustration of the PCA-based binary tree construction run on a fine uniform grid of data points in $\R^2$ is shown in Figure \ref{fig:sec3:binarytree}. Note that our tree construction method ensures the the number of indices assigned to each subgroup (color) at each level is equal to with $\pm 1$ point. As such, we end of with an even partition of data points into spatially correlated sets. 
Two indices will be more heavily negatively correlated if they lie in the same set at a higher depth number.

\section{Theoretical Analysis}
\label{sec:analysis}
As will be shown experimentally, when using probabilities $\tilde{p}_1, \ldots, \tilde{p}_n$ that are proportional to the statistical leverage scores of $\bv{A}$, our tree-based pivotal method from the previous section significantly outperforms Bernoulli leverage score sampling for active regression. Theoretically justifying this improvement in performance without further assumptions on $\bv{A}$ and its relation to the pre-feature transformation data matrix $\bv{X}$ seems challenging. However, we are able to prove the next best thing: even though it produces non-independent samples, our pivotal method will never perform \emph{worse} than Bernouilli sampling. In particular, it matches the $O(d\log d + d/\epsilon)$ sample complexity of independent leverage score sampling. This result is stated as Theorem \ref{thm:main}. Due to space limitations, its proof is relegated to Appendix \ref{sec:proofs}. However, we outline our main approach here. 

Following existing proofs for independent random sampling (e.g. \cite{Woodruff:2014} Theorem \ref{thm:main} requires two main ingredients: a subspace embedding result, and an approximate matrix-vector multiplication result. In particular, let $\bv{U}\in \R^{n\times d}$ be any orthogonal span for the columns of $\bv{A}$. Let $\bv{S} \in \R^{k\times n}$ be a subsampling matrix that contains a row for every index $i$ selected by our sampling scheme, which has value ${1}/{\sqrt{\tilde{p}_i}}$ at entry $i$, and is $0$ everywhere else. So, in the notation of  Theorem \ref{thm:main}, $\tilde{\bv{A}} = \bv{S}\bv{A}$ and $\tilde{\bv{b}} = \bv{S}\bv{b}$. To prove the theorem, it suffices to show that with high probability, 
\begin{enumerate}
\item \textbf{Subspace Embedding:} For all $\bv{x}\in \R^d$, $\frac{1}{2}\|\bv{x}\|_2 \leq \|\bv{S}\bv{U}\bv{x}\|_2 \leq 1.5\|\bv{x}\|_2$.
\item \textbf{Approximate Matrix-Vector Multiplication}: $\| \mathbf{U}^T \mathbf{S}^T \mathbf{S} (\mathbf{b} - \mathbf{A} \mathbf{x}^*) \|_2^2 \leq \epsilon \| \mathbf{b} - \mathbf{A} \mathbf{x}^* \|_2^2$.
\end{enumerate}
The first property is equivalent to $\|\bv{U}^T\bv{S}^T\bv{S}\bv{U} - \bv{I}\|_2 \leq 1/2$. I.e., after subsampling, $\bv{U}$ should remain nearly orthogonal. The second property requires that after subsampling with $\bv{S}$, the optimal residual, $\bv{b} - \bv{A}\bv{x}^*$, should have small product with $\bv{U}$. Note that without subsampling,  $\|\bv{U}^T(\mathbf{b} - \mathbf{A} \mathbf{x}^*)\|_2^2 = 0$. 

We show that both of the above bounds can be established for any sampling method that 1) samples index $i$ with marginal probability proportional to the $i$-th statistical leverage score 2) is homogeneous, meaning that it always takes a fixed number of samples $k$, and 3) produces a distribution over binary vectors satisfying the following property with $D_{\text{inf}}$ equal to a fixed constant:
\begin{restatable}[One-sided $\ell_{\infty}$-independence]{definition}{linfinityind}\label{def:l-infty}
Let $\xi_1, \cdots, \xi_n \in \{0,1\}^n$ be random variables with joint distribution $\mu$. Let $\mathcal{S}\subseteq [n]$ and let $i,j \in [n] \backslash \mathcal{S}$. The one-sided influence matrix $\mathcal{I}_{\mu}^{\mathcal{S}}$ is defined as:\vspace{-1em}
    \begin{align*}\begin{split}
        \mathcal{I}_{\mu}^{\mathcal{S}}(i, j) = \Pr_{\mu} [\xi_j=1 | \xi_i = 1 \wedge \xi_{\ell} = 1 \forall \ell \in \mathcal{S}] - \Pr_{\mu} [\xi_j=1 | \xi_{\ell} = 1 \forall \ell \in \mathcal{S}]
    \end{split}\end{align*}
    We say $\mu$ is one-sided $\ell_{\infty}$-independent with parameter $D_{\text{inf}}$ if, for all subsets $\mathcal{S} \subset [n]$,
    \begin{align*}\begin{split}
        \lVert \mathcal{I}_{\mu}^{\mathcal{S}} \rVert_{\infty} = \max_{i \in [n]} \sum_{j \in [n]} | \mathcal{I}_{\mu}^{\mathcal{S}}(i, j) | \leq D_{\text{inf}}
    \end{split}\end{align*}
    Note that if $\xi_1, \ldots, \xi_n$ were truly independent, we would have $D_{\text{inf}} = 1$.
\end{restatable}
It was recently shown that it is possible to prove matrix Chernoff-type bounds for sums of random matrices involving random variables that are one-side $\ell_{\infty}$-independent \cite{KKS22}. We rely on that work directly to prove the subspace embedding property. For the matrix-vector multiplication property we  provide a new analysis, which generalizes the approach of \cite{DrineasKannanMahoney:2006} (which holds for independent random samples) to any distribution that satisfies one-side $\ell_{\infty}$-independence. 

Our Corollary \ref{corr:main} for pivotal sampling follows from showing that, no matter what binary tree is used, the method produces samples from a distribution that is one-sided $\ell_{\infty}$-independent with parameter $D_{\text{inf}} = 2$. This can be proven using the fact that pivotal sampling is strongly Rayleigh  \cite{KKS22}. However, we also provide a proof from first principals in the supplemental material.

\section{Experiments}\label{sec:experiments}
With our main theoretical result in place, we experimentally evaluate the performance of our pivotal sampling methods on active regression problems with low-dimensional structure. The benefits of leverage score sampling over uniform sampling for such problems has already been well establish (see e.g. \cite{HamptonDoostan:2015} or \cite{GHM22}) and we provide some additional experiments confirming this fact in Appendix \ref{sec:appendix:experiments}. So, we focus on comparing our pivotal methods to the widely used baseline of \emph{independent} Bernouilli leverge score sampling. 

\noindent \textbf{Test Problems.} We consider several test problems inspired by applications to parametric PDEs. In these problems, we are given a differential equation with parameters, and seek to compute some observable quantity of interest (QoI) for different choices of parameters. This can be done directly by solving the PDE, but doing so is computationally expensive. Instead, the goal is to use a small number of solutions to the PDE to fit a surrogate model (in our case, a low-degree polynomial) that approximates the QoI well, either over a given parameter range, or on average for parameters drawn from a non-uniform distribution (e.g. Gaussian). The problem is naturally an active regression problem because we can select exactly what choices of parameters we solve the PDE for. 

The first differential equation we consider models the displacement $x$ of a damped harmonic oscillator with a sinusoidal force over time $t$. 
\begin{align*}\begin{split}
    \frac{d^2 x}{d t^2}(t) + c \frac{dx}{dt}(t) + k x(t) = f \cos(\omega t), \quad x(0)=x_0, \quad \frac{dy}{dt}(0)=x_1
\end{split}\end{align*}
The equation has four parameters; damping coefficient $c$, spring constant $k$, forcing amplitude $f$, and frequency $\omega$. As a QoI, we consider the maximum oscillator displacement after $20.0$ seconds. We fix $c = 0.5$ and $f = 0.5$, and seek to approximate this displacement over the range domain $k \times \omega = [1,3] \times [0,2]$, which we shift to $k' \times \omega' = [-1, 1] \times [-1, 1]$.

We also consider the heat equation for values of $x \in [0,1]$ with a time-dependent boundary equation and sinusoidal initial condition paramaterized by a frequency $\omega$. The heat equation that we consider describes the temperature $f(t,x)$ by the partial differential equation.
\begin{align*}\begin{split}
    \pi \frac{\partial f}{\partial t} = \frac{\partial^2 f}{\partial x^2}, \quad f(0,t)=0, \quad f(x,0)=\sin(\omega \pi x), \quad \pi e^{-t} + \frac{\partial f(1, t)}{\partial t} = 0
\end{split}\end{align*}
As a QoI, we seek to estimate the maximum temperature over all values of $x$ for $t \in [0,3]$ and  $\omega \in \times [0,5]$. Again we shift and scale to $t' \times \omega' = [-1,1] \times [-1,1]$.

\begin{figure}[t]
\vspace{-1em}
    \centering
    \begin{subfigure}[b]{0.49\textwidth}
        \centering
        \includegraphics[width=.7\linewidth]{NeurIPS/plots/sec5_spring_2.png}
        \vspace{-.5em}
        \caption{Damped Harmonic Oscillator.}    
    \end{subfigure}
    \hfill
    \centering
    \begin{subfigure}[b]{0.49\textwidth}
        \centering
        \includegraphics[width=.7\linewidth]{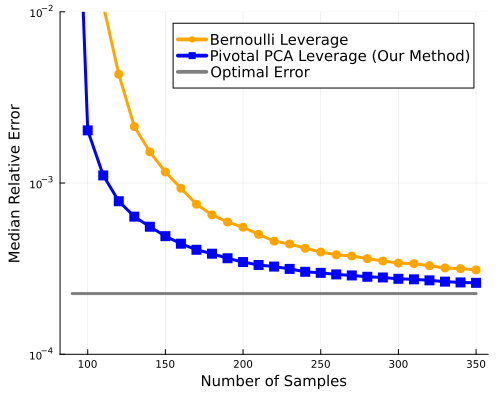}
                \vspace{-.5em}
        \caption{Heat Equation.}    
    \end{subfigure}
    \caption{Results for degree $12$ active polynomial regression for the damped harmonic oscillator QoI in (a) and the heat equation QoI in (b). Our leverage-score based pivotal method outperforms standard Bernouilli leverage score sampling, suggesting the benefits of spatially-aware sampling.}
    \label{fig:sec5:result}
    \vspace{-1em}
\end{figure}

\begin{figure}[b]
    \vspace{-1em}
    \centering
    \begin{subfigure}[b]{0.32\textwidth}
        \centering
        \includegraphics[width=.9\linewidth]{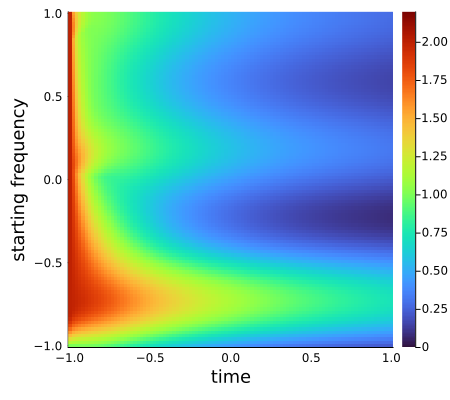}
                \vspace{-.5em}
        \caption{Target Function.}    
    \end{subfigure}
    \hfill
    \centering
    \begin{subfigure}[b]{0.32\textwidth}
        \centering
        \includegraphics[width=.9\linewidth]{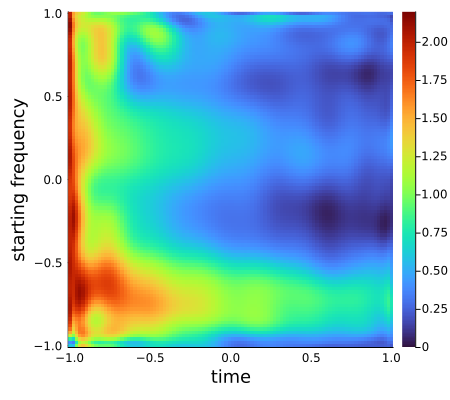}
                \vspace{-.5em}
        \caption{Bernoulli Sampling.}    
    \end{subfigure}
    \hfill
    \centering
    \begin{subfigure}[b]{0.32\textwidth}
        \centering
        \includegraphics[width=.9\linewidth]{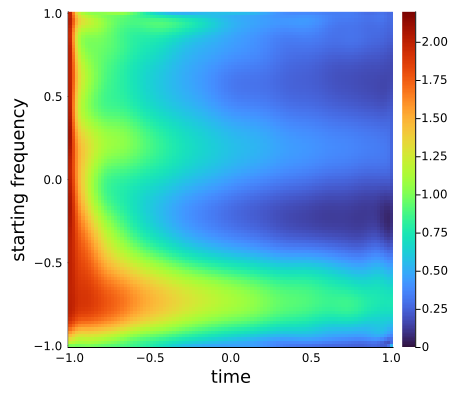}
                \vspace{-.5em}
        \caption{Pivotal Sampling (our method).}    
    \end{subfigure}
    \caption{
    Visualizations of a polynomial approximation to the maximum temperature of a heat diffusion problem, as a function of time and starting condition. (a) is the target value and both (b) and (c) draw $250$ samples using the leverage score and perform polynomial regression of degree $20$.  However, (b) uses Bernoulli sampling while (c) employs our PCA-based pivotal sampling. The pivotal sampling method results in a fit with fewer artifacts that better match the target.}
    \label{fig:sec5:heat}
    \vspace{-1em}
\end{figure}

\noindent \textbf{Data Matrix.} For both problem, we construct $\bv{A}$ by uniformly selecting $n = 10^5$ data points in the 2-dimensional parameter range of interest. We then add all polynomial features of degree $p=12$ as discussed in Section \ref{subsec:contributions}. We compute sampled entries from the target vector $\bv{b}$ using standard MATLAB routines. Results comparing our PCA-based pivotal method and Bernoulli leverage score sampling are show in Figure \ref{fig:sec5:result}. We report median normalized error ${\lVert \mathbf{A} \tilde{\mathbf{x}}^* - \mathbf{b} \rVert_2^2}/{\lVert \mathbf{b} \rVert_2^2}$ after $1000$ trials. 
By drawing more samples from $\bv{b}$, the errors of all methods eventually converge to the optimal error  ${\lVert \mathbf{A} {\mathbf{x}}^* - \mathbf{b} \rVert_2^2}/{\lVert \mathbf{b} \rVert_2^2}$, but clearly the pivotal method requires less samples to achieve a given level of accuracy, confirming the benefits of spatially-aware sampling.

We also visualize results for the damped harmonic oscillator in Figure \ref{fig:sec1:comparison}, showing approximations obtained with $250$ samples. Visualizations for the heat equation are given in Figure \ref{fig:sec5:heat}. For both targets, one can directly see that pivotal sampling improves the performance over Bernoulli sampling. 

We also consider a chemical surface coverage problem from \cite{HamptonDoostan:2015}. Details are relegated to \Cref{sec:appendix:experiments}. Again, we vary two parameters, and seek to fit a surrogate model using a small number of example pairs of parameters. Instead of a uniform distribution, for this problem the rows of $\bv{X}$ are drawn from a truncated Gaussian distribution. We construct $\bv{A}$ using polynomial features of varying degree. Convergence results are shown in Figure \ref{fig:sec7.5:result} and the fit visualized for data near the origin in Figure \ref{fig:sec7.5:surface}. This is a challenging problem since the target function has sharp threshold behavior that is difficult to approximate with a polynomial. However, our pivotal based leverage score sampling performs well, providing a better fit than Bernoulli sampling. Additional experiments, including on 3D problems are reported in Appendix \ref{sec:appendix:experiments}.

\begin{figure}[h]
    \vspace{-.6em}
    % \centering
    % \begin{subfigure}[b]{0.49\textwidth}
    %     \centering
    %     % \includegraphics[width=\linewidth]{plots/sec7_5_spring3D.png}
    %     \includegraphics[width=\linewidth]{NeurIPS/plots/spring3D_p8.png}
    %     \caption{Oscillator Model in 3D Space.}    
    % \end{subfigure}
    % \hfill
    \centering
    \begin{subfigure}[b]{0.32\textwidth}
        \centering
        \includegraphics[width=\linewidth]{NeurIPS/plots/surface_p15.png}
            \vspace{-.5em}
        \caption{Surface Reaction with $p=15$.}    
    \end{subfigure}
    \hfill
    \centering
    \begin{subfigure}[b]{0.32\textwidth}
        \centering
        \includegraphics[width=\linewidth]{NeurIPS/plots/surface_p20.png}
            \vspace{-.5em}
        \caption{Surface Reaction with $p=20$.}    
    \end{subfigure}
    \hfill
    \centering
    \begin{subfigure}[b]{0.32\textwidth}
        \centering
        \includegraphics[width=\linewidth]{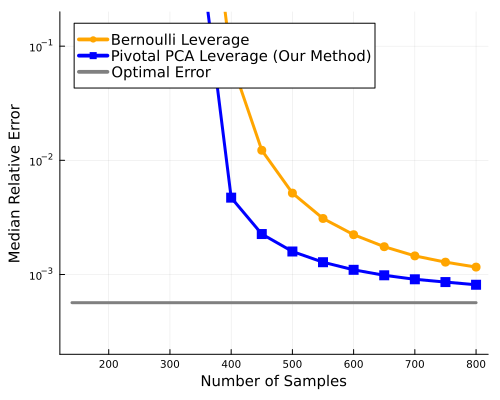}
            \vspace{-.5em}
        \caption{Surface Reaction with $p=25$.}    
    \end{subfigure}
    \caption{Error on fitting a surface reaction model with polynomial degree $p=15, 20, 25$, respectively. As expected, when the degree is higher, the best obtainable error is lower, and more samples are needed for an accurate fit. Again, pivotal sampling outperforms Bernoulli sampling.}
    \label{fig:sec7.5:result}
    \vspace{-.9em}
\end{figure}

\begin{figure}[h!]
    \vspace{-.5em}
    \centering
    \begin{subfigure}[b]{0.32\textwidth}
        \centering
        \includegraphics[width=.9\linewidth]{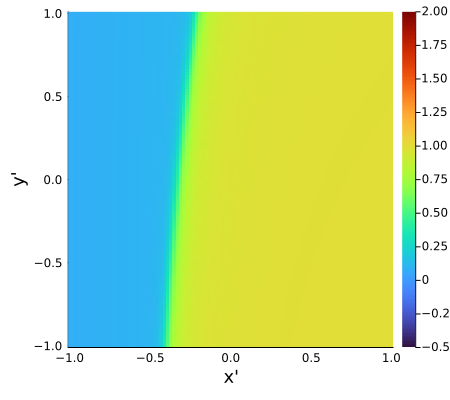}
                    \vspace{-.5em}
        \caption{Target Function.}    
    \end{subfigure}
    \hfill
    \centering
    \begin{subfigure}[b]{0.32\textwidth}
        \centering
        \includegraphics[width=.9\linewidth]{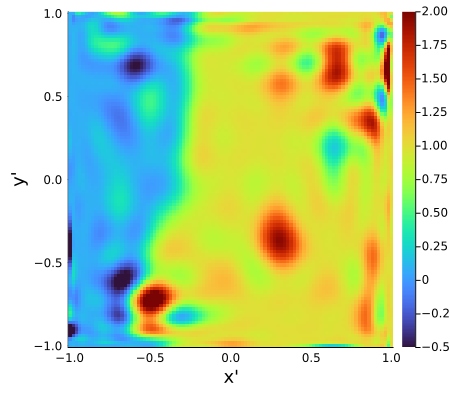}
                    \vspace{-.5em}
        \caption{Bernoulli sampling.}    
    \end{subfigure}
    \hfill
    \centering
    \begin{subfigure}[b]{0.32\textwidth}
        \centering
        \includegraphics[width=.9\linewidth]{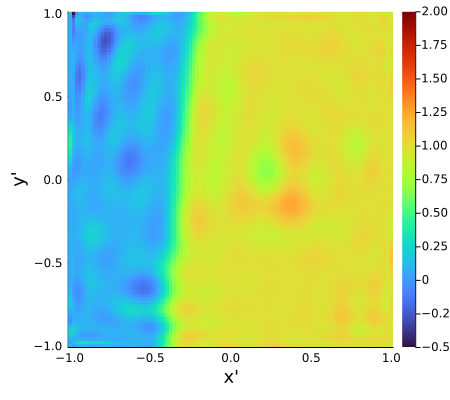}
                    \vspace{-.5em}
        \caption{Pivotal sampling (our method).}    
    \end{subfigure}
    \caption{Visualization of the approximation of the surface reaction model when $400$ samples are used to fit a degree $25$ polynomial.}
    \label{fig:sec7.5:surface}
    \vspace{-1em}
\end{figure}

\section{Conclusion and Future Work}\label{sec:conclusion}
\vspace{-.25em}
In this paper, we introduce a \textit{spatially-aware} pivotal sampling method and show its effectiveness for active linear regression for spatially-structured problems. We prove a general theorem that can be used to analyze the number of samples required in the agnostic setting for any similar method that samples with marginal probabilities proportional to the leverage scores using a distribution that satisfies one-sided $\ell_{\infty}$ independence. Currently, the most obvious limitation of our work is that our theoretical results fall short of explaining why pivotal sampling outperforms (rather than just matches) Bernoulli leverage scores sampling. Providing such an explanation is an interesting avenue for future work. For example, it is possible that for polynomial regression specifically, our method could obtain $O(d/\epsilon)$ complexity, eliminating a $\log(d)$ factor from our current bound.

\bibliographystyle{plainnat}
\bibliography{references}

\newpage 
\section{Probability Pre-processing}
\label{app:preprocess}
As discussed in Section \ref{sec:preliminaries}, our sampling methods require computing probabilities $\tilde{p}_1, \ldots, \tilde{p}_n$ where $\tilde{p}_i = \min(1, c_k\cdot \tau_i)$ for some fixed constant $c_k$ chosen so that $\sum_{i=1}^n \tilde{p}_i = k$. We can find such probabilities using a simple iterative method, which takes as input initial ``probabilities''
\begin{align*}
    p_i = \frac{k}{d}\tau_i
\end{align*}
that are proportional to the leverage scores $\tau_1, \ldots, \tau_n$ of $\bv{A}$. Note that $p_i$ could be larger than $1$ if $k > d$. Pseudocode for how to adjust these probabilities is included in Algorithm \ref{algo:ceiling}.
\begin{algorithm}[h!]\caption{Probability Ceiling}\label{algo:ceiling}
    \begin{algorithmic}[1]
        \Require Number of samples to choose $k$, inclusion probabilities $\{p_1, \cdots, p_n\}$.
        \State Set $\tilde{p}_i = p_i$ for all $i \in [n]$.
        \While{$\mathcal{S} = \{i \in [n]; \tilde{p}_i > 1\}$ is not empty}
            \State $\tilde{p}_i \gets 1$ for all $i \in \mathcal{S}$.
            \State $\mathcal{D} = \{i \in [n]; \tilde{p}_i = 1\}$, $\mathcal{R} = \{i \in [n]; \tilde{p}_i < 1\}$.
            \State $\tilde{p}_i \gets \frac{k - |\mathcal{D}|}{\sum_{i \in \mathcal{R}} \tilde{p}_i} \tilde{p}_i$ for all $i \in \mathcal{R}$.
        \EndWhile 
        \Return $\{\tilde{p}_1, \cdots, \tilde{p}_n\}$
    \end{algorithmic}
\end{algorithm}

To see that the method returns $\tilde{p}_i$ with the desired properties, first note that $\frac{k - |\mathcal{D}|}{\sum_{i \in \mathcal{R}} \tilde{p}_i}$ is always greater than $1$. In particular, at the beginning of the while loop, we always have the invariant that $\sum_{i=1}^n \tilde{p}_i = k$. Accordingly, $\sum_{i\in \mathcal{R}}\tilde{p}_i = k - \sum_{i\in \mathcal{D}}\tilde{p}_i < k - |\mathcal{D}|$. As a result, at the end of the algorithm, $\tilde{p}_i$ certainly equals $\min(1, c\cdot \tau_i)$ for some constant $c \geq 1$. And in fact, it must be that $c = c_k$ by the invariant that $\sum_{i=1}^n \tilde{p}_i = k$.

The main loop in Algorithm \ref{algo:ceiling} terminates as soon as $\mathcal{S}$ is empty, which happens after at most $k$ steps. In practice however, the method usually converges in 2 or 3 iterations. Nevertheless, our primary objective is to minimize the number of samples required for an accurate regression solution -- we do not consider runtime costs in detail.

\section{Proof of Main Result}
In this section, we prove our main result, Theorem \ref{thm:main}. The result holds for any row sampling distribution whose marginal probabilities are proportional to the leverage scores of $\bv{A}$, and which satisfies a few additional conditions. In particular, we require that the distribution is both $k$-homogeneous and is $\ell_\infty$-independent with constant parameter $D$. We define the first requirement below, and also recall the definition of $\ell_\infty$-independence from Section \ref{sec:analysis}. For both definitions, we view our subsampling process as a method that randomly selects a binary vector $\bm{\xi} = \{\xi_1, \cdots, \xi_n\}$ from $\{0,1\}^n$. Entries of $1$ in the vector correspond to rows that are sampled from $\bv{A}$ (and thus labels that are observed in $\bv{b}$). Under this notation, our requirement on marginal probabilities is equivalent to requiring that $\E[\xi_i] = \min(1,c \cdot \tau_i) = \tilde{p}_i$ for some fixed constant $c$, where $\tau_i$ is the $i$-th leverage score of $\bv{A}$.

\begin{definition}[Homogeneity]\label{def:homogeneity}
    A distribution over binary vectors $\bm{\xi} = \{\xi_1, \cdots, \xi_n\}$ is \textit{$k$-homogeneous} if all possible realizations of  $\bm{\xi}$ contain exactly $k$ ones.
\end{definition}

\linfinityind*

With these definitions in place, we are ready for our main proof:
\begin{proof}[Proof of Theorem \ref{thm:main}]
As outlined in Section \ref{sec:analysis}, we need to establish two results: a subspace embedding guarantee and an approximate matrix-vector multiplication result. We prove these separately below, then combine the results to complete the proof of the theorem

\noindent{\textbf{Subspace Embedding.}}
Let $\bv{S}$ be a subsampling matrix corresponding to a vector $\bs{\xi}$ selected using a distribution with the properties above. That is, $\bv{S} \in \R^{k\times n}$ contains a row for every index $i$ such that $\xi_i = 1$, and that row has value ${1}/{\sqrt{\tilde{p}_i}}$ at entry $i$, and is $0$ everywhere else. Let $\bv{U}\in \R^{n\times d}$ be an orthonormal span for $\bv{A}$'s columns. We claim that, if $k = O\left(\frac{d \log (d / \delta)}{\alpha^2}\right)$, then with probability $1 - \delta$,
\begin{align}\begin{split}\label{subspacegoal}
    \lVert \mathbf{U}^T \mathbf{S}^T \mathbf{S} \mathbf{U} - \mathbf{I} \rVert_2 \leq \alpha.
\end{split}\end{align}
Note that, since $\left| \lVert \mathbf{S} \mathbf{U} \mathbf{x} \rVert_2^2 - \lVert \mathbf{U} \mathbf{x} \rVert_2^2 \right|
    = \left| \mathbf{x}^T \left(\mathbf{U}^T \mathbf{S}^T \mathbf{S} \mathbf{U} - \mathbf{I}\right) \mathbf{x} \right|
    \leq \lVert \mathbf{U}^T \mathbf{S}^T \mathbf{S} \mathbf{U} - \mathbf{I} \rVert_2 \lVert \mathbf{x} \rVert_2^2$, this is equivalent to showing that, for all $\bv{x}\in \R^d$,
\begin{align}\begin{split}
    (1-\alpha) \lVert \mathbf{U} \mathbf{x} \rVert_2^2 \leq \lVert \mathbf{S} \mathbf{U} \mathbf{x} \rVert_2^2 \leq (1 + \alpha) \lVert \mathbf{U} \mathbf{x} \rVert_2^2.
\end{split}\end{align}

We establish \eqref{subspacegoal} using the following matrix Chernoff bound from \cite{KKS22}:
\begin{lemma}[Matrix Chernoff for $\ell_{\infty}$-independent Distributions]\label{lemma:chernoff}
    Let $\xi_1, \cdots, \xi_n$ be binary random variables with joint distribution $\mu$ that is $z$-homogeneous for any positive integer $z$ and $\ell_{\infty}$-independent with parameter $D$. Let $\mathbf{Y}_1, \ldots, \mathbf{Y}_n \in \R^{d \times d}$ be positive semidefinite (PSD) matrices with spectral norm $\|\mathbf{Y}_i\|_2 \leq R$ for some $R>0$. Let $\mu_{\max} = \lambda_{\max}(\mathbb{E}_{\bm{\xi} \sim \mu}[\sum_{i=1}^n \xi_i \bv{Y}_i])$. Then, for any $\alpha \in (0,1)$ and a fixed constant $c$,
    \begin{align}\begin{split}\label{eq:nachernoff}
        \Pr \left[\lambda_{\max}\left( \sum_{i=1}^n \xi_i \mathbf{Y}_i \right) \geq (1+\alpha) \mu_{\max}\right] \leq d \exp \left( - \frac{\alpha^2 \mu_{\max}}{cRD^2} \right).
    \end{split}\end{align}
    % Note that $\bv{I}$ denotes the $d\times d$ identity matrix.
\end{lemma}

Let $\mathcal{D} = \{i \in [n]: \tilde{p}_i = 1\}$ and $\mathcal{R} = \{i \in [n]: \tilde{p}_i < 1\}$. 
Letting $\bv{u}_i$ denote the $i$-th row of $\bv{U}$,
$\mathbf{U}^T \mathbf{S}^T \mathbf{S} \mathbf{U}$ can be decomposed as:
\begin{align}\begin{split}
    \mathbf{U}^T \mathbf{S}^T \mathbf{S} \mathbf{U} = \sum_{i\in \mathcal{R}} \xi_i \frac{\mathbf{u}_i \mathbf{u}_i^T}{\tilde{p}_i} + \sum_{i\in \mathcal{D}} \mathbf{u}_i \mathbf{u}_i^T.
\end{split}\end{align}
Recall that $\tilde{p}_i = \min(1,c\cdot\tau_i)$ for some constant $c$, and in fact, since $\sum_{i=1}^n \tilde{p}_i = k$ and $\sum_{i=1}^n \tau_i \leq d$, it must be that $c \geq \frac{k}{d}$. Define $p_i = \frac{k}{d}\tau_i$.
For each $i \in \mathcal{D}$, let $m_i$ be an integer such that $m_i \geq p_i$. For $i \in \mathcal{D}$, $p_i = \frac{k}{d}\tau_i \geq 1$, so $m_i \geq 1$. Additionally let $\xi_i^1, \ldots, \xi_i^{m_i} = \xi_i$ be variables that are deterministically $1$. Then we can trivially split up the second sum above so that:
\begin{align}\label{eq:split_sum}
\begin{split}
    \mathbf{U}^T \mathbf{S}^T \mathbf{S} \mathbf{U} = \sum_{i\in \mathcal{R}} \xi_i \frac{\mathbf{u}_i \mathbf{u}_i^T}{\tilde{p}_i} + \sum_{i\in \mathcal{D}} \sum_{j=1}^{m_i} \xi_i^j \frac{\mathbf{u}_i \mathbf{u}_i^T}{m_i}.
\end{split}\end{align}
It is easy to see that $\{\xi_i : i\in \mathcal{R}\} \cup \{\xi_i^j : i\in \mathcal{D}, j\in [m_i]\}$ are $\ell_{\infty}$-independent random variables with parameter $D$ and form a distribution that is $\bar{k}$ homogeneous for $\bar{k} = k + \sum_{i\in d} (m_i - 1)$. So, noting that $\bv{u}_i\bv{u}_i^T$ is PSD, we can apply Lemma \ref{lemma:chernoff} to \eqref{eq:split_sum}. We are left to bound $R$ and $\mu_{\max}$. 

To bound $R$, for $i \in \mathcal{R}$, let $\bv{Y}_i = \frac{\mathbf{u}_i \mathbf{u}_i^T}{\tilde{p}_i}$. For such $i$, we have that $\tilde{p}_i \geq p_i$. So,
\begin{align}\begin{split}
    \lVert \mathbf{Y}_i \rVert_2 = \frac{\lVert \mathbf{u}_i \mathbf{u}_i^T \rVert_2}{\tilde{p}_i} = \frac{\lVert \mathbf{u}_i \rVert_2^2}{\tilde{p}_i} = \frac{\tau_i}{\tilde{p}_i} \leq \frac{\tau_i}{p_i} = \frac{d}{k}.
\end{split}\end{align}
I.e., for all $i \in \mathcal{R}$, $\|\bv{Y}_i\|_2 \leq R$ for $R = \frac{d}{k}$. Additionally, for $i \in \mathcal{D}$, let  $\bv{Y}_i = \frac{\mathbf{u}_i \mathbf{u}_i^T}{m_i}$. Since we chose $m_i \geq p_i$, by the same argument, we have that $\lVert \mathbf{Y}_i \rVert_2 \leq R$ for $R = \frac{d}{k}$.

Next, we bound $\mu_{\max}$ by simply noting that
\begin{align}\begin{split}
    \mathbb{E}\left[ \bv{U}^T\bv{S}^T\bv{S}\bv{U}\right] = \sum_{i\in \mathcal{R}} \mathbb{E}[\xi_i] \frac{\mathbf{u}_i \mathbf{u}_i^T}{\tilde{p}_i} + \sum_{i\in \mathcal{D}} \mathbf{u}_i \mathbf{u}_i^T = \sum_{i\in \mathcal{R}}\bv{u}_i\bv{u}_i^T + \sum_{i\in \mathcal{D}} \mathbf{u}_i \mathbf{u}_i^T = \bv{I}, 
\end{split}\end{align}
where $\bv{I}$ denotes the $d\times d$ identity matrix. It follows that  $\mu_{\max} = 1$.
Recalling that we assume that the $\ell_\infty$-independence parameter $D$ is a constant, plugging into Lemma \ref{lemma:chernoff}, we obtain
\begin{align}\begin{split}
    \Pr \left[\lambda_{\max}\left( \mathbf{U}^T \mathbf{S}^T \mathbf{S} \mathbf{U} \right) \geq 1+\alpha \right] &\leq d \exp \left( - \frac{\alpha^2}{cRD^2} \right) \\
    % \Pr \left[ \lVert \bar{\mathbf{U}}^T \bar{\mathbf{S}}^T \bar{\mathbf{S}} \bar{\mathbf{U}} - \mathbf{I} \rVert_2 \geq \alpha \right] &\leq d \exp \left( - \frac{k \alpha^2}{O(d)} \right) \\
    \Pr \left[ \lVert \mathbf{U}^T \mathbf{S}^T \mathbf{S} \mathbf{U} - \mathbf{I} \rVert_2 \geq \alpha \right] &\leq d \exp \left( - \frac{k \alpha^2}{c'd} \right),
\end{split}\end{align}
where $c, c'$ are fixed constants. 
Setting $k = O(\frac{d \log (d / \delta)}{\alpha^2})$, we have $\lVert \mathbf{U}^T \mathbf{S}^T \mathbf{S} \mathbf{U} - \mathbf{I} \rVert_2 \leq \alpha$ with probability at least $1 - \delta$. This establishes \eqref{subspacegoal}.

\noindent{\textbf{Approximate Matrix-Vector Multiplication}.}
With our subspace embedding result in place, we move onto proving the necessary approximate matrix-vector multiplication result that we briefly introduced in Section \ref{sec:analysis}. In particular, we wish to show that, with good probability, $\|\mathbf{U}^T \mathbf{S}^T \mathbf{S} (\mathbf{b} - \mathbf{A} \mathbf{x}^*) \|_2^2 \leq \epsilon \| \mathbf{b} - \mathbf{A} \mathbf{x}^* \|_2^2$ where $\bv{x}^* = \argmin_\bv{x}\|\bv{A}\bv{x} - \bv{b}\|_2^2$. Since the residual of the optimal solution, $\mathbf{b} - \mathbf{A} \mathbf{x}^*$, is the same under any column transformation of $\bv{A}$, we have $\mathbf{b} - \mathbf{A} \mathbf{x}^* = \bv{b} - \bv{U}\bv{y}^*$ where $\bv{y}^* = \argmin_\bv{y}\|\bv{U}\bv{y} - \bv{b}\|_2^2$. We can then equivalently show that, if we take $k = O\left(\frac{d}{\epsilon\delta}\right)$ samples, with probability $1-\delta$,
\begin{align}
\label{eq:matvec_result}
    \| \mathbf{U}^T \mathbf{S}^T \mathbf{S} (\mathbf{b} - \mathbf{U} \mathbf{y}^*) \|_2^2 \leq \epsilon \| \mathbf{b} - \mathbf{U} \mathbf{y}^* \|_2^2.
\end{align}
As in the proof for independent random samples \cite{DrineasKannanMahoney:2006}, we will prove \eqref{eq:matvec_result} by bounding the expected squared error and applying Markov's inequality. In particular, we have
\begin{align}\begin{split}\label{eq:markov}
    \Pr \left[ \lVert \mathbf{U}^T \mathbf{S}^T \mathbf{S} (\mathbf{b} - \mathbf{U} \mathbf{y}^*) \rVert_2^2 \geq \epsilon \lVert \mathbf{b} - \mathbf{U} \mathbf{y}^* \rVert_2^2 \right] \leq \frac{\mathbb{E}\big[\lVert \mathbf{U}^T \mathbf{S}^T \mathbf{S} (\mathbf{b} - \mathbf{U} \mathbf{y}^*) \rVert_2^2 \big]}{\epsilon \lVert \mathbf{b} - \mathbf{U} \mathbf{y}^* \rVert_2^2} .
\end{split}\end{align}
Let $\mathbf{z} = \mathbf{b} - \mathbf{U} \mathbf{y}^*$ and note that $\mathbf{U}^T \mathbf{z} = \mathbf{0}$. The numerator on the right side can be transformed as
\begin{align}\begin{split}
    \mathbb{E}\left[ \lVert \mathbf{U}^T \mathbf{S}^T \mathbf{S} \mathbf{z} \rVert_2^2 \right] = \mathbb{E}\left[ \lVert \mathbf{U}^T \mathbf{S}^T \mathbf{S} \mathbf{z} - \mathbf{U}^T \mathbf{z} \rVert_2^2 \right] 
    = \mathbb{E}\left[ \lVert \mathbf{U}^T (\mathbf{S}^T \mathbf{S} - \mathbf{I}) \mathbf{z} \rVert_2^2 \right]. 
\end{split}\end{align}
Note that above $\mathbf{S}^T \mathbf{S} - \mathbf{I}$ is a diagonal matrix with $i$-th diagonal entry equal to $\frac{1}{\tilde{p}_i} - 1$ if $\xi_i=1$ and $-1$ if $\xi_i=0$. Expanding the $\ell_2$-norm and using that $\xi_i = 1$ and $\tilde{p}_i=1$ for $i \notin \mathcal{R}$, we have
\begin{align}
    \mathbb{E}\big[ \lVert \mathbf{U}^T \mathbf{S}^T \mathbf{S} \mathbf{z} \rVert_2^2 \big] &= \sum_{j=1}^d \mathbb{E} \left[ \left( \sum_{i=1}^n \left(\frac{\xi_i}{\tilde{p}_i}-1\right) u_{ij} z_i \right)^2 \right]
    % &= \sum_{j=1}^d \mathbb{E} \left[ \left( \sum_{i \in \mathcal{D}} \left(\frac{\xi_i}{\tilde{p}_i}-1\right) u_{ij} z_i + \sum_{i \in \mathcal{R}} \left(\frac{\xi_i}{\tilde{p}_i}-1\right) u_{ij} z_i \right)^2 \right] 
    \nonumber \\
    &= \sum_{j=1}^d \mathbb{E} \left[ \left( \sum_{i \in \mathcal{R}} \left(\frac{\xi_i}{\tilde{p}_i}-1\right) u_{ij} z_i \right)^2 \right] \nonumber \\
    &= \sum_{j=1}^d \sum_{i \in \mathcal{R}} \sum_{l \in \mathcal{R}} \frac{c_{il}}{\tilde{p}_i \tilde{p}_l} u_{ij} z_i u_{lj} z_l, 
\end{align}
where $c_{il} = \text{Cov}(\xi_i, \xi_l) = \E[(\xi_i - \tilde{p}_i)(\xi_l - \tilde{p}_l)] = \E[\xi_i\xi_l] - \tilde{p}_i\tilde{p}_l$. We will show that 
\begin{align}
    \sum_{j=1}^d \sum_{i \in \mathcal{R}} \sum_{l \in \mathcal{R}} \frac{c_{il}}{\tilde{p}_i \tilde{p}_l} u_{ij} z_i u_{lj} z_l 
    \leq D \sum_{j=1}^d \sum_{i \in \mathcal{R}} \frac{u_{ij}^2 z_i^2}{\tilde{p}_i},
\end{align}
where $D$ is the $\ell_\infty$-independence parameter. It suffices to show that for every $j$, we have
\begin{align}
\label{eq:before_psd}
    \sum_{i \in \mathcal{R}} \sum_{l \in \mathcal{R}} \frac{c_{il}}{\tilde{p}_i \tilde{p}_l} u_{ij} z_i u_{lj} z_l \leq D \sum_{i \in \mathcal{R}} \frac{u_{ij}^2 z_i^2}{\tilde{p}_i}.
\end{align}
Define a symmetric matrix $\bv{M} \in \R^{|\mathcal{R}| \times |\mathcal{R}|}$ with entries $m_{il} = \frac{c_{il}}{\sqrt{\tilde{p}_i} \sqrt{\tilde{p}_l}}$ and a vector $\bv{v} \in \R^{|\mathcal{R}|}$ with entries $v_i = \frac{u_{ij} z_i}{\sqrt{\tilde{p}_i}}$. The desired result in \eqref{eq:before_psd} can be expressed as 
\begin{align}\begin{split}
    \bv{v}^T \bv{M} \bv{v} \leq D \lVert \bv{v} \rVert_2^2.
\end{split}\end{align}
With $\mathcal{S} = \emptyset$, the one-sided $\ell_\infty$-independence condition implies that, for all $i \in [n]$, 
\begin{align}\begin{split}
    \sum_{l \in \mathcal{R}} \left| \frac{\E[\xi_i \xi_l]}{\tilde{p}_i} -\tilde{p}_l \right| = \sum_{l \in \mathcal{R}} \frac{\sqrt{\tilde{p}_l}}{\sqrt{\tilde{p}_i}} | m_{il} | \leq D.
\end{split}\end{align}
Equivalently, if we define a diagonal matrix $\bm{\Lambda} \in \R^{|\mathcal{R}| \times |\mathcal{R}|}$ such that $\Lambda_{ii} = \sqrt{p_i}$,  then, we have shown:
\begin{align}\begin{split}
    \lVert \bm{\Lambda}^{-1} \bv{M} \bm{\Lambda} \rVert_{\infty} \leq D,
\end{split}\end{align}
where for a matrix $\bv{B}$, $\|\bv{B}\|_{\infty}$ denotes $\max_i \sum_{l} |\bv{B}_{il}| = \max_\bv{x} \frac{\|\bv{B}\bv{x}\|_{\infty}}{\|\bv{x}\|_{\infty}}$.
It follows that the largest eigenvalue of $\bm{\Lambda}^{-1} \bv{M} \bm{\Lambda}$ is at most $D$ and thus, the largest eigenvalue of $\bv{M}$ is also at most $D$. Therefore, we have $\bv{v}^T \bv{M} \bv{v} \leq D \lVert \bv{v} \rVert_2^2$. Considering that $\tilde{p}_i \geq p_i = \frac{k}{d}\tau_i$ for $i \in \mathcal{R}$, we have
\begin{align}
    \sum_{j=1}^d \mathbb{E} \left[ \left( \sum_{i \in \mathcal{R}} \left(\frac{\xi_i}{\tilde{p}_i}-1\right) u_{ij} z_i \right)^2 \right] &\leq D \sum_{j=1}^d \sum_{i \in \mathcal{R}} \frac{u_{ij}^2 z_i^2}{\tilde{p}_i} \nonumber\\
    &= D\sum_{i \in \mathcal{R}} \frac{z_i^2}{\tilde{p}_i}\sum_{j=1}^d u_{ij}^2 \nonumber \\
    &= D\sum_{i \in \mathcal{R}} \frac{z_i^2}{\tilde{p}_i}\tau_i \nonumber\\ &\leq D \frac{d}{k} \lVert \bv{z} \rVert_2^2.
\end{align}

Recalling that $D$ is a constant and $\bv{z} = \bv{U}\bv{y}^* - \bv{b}$, we can plug into (\ref{eq:markov}) with $k = O\left( \frac{d}{\epsilon \delta} \right)$ samples, which proves that \eqref{eq:matvec_result}  holds with probability at least $1 - \delta$.

\noindent{\textbf{Putting it all together.}}
With \eqref{subspacegoal} and \eqref{eq:matvec_result} in place, we can prove our main result, \eqref{eq:main_thm_gaur} of \ref{thm:main}. We follow a similar approach to \cite{Woodruff:2014}. By reparameterization, proving this inequality is equivalent to showing that
\begin{align}\label{eq:reparam}\begin{split}
    \lVert \mathbf{U} \tilde{\mathbf{y}}^* - \mathbf{b} \rVert_2^2 \leq (1+\epsilon) \lVert \mathbf{U} \mathbf{y}^* - \mathbf{b} \rVert_2^2,
\end{split}\end{align}
where $\tilde{\mathbf{y}}^* = \argmin_{\bv{y}}\|\bv{S}\bv{U}\bv{y} - \bv{S}\bv{b}\|_2^2$ and ${\mathbf{y}}^* = \argmin_{\bv{y}}\|\bv{U}\bv{y} - \bv{b}\|_2^2$.
Since $\mathbf{y}^*$ is the minimizer of $\lVert \mathbf{U}\mathbf{y} - \mathbf{b} \rVert_2^2$, we have $\nabla_{\mathbf{y}} \lVert \mathbf{U} \mathbf{y} - \mathbf{b} \rVert_2^2 = 2 \mathbf{U}^T (\mathbf{U}\mathbf{y}- \mathbf{b}) = \mathbf{0}$ at $\mathbf{y}^*$. This indicates that $\mathbf{U} \mathbf{y}^* - \mathbf{b}$ is orthogonal to any vector in the column span of $\mathbf{U}$. Particularly, $\mathbf{U} \mathbf{y}^* - \mathbf{b}$ is orthogonal to $\mathbf{U} \tilde{\mathbf{y}}^* - \mathbf{U} \mathbf{y}^*$. Therefore, by the Pythagorean theorem, we have 
\begin{align}
\label{original_split}
    \lVert \mathbf{U} \tilde{\mathbf{y}}^* - \mathbf{b} \rVert_2^2 = \lVert \mathbf{U} \mathbf{y}^* - \mathbf{b} \rVert_2^2 + \lVert \mathbf{U} \tilde{\mathbf{y}}^* - \mathbf{U} \mathbf{y}^* \rVert_2^2 = \lVert \mathbf{U} \mathbf{y}^* - \mathbf{b} \rVert_2^2 + \lVert \tilde{\mathbf{y}}^* - \mathbf{y}^* \rVert_2^2.
\end{align}
So to prove \eqref{eq:reparam}, it suffices to show that 
\begin{align}\begin{split}
    \lVert \tilde{\mathbf{y}}^* - \mathbf{y}^* \rVert_2^2 \leq \epsilon \lVert \mathbf{U} \mathbf{y}^* - \mathbf{b} \rVert_2^2.
\end{split}\end{align}
Applying \eqref{subspacegoal} with $k = O(d\log d + d/\epsilon) \geq O(d\log d)$ and $\delta = 1/200$, we have that, with probability $99.5/100$, $\lVert \mathbf{U}^T \mathbf{S}^T \mathbf{S} \mathbf{U} - \mathbf{I} \rVert_2 \leq \frac{1}{2}$.
Then, by triangle inequality,
\begin{align}
\label{eq:matvec_step1}
    \lVert \tilde{\mathbf{y}}^* - \mathbf{y}^* \rVert_2 &\leq \lVert \mathbf{U}^T \mathbf{S}^T \mathbf{S} \mathbf{U} (\tilde{\mathbf{y}}^* - \mathbf{y}^*) \rVert_2 + \lVert \mathbf{U}^T \mathbf{S}^T \mathbf{S} \mathbf{U} (\tilde{\mathbf{y}}^* - \mathbf{y}^*) - (\tilde{\mathbf{y}}^* - \mathbf{y}^*) \rVert_2 \nonumber\\
    &\leq \lVert \mathbf{U}^T \mathbf{S}^T \mathbf{S} \mathbf{U} (\tilde{\mathbf{y}}^* - \mathbf{y}^*) \rVert_2 + \lVert \mathbf{U}^T \mathbf{S}^T \mathbf{S} \mathbf{U} - \mathbf{I} \rVert_2 \lVert \tilde{\mathbf{y}}^* - \mathbf{y}^* \rVert_2  \nonumber \\
    &\leq \lVert \mathbf{U}^T \mathbf{S}^T \mathbf{S} \mathbf{U} (\tilde{\mathbf{y}}^* - \mathbf{y}^*) \rVert_2 + \frac{1}{2} \lVert \tilde{\mathbf{y}}^* - \mathbf{y}^* \rVert_2.
\end{align}
Rearranging, we conclude that $\lVert \tilde{\mathbf{y}}^* - \mathbf{y}^* \rVert_2^2 \leq 4 \lVert \mathbf{U}^T \mathbf{S}^T \mathbf{S} \mathbf{U} (\tilde{\mathbf{y}}^* - \mathbf{y}^*) \rVert_2^2$. Since $\tilde{\mathbf{y}}^*$ is the minimizer of $\lVert \mathbf{S} \mathbf{U} \mathbf{y} - \mathbf{S} \mathbf{b} \rVert_2^2$, we have $\nabla_{\mathbf{y}} \lVert \mathbf{S}\mathbf{U}\mathbf{y} - \mathbf{S} \mathbf{b} \rVert_2^2 = 2(\mathbf{S}\mathbf{U})^T(\mathbf{S}\mathbf{U}\tilde{\mathbf{y}}^* - \mathbf{S} \mathbf{b}) = \mathbf{0}$. Thus,
\begin{align}\begin{split}
    \lVert \mathbf{U}^T \mathbf{S}^T \mathbf{S} \mathbf{U} (\tilde{\mathbf{y}}^* - \mathbf{y}^*) \rVert_2^2 = \lVert \mathbf{U}^T \mathbf{S}^T (\mathbf{S} \mathbf{U} \tilde{\mathbf{y}}^* - \mathbf{S} \mathbf{b} + \mathbf{S} \mathbf{b} - \mathbf{S} \mathbf{U} \mathbf{y}^*) \rVert_2^2 = \lVert \mathbf{U}^T \mathbf{S}^T \mathbf{S} (\mathbf{b} - \mathbf{U} \mathbf{y}^*) \rVert_2^2.
\end{split}\end{align}
Applying \eqref{eq:matvec_result} with $\delta = 1/200$ and combining with \eqref{eq:matvec_step1} using union bound, we thus have that with probability $99/100$, 
\begin{align}\begin{split}
    \lVert \tilde{\mathbf{y}}^* - \mathbf{y}^* \rVert_2^2 \leq 4 \lVert \mathbf{U}^T \mathbf{S}^T \mathbf{S} \mathbf{U} (\tilde{\mathbf{y}}^* - \mathbf{y}^*) \rVert_2^2 = 4 \lVert \mathbf{U}^T \mathbf{S}^T \mathbf{S} (\mathbf{b} - \mathbf{U} \mathbf{y}^*) \rVert_2^2 \leq 4 \epsilon \lVert \mathbf{U} \mathbf{y}^* - \mathbf{b} \rVert_2^2.
\end{split}\end{align}
Plugging back into \eqref{original_split} and adjusting $\epsilon$ by a constant factor completes the proof of Theorem \ref{thm:main}.
\end{proof}

\subsection{Proof of Corollary \ref{corr:main}}
\label{sec:appendix:l_infty} 
We briefly comment on how to derive Corollary \ref{corr:main} from our main result. By definition of the method, we immediately have that our binary-tree-based pivotal sampling is $k$-homogeneous, so we just need to show that it produces a distribution over samples that is $\ell_\infty$-independent with constant parameter $D$. 
% As discussed in Section \ref{sec:methods}, our sampling selects exactly $k = \sum_{i=1}^n \tilde{p}_i$ samples, and thus, $k$-homogeneous. So, it suffices to show that $D$ is a constant. 
This fact can be derived directly from a line of prior work. In particular, \cite{BJ12} proves that binary-tree-based pivotal sampling satisfies negative association, \cite{PP14} proves that negative association implies a stochastic covering property, and \cite{KKS22} shows that any distribution satisfying the stochastic covering property has $\ell_\infty$-independence parameter at most $D = 2$. We also give an arguably more direct alternative proof below based on a natural conditional variant of negative correlation. 

\begin{proof}[Proof of Corollary \ref{corr:main}]
\cite{BJ12} proves that binary-tree-based pivotal sampling is \emph{conditionally negatively associated} (CNA). Given a set $\mathcal{C} \subseteq [n]$ and a vector $\bv{c} \in \{0, 1\}^{|\mathcal{C}|}$, we denote the condition $\xi_c = c_i$ for all $i \in \mathcal{C}$ by $C$.  Conditional negative association asserts that, for all $C$, any disjoint subsets $\mathcal{S}$ and $\mathcal{T}$ of $\{\xi_1, \cdots, \xi_n \}$, and any non-decreasing functions  $f$ and $g$,
\begin{align}\begin{split}\label{eq:cna}
    \mathbb{E}[f(\mathcal{S}) | C]\cdot \mathbb{E}[g(\mathcal{T}) | C] \geq \mathbb{E}[f(\mathcal{S})g(\mathcal{T}) | C].
\end{split}\end{align}
When $\mathcal{S}$ and $\mathcal{T}$ are singletons and $f$ and $g$ are the identity functions, we have
\begin{align}\begin{split}
\label{eq:cpnc}
    \mathbb{E}[\xi_i | C] \cdot \mathbb{E}[\xi_j | C] \geq \mathbb{E}[\xi_i \xi_j | C].
\end{split}\end{align}
Since $\E[\xi_i | C] = \Pr[\xi_i = 1 | C]$, we also have 
\begin{align}\begin{split}\label{eq:cnc}
    \Pr[\xi_i = 1 | C] \Pr[\xi_j = 1 | C] &\geq \Pr[\xi_i = 1 \wedge \xi_j = 1 | C] \\
    \Pr[\xi_i = 1 | C] &\geq \frac{\Pr[\xi_i = 1 \wedge \xi_j = 1 | C]}{\Pr[\xi_j = 1 | C]} = \Pr[\xi_i = 1 | \xi_j = 1 \wedge C].
\end{split}\end{align}
In words, the entries of our vector $\bs{\xi}$ are negatively correlated, even conditioned on fixing any subset of entries in the vector.  We will use this fact to show that $\sum_{j \in [n]} | \mathcal{I}_{\mu}^{\mathcal{S}}(i, j) | \leq 2$ for all $i \in [n]$, where $\mathcal{I}_{\mu}^{\mathcal{S}}$ is as defined in Definition \ref{def:l-infty}. For a fixed $i$, let $q_i = \Pr_{\bm{\xi} \sim \mu} [\xi_i=1 | \xi_{\ell} = 1 \forall \ell \in \mathcal{S}]$. Then, we have $| \mathcal{I}_{\mu}^{\mathcal{S}}(i, j) | = 0$ for $j \in \mathcal{S}$, $| \mathcal{I}_{\mu}^{\mathcal{S}}(i, j) | = 1 - q_i$ for $j = i$, and $\sum_{j \in [n] \backslash \mathcal{S} \cup \{ i \}} | \mathcal{I}_{\mu}^{\mathcal{S}}(i, j) | = 1 - q_i$. The last fact follows from $k$-homogeneity, i.e. that $\sum_{i=1}^n q_i = k$, and \eqref{eq:cnc}, which implies that $\mathcal{I}_{\mu}^{\mathcal{S}}(i, j) \leq 0$ for all $j$ in $[n] \backslash \mathcal{S} \cup \{ i \}$, so $\sum_{j \in [n] \backslash \mathcal{S} \cup \{ i \}} | \mathcal{I}_{\mu}^{\mathcal{S}}(i, j) | = \left|\sum_{j \in [n] \backslash \mathcal{S} \cup \{ i \}} \mathcal{I}_{\mu}^{\mathcal{S}}(i, j) \right|$. Thus, we have $\sum_{j \in [n]} | \mathcal{I}_{\mu}^{\mathcal{S}}(i, j) | = 2 - 2q_i \leq 2$.
\end{proof}

Notably, in the above result, we only required property \eqref{eq:cpnc} to hold when $C$ is the event that $\xi_{i_1}= 1, \ldots, \xi_{i_q} = 1$ for some set of indices $i_1, \ldots, i_q$. I.e., it suffices to show that, for any $\mathcal{S}\subseteq [n]$,
\begin{align}\begin{split}
\label{eq:cpnc2}
    \mathbb{E}[\xi_i | \xi_\ell = 1 \forall \ell \in \mathcal{S}] \cdot \mathbb{E}[\xi_j | \xi_\ell = 1 \forall \ell \in \mathcal{S}] \geq \mathbb{E}[\xi_i \xi_j | \xi_\ell = 1 \forall \ell \in \mathcal{S}].
\end{split}\end{align}
We provide a self-contained proof of this fact for tree-based pivotal sampling in Appendix \ref{sec:appendix:cnpc}. Our proof is from first principals, and does not require prior work on Conditional Negative Association or the strongly Rayleigh property.

\section{Complementary Experiments}\label{sec:appendix:experiments}

In Section \ref{sec:experiments}, we conduct experiments using three targets; a damped harmonic oscillator, a heat equation, and a chemical surface reaction. In this appendix, we show the results of additional simulations. Thus far, the original domain of the experiments is 2D for visualization purposes. Here, however, we show the simulation results of the 3D original domain by freeing one more parameter in the damped harmonic oscillator model. We also summarize the number of samples required to achieve a given target error for all four test problems. To corroborate the discussion in Section \ref{sec:lev_score}, we provide a simulation result stating that leverage score sampling actually outperforms uniform sampling. Lastly, the performance of the randomized BSS algorithm on the 2D damped harmonic oscillator target is displayed which supports our discussion in Section \ref{subsec:relatedwork}. Before showing the simulation results, we begin this section with a deferred detail of the chemical surface coverage target in Section \ref{sec:experiments}.

\noindent \textbf{Chemical Surface Coverage.} As explained in \cite{HamptonDoostan:2015}, the target function models the surface coverage of certain chemical species and considers the uncertainty $\rho$ which is parameterized by absorption $\alpha$, desorption $\gamma$, the reaction rate constant $\kappa$, and time $t$. Given a position $(x, y)$ in 2D space, our quantity of interest $\rho$ is modeled by the non-linear evolution equation:
\begin{align}\begin{split}
    \frac{d \rho}{d t} = \alpha (1-\rho) - \gamma \rho - \kappa (1-\rho)^2 \rho, \quad \rho(t=0)=0.9, \\ \alpha = 0.1 + \exp(0.05 x), \quad \gamma = 0.001 + 0.01 \exp(0.05 y)
\end{split}\end{align}
In this experiment, we set $\kappa=10$ and focus on $\rho$ after $t=4$ seconds. The domain is defined as $x \times y = [-10, 10] \times [-10, 10]$ which is scaled to $x' \times y' = [-1, 1] \times [-1, 1]$.

\noindent \textbf{3D Damped Harmonic Oscillator.} The damped harmonic oscillator model is given as (restated),
\begin{align}\begin{split}
    \frac{d^2 x}{d t^2}(t) + c \frac{dx}{dt}(t) + k x(t) = f \cos(\omega t), \quad x(0)=x_0, \quad \frac{dy}{dt}(0)=x_1
\end{split}\end{align}
In section \ref{sec:experiments}, we define the domain as $k \times \omega = [1,3] \times [0,2]$ while $f$ was fixed as $f=0.5$. This time, we extend it to 3D space by setting the domain as $k \times f \times \omega = [1,3] \times [0,2] \times [0,2]$, which is shifted and scaled to $k' \times f' \times \omega' = [-1,1] \times [-1,1] \times [-1,1]$. $3$D plot of the target is given in Figure \ref{fig:appendix:leverage} (a). Note that if we slice the cube at $f'=0.5$, we obtain the target in Section \ref{sec:experiments}.

Again, we set the domain to be the hypercube $[-1,1]^3$. This time, we create the base data matrix $\mathbf{A}' \in \R^{n \times d'}$ by constructing a $n = 51^3$ fine grid, and the polynomial degree is set to $12$. Figure \ref{fig:appendix:leverage} (b) and (c) give examples of the leverage score of this data matrix.

\begin{figure}[h]
    \vspace{-1em}
    \centering
    \begin{subfigure}[b]{0.32\textwidth}
        \centering
        \includegraphics[width=\linewidth]{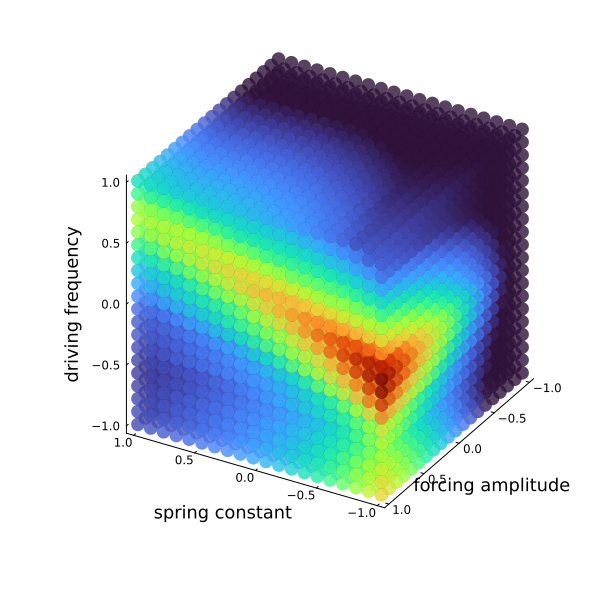}
                \vspace{-.5em}
        \caption{Target Function.}   
    \end{subfigure}
    \hfill
    \centering
    \begin{subfigure}[b]{0.32\textwidth}
        \centering
        \includegraphics[width=\linewidth]{plots/sec7_5_spring3D_leverage_@0.png}
                \vspace{-.5em}
        \caption{Leverage Score at $f'=0.0$.}     
    \end{subfigure}
    \hfill
    \centering
    \begin{subfigure}[b]{0.32\textwidth}
        \centering
        \includegraphics[width=\linewidth]{plots/sec7_5_spring3D_leverage_@1.png}
                \vspace{-.5em}
        \caption{Leverage Score at $f'=1.0$.}    
    \end{subfigure}
    \caption{(a): The target value of the 3D damped harmonic oscillator model. (b) and (c): Its leverage score.We have $f' \in [-1,1]$, and the grid is sliced at $f'=0.0$ in (b) and at $f'=1.0$ in (c).}
    \label{fig:appendix:leverage}
    \vspace{-1em}
\end{figure}

Figure \ref{fig:appendix:3plots} (a) shows the relative error of Bernoulli sampling and our pivotal sampling both using the leverage score. As expected, our method shows a better fit than Bernoulli sampling again. 

% \begin{figure}[t]
% \vspace{-1em}
%     \centering
%         \includegraphics[width=.343\linewidth]{NeurIPS/plots/appendix_spring3D.png}
%         \vspace{-.5em}
%     \caption{Results for degree $12$ active polynomial regression for the damped harmonic oscillator QoI in 3D space.}
%     \label{fig:appendix:3dspring}
%     \vspace{-1em}
% \end{figure}

\noindent \textbf{Samples Needed for a Certain Error.} Table \ref{table:2opt} and \ref{table:1.1opt} summarize the number of samples required to achieve $2 \times \text{OPT}$ error and $1.1 \times \text{OPT}$ error where OPT is the error we could obtain when all the data are labeled. In all four test problems, our pivotal sampling achieves the target error with fewer samples than the existing method denoted by Bernoulli in the tables. Our method is especially sample efficient when we aim at the target error close to OPT. For instance, to achieve the $1.1 \times \text{OPT}$ error in the 2D damped harmonic oscillator model, our method requires less than half samples that Bernoulli sampling requires, showing a significant reduction in terms of the number of samples needed.

\begin{table}[b]
\vspace{-1em}
\centering
\begin{tabular}{c|cccc}
     & Oscillator 2D & Heat Eq. & Surface Reaction & Oscillator 3D \\
     \hline
    n & $10000$ & $10000$ & $10000$ & $51^3$ \\
    poly. deg. & 20 & 20 & 20 & 10 \\
    \hline
    Bernoulli (a) & 574 & 554 & 545 & 671 \\
    Pivotal (b) & 398 & 395 & 390 & 533 \\
    \hline
    Efficiency (b / a) & 0.693 & 0.713 & 0.716 & 0.794
\end{tabular}
\caption{Number of samples needed to achieve $2 \times \text{OPT}$ error.}
\label{table:2opt}
\vspace{-1em}
\end{table}

\begin{table}[h!]
\vspace{-1em}
\centering
\begin{tabular}{c|cccc}
     & Oscillator 2D & Heat Eq. & Surface Reaction & Oscillator 3D \\
     \hline
    n & $10000$ & $10000$ & $10000$ & $51^3$ \\
    poly. deg. & 12 & 12 & 12 & 10 \\
    \hline
    Bernoulli (a) & 924 & 814 & 903 & 3121 \\
    Pivotal (b) & 450 & 442 & 492 & 1943 \\
    \hline
    Efficiency (b / a) & 0.487 & 0.523 & 0.545 & 0.623
\end{tabular}
\caption{Number of samples needed to achieve $1.1 \times \text{OPT}$ error.}
\label{table:1.1opt}
\vspace{-1em}
\end{table}

\begin{figure}[t]
\vspace{-1em}
    \centering
    \begin{subfigure}[b]{0.49\textwidth}
        \centering
        \includegraphics[width=.7\linewidth]{NeurIPS/plots/appendix_spring3D.png}
        \vspace{-.5em}
        \caption{3D Damped Harmonic Oscillator.}    
    \end{subfigure}
    \centering
    \begin{subfigure}[b]{0.49\textwidth}
        \centering
        \includegraphics[width=.7\linewidth]{NeurIPS/plots/appendix_unifrom_vs_leverage_p15.png} % I named the file incorrectly. It says p15 but the actual polynomial degree is 12. (Atsushi)
        \vspace{-.5em}
        \caption{Leverage Score v.s. Uniform Sampling.}    
    \end{subfigure}
    \hfill
    \centering
    \begin{subfigure}[b]{0.49\textwidth}
        \centering
        \includegraphics[width=.7\linewidth]{NeurIPS/plots/appendix_randBSS_p15.png}
                \vspace{-.5em}
        \caption{Randomized BSS.}    
    \end{subfigure}
    \caption{(a): Results for degree $12$ active polynomial regression for the damped harmonic oscillator QoI in 3D space. (b), (c): Results for degree $12$ active polynomial regression for the damped harmonic oscillator QoI in 2D space. (b) includes the performance of the Bernoulli sampling and our pivotal sampling both using uniform inclusion probability which is given with the thin lines. (c) demonstrates the approximation power of the randomized BSS algorithm. We run the sampling $2000$ times with different parameters, round the number of samples to ten place, and take the median error.}
    \label{fig:appendix:3plots}
    \vspace{-1.5em}
\end{figure}

\noindent \textbf{Leverage Score Sampling v.s. Uniform Sampling.} We also conduct a complementary experiment to show empirically that leverage score sampling is much more powerful than uniform sampling. We extend the simulation in Figure \ref{fig:sec1:comparison} to a uniform inclusion probability setting. This time, we draw $350$ samples, repeat the simulation $100$ times, and report the approximation with a median error in Figure \ref{fig:appendix:approximationexample}. The result shows poor performance of the uniform sampling. As they draw fewer samples near the boundaries compared to leverage score sampling, they are not able to pin down the polynomial function near the edges, resulting in suffering large errors in these areas. The relative error plot is given in Figure \ref{fig:appendix:3plots} (b). We point to three observations. 1) Comparing the thick lines and the thin lines, one can tell that the use of leverage score significantly improves performance. 2) Comparing the orange lines and the blue lines, one can see that our \textit{spatially-aware} pivotal sampling outperforms the Bernoulli sampling. 3) By combining the leverage score and \textit{spatially aware} pivotal sampling (our method), one can attain the best approximation among the four sampling strategies.

\begin{figure}[h]
    \vspace{-1em}
    \centering
    \begin{subfigure}[b]{0.24\textwidth}
        \centering
        \includegraphics[width=.9\linewidth]{NeurIPS/plots/appendix_spring_p20_350_bernoulli_uniform.png}
                \vspace{-.5em}
        \caption{Bernoulli Uniform.}    
    \end{subfigure}
    \hfill
    \centering
    \begin{subfigure}[b]{0.24\textwidth}
        \centering
        \includegraphics[width=.9\linewidth]{NeurIPS/plots/appendix_spring_p20_350_bernoulli_leverage.png}
                \vspace{-.5em}
        \caption{Bernoulli Leverage.}    
    \end{subfigure}
    \hfill
    \centering
    \begin{subfigure}[b]{0.24\textwidth}
        \centering
        \includegraphics[width=.9\linewidth]{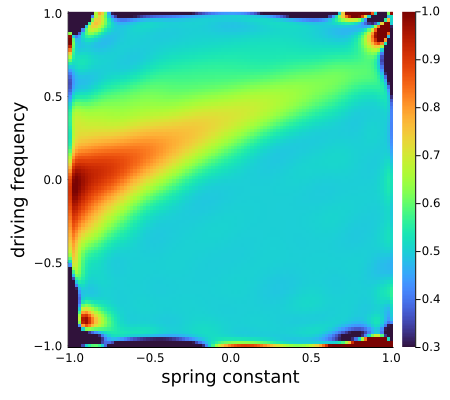}
                \vspace{-.5em}
        \caption{Pivotal Uniform.}    
    \end{subfigure}
    \hfill
    \centering
    \begin{subfigure}[b]{0.24\textwidth}
        \centering
        \includegraphics[width=.9\linewidth]{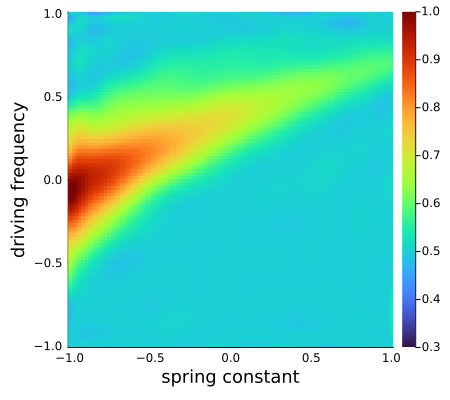}
                \vspace{-.5em}
        \caption{Pivotal Leverage.}    
    \end{subfigure}
    \caption{
    Visualizations of a polynomial approximation to the maximum displacement of a damped harmonic oscillator, as a function of driving frequency and spring constant. $\mathcal{X}$ is a uniform distribution over a box. All four draw $350$ samples but by different sampling methods; (a) and (b) use Bernoulli sampling but (c) and (d) use pivotal sampling. Also, (a) and (c) select samples with uniform probability while (b) and (d) employs the leverage score. Clearly, the leverage score successfully pins down the polynomial function near the boundaries, resulting in better approximation.}
    \label{fig:appendix:approximationexample}
    \vspace{-1em}
\end{figure}

\noindent \textbf{Randomized BSS Algorithm.} Finally, we run the randomized BSS algorithm \cite{CP19} on the 2D damped harmonic oscillator target function, and report the relative error together with Bernoulli leverage score sampling and our pivotal sampling in Figure \ref{fig:appendix:3plots} (c). The setting is the same as the Section \ref{sec:experiments} that the initial data points are drawn uniformly at random from $[-1, 1]^2$ box, and the polynomial degree is set to $12$. Even though this algorithm has a theoretical sample complexity of $O(d/\epsilon)$ for the regression guarantee in (\ref{eq:main_thm_gaur}), our sampling method achieves a certain relative error with fewer samples at least in this problem, holding out a hope that one might be able to remove the $\log d$ factor in our sample complexity.

\section{Conditional Negative Correlation} 
\label{sec:appendix:cnpc}

Here, we offer a direct proof of the conditional negative correlation property used in \ref{sec:appendix:l_infty} for pivotal sampling, without resorting to Conditional Negative Assoication property. Specifically, for a subset of binary random variables $\mathcal{S} \subseteq \bm{\xi}$ that has a positive probability to set $\xi = 1$, $\forall \xi \in \mathcal{S}$, conditional negative correlation property means that, for any two $\xi_i, \xi_j \in \bm{\xi}\backslash \mathcal{S}$, 
\begin{align}
    \prob{\xi_i = 1 | \xi_l = 1 \forall l \in \mathcal{S}} \prob{\xi_j = 1 | \xi_l = 1 \forall l \in \mathcal{S}} \geq  \prob{\xi_i = 1, \xi_j = 1| \xi_l = 1 \forall l \in \mathcal{S}}.
\end{align}

We actually get a more abundant structure with conditioning. We will prove that conditioning on a subset being selected, the conditional distribution also follows from pivotal sampling, with different initial inclusion probabilities. 

\begin{theorem}\label{thm: cpnc_main}
    $\bm \xi$ denotes the binary random vector of all the leaf nodes in a binary-tree-based pivotal sampling. Conditioned on that a subset $\mathcal{S} \subseteq \bm{\xi}$,  are set to $1$, the distribution of remaining binary random variables $\bm{\xi} \backslash\mathcal{S}$ can be obtained from binary-tree-based pivotal sampling, where the leaf nodes are prescribed with some probabilities.
\end{theorem}

\cite{DJR05} proves that the binary-tree-based pivotal sampling satisfies the negative association (NA) property, which includes and generalizes the negative correlation property. 
% not consistent with intro...?
\begin{proposition}[\cite{DJR05}]
    The binary random variables corresponding to the leaf nodes in a binary-tree-based pivotal sampling are negatively correlated.
\end{proposition}

So the conditional negative correlation property is a corollary of Theorem \ref{thm: cpnc_main}.
\begin{corollary}
    Conditioned on that all binary random variables in the subset $\mathcal{S}$ are set to $1$, the binary random variables in $\bm{\xi} \backslash \mathcal{S}$ are negatively correlated. In other words, for any two $\xi_i, \xi_j \in \bm{\xi} \backslash \mathcal{S}$, we have
    % remember to replace all \prob with \Pr.
    \begin{align}
    \prob{\xi_i = 1 | \xi_l = 1 \forall l \in \mathcal{S}} \prob{\xi_j = 1 | \xi_l = 1 \forall l \in \mathcal{S}} \geq  \prob{\xi_i = 1, \xi_j = 1| \xi_l = 1 \forall l \in \mathcal{S}}.
    \end{align}
\end{corollary}

In the following, we will prove the case where there is only one index in the set $\mathcal{S}$. Then by induction, we get the case in Theorem \ref{thm: cpnc_main}. In other words, we will prove the following theorem.
\begin{theorem} \label{thm: cpnc_one}
Conditioned on that one binary random variable corresponding to a leaf node $\xi_i = 1$, the distribution of remaining binary random variables $\bm{\xi} \backslash\{\xi_i\}$ can be obtained from binary tree based pivotal sampling, where the leaf nodes are prescribed with some probabilities.
\end{theorem}

To prove this, we will first elaborate more on structure of binary-tree-based pivotal sampling in Algorithm \ref{algo:pivotal}. We will then formulate binary-tree-based pivotal sampling from a new perspective. 

In Algorithm \ref{algo:pivotal}, the tree is sequentially pruned as sampling proceeds. We can actually preserve the tree structure along the sampling. We need to nail down some notations first.

\textbf{Notations.} We can assign a binary random variable $\xi$ for any non-root node in the tree. We will later talk about what it means for a non-leaf node to have $\xi = 0$ or $\xi = 1$. To denote a set or a vector of these binary random variables, we will use $\bm{\xi_S}$, where $S$ is a set or an ordering of the indices. We will abuse the notation $\xi$, for both the node in the tree, and the binary random variable corresponding to that node. The \emph{status} of a node is either $0$ or $1$. We call the structure containing one parent node, two children nodes and the two edges connecting the parent and the children as a \emph{branch}. At a branch, we will usually use $\xi_0$ to denote the parent node, and $\xi_1$ and $\xi_2$ to denote two children nodes.

The following proposition is easy to see and check from the pivotal sampling procedure, so we omit the proof.
\begin{proposition} \label{prop: pivotal}
    Algorithm \ref{algo:pivotal} is equivalent to the following procedure:

    Given the inclusion probabilities of the children nodes $\xi_1$, $\xi_2$ at each branch as $p_1, p_2$, we can determine the probability ``carried'' by the parent node.
    
    Specifically, if $p_1 + p_2 > 1$, the parent node $\xi_0$ carries probability $p_1+p_2 - 1$. If $p_1 + p_2 < 1$, the parent node $\xi_0$ carries probability $p_1 + p_2$. 

    For a non-leaf node $\xi$ being set to $1$, it means that in the sampling procedure, the leaf node promoted to $\xi$ is included in the end. For $\xi$ set to $0$, it means that the leaf node promoted to $\xi$ is not included. 

    We can check that for a non-leaf node $\xi$, the probability $\prob{\xi = 1}$ is exactly the probability it carries, just as the case for the leaf nodes. So we will simply say that ``carried'' probability as ``the probability of $\xi$''.

    Since we can know the probability of a node from its children, we can go from leaf to root, and know the probability of all nodes in the tree.

    After  setting the probability of all nodes in the tree, we can start pivotal sampling.   
    Specifically, first, two leaves, say $\xi_i$ and $\xi_j$ take the match, with the probability of these two nodes. Say after the match, $\xi_i$ is set to be chosen or not chosen, it means that we will set the binary random variable $\xi_i = 0$ or $1$. Instead of pruning the two nodes $\xi_i$ and $\xi_j$, we keep the tree structure. We store $j$ to the parent node, meaning that later match results of the parent node will go to set the status of $\xi_j$.

    For the match between two non-leaf nodes, the probabilities involved are the probabilities of the two nodes.

    In the end, we do the matches from leaves to the root, and set the statuses of all nodes in the tree as either $0$ or $1$.
\end{proposition}

{\remark In Proposition \ref{prop: pivotal}, we only talk about the case where $p_1 + p_2 \neq 1$. In fact, if $p_1 + p_2 = 1$ with corresponding nodes $\xi_1$ and $\xi_2$, then after this match, $(\xi_1, \xi_2)$ is set to be $(1, 0)$ or $(0, 1)$ already. We don't need further matches to decide their status. In other words, the subtree rooted at at the parent node of $\xi_1$ and $\xi_2$ is independent of the matches outside of this subtree. And this subtree implies that the distribution of leaf nodes in this subtree follows from a binary-tree-based pivotal sampling. In this case, we single out the subtree before sampling, and let trees do pivotal sampling independently. So, we will assume that $p_1 + p_2 \neq 1$, except when $\xi_1$ and $\xi_2$ are the two children nodes of the root.

We will also assume that $0 < p_1, p_2 < 1 $ for all pairs of sibling nodes. For $p_1, p_2$ not satisfying this, some subtrees can be separated from the original tree and perform pivotal sampling independently.  \label{remark: pivotal}
}

In Algorithm \ref{algo:pivotal}, or equivalently, Proposition \ref{prop: pivotal},
we percolate up the tree from the leaf nodes, and perform repeated head-to-head comparisons of indices at sibling nodes. We call this the ``bottom-up'' perspective of pivotal sampling. In fact, there is an equivalent way of performing pivotal sampling, which we call as the ``top-down'' perspective. We will state this new perspective in Proposition \ref{prop: top-down}, after some preparations.  

First, a key property of pivotal sampling is the following. 

\begin{proposition}
    \label{prop: cond_indep} % "thus" part, should say that independent of all the nodes, rather than just leaf nodes
    In binary-tree-based pivotal sampling, for any non-leaf node, conditioned on its status, 
    % need some explanations of algorithm 1. It is too simplified.
    the distribution of all the binary random variables corresponding to the nodes in the subtree rooted at this non-leaf node is independent of the conditional distribution of the nodes outside of this subtree. 
\end{proposition}
%{\remark{Several original leaf nodes can be candidates which is finally stored at this non-leaf node location after some matches. }}
% more words might be needed here. % also, connection with algorithm 1.

Proposition \ref{prop: cond_indep} follows from the following key observation.
\begin{proposition} 
 % do we also need to say that the results are thus independent... ?
\label{prop: same_prob}
    For a non-leaf node $\xi$ in the binary tree, the probability it carries to interact with its sibling node is a constant. Specifically, it is independent of the match results in the subtree rooted at $\xi$.
\end{proposition}
This follows from Proposition \ref{prop: pivotal}, where we can set the probability just by the inclusion probabilities of leaf nodes.

\begin{lemma}   \label{lem: status_matches}
    Fixing a non-leaf node being $0$ or $1$, then the final status ($0$ or $1$) of any leaf nodes  
    % (need to say whether it is just leaf or all. If it is all, then there might be several nodes in the subtree that is finally determined by the final choice)
    % not sure whether it is good to say "any nodes" -- I want to include the case where we talk about multiple nodes at the same time
    in the subtree rooted at this node is completely determined by the matches inside this subtree.
\end{lemma}
% maybe we need to define match results.
\begin{proof} [Proof of Lemma \ref{lem: status_matches}]
    The matches inside the subtree will determine all the binary random variables being either $0$ or $1$, except only one that is promoted to the root of the subtree. Which one is promoted to the root is also determined by the matches inside of the subtree. Whether the root of the subtree 
    $\xi_0$ 
    % maybe say about this abuse of term about the binary random variable vs the node itself
    is set to $0$ or $1$ is independent of the match results inside of the subtree, since it carries a probability that is invariant with respect to the match results inside the subtree, by Proposition \ref{prop: same_prob}. By fixing the root being $0$ or $1$, we know the final status of all the leaf nodes in this subtree.
\end{proof}

With these propositions and lemmas, we can prove Proposition \ref{prop: cond_indep}.

\begin{proof}[Proof of Proposition \ref{prop: cond_indep}]

    Fix a non-leaf node $\xi_0$. 
    Suppose $\bm{\xi_I}$ is a vector of nodes in the subtree rooted at $\xi_0$. $\bm{\xi_O}$ is a vector of nodes outside the subtree. 
    Denote a binary-valued scalar as $\tilde \xi_0$, two binary-valued vector as $\bm{\tilde \xi_I}$ and $\bm{\tilde \xi_O}$ with the same length as $\bm{\xi_I}$ and $\bm{\xi_O}$.

    We want to prove that
    \[
        \prob{\bm{\xi_I} = \bm{\tilde \xi_I}, \bm{\xi_O} = \bm{\tilde \xi_O} | \xi_0 =\tilde \xi_0} = \prob{\bm{\xi_I} = \bm{\tilde \xi_I}| \xi_0 =\tilde \xi_0}\prob{\bm{\xi_O} = \bm{\tilde \xi_O}| \xi_0 =\tilde \xi_0}.
    \]

        % maybe we should decide whether it is about matches or about the statement here. Also, whether to include the nodes in the binary tree rather than just the leaf node.
    
    By Lemma \ref{lem: status_matches}, we denote the set of the match results inside the subtree of getting $\bm{\xi_I} = \bm{\tilde \xi_I}$ as $\mathcal{M}(\bm{\xi_I} = \bm{\tilde \xi_I}, \xi_0 = \tilde \xi_0)$.  Also by Lemma \ref{lem: status_matches}, we know that
    \begin{align}
\prob{\bm{\xi_O} = \bm{\tilde \xi_O}| \xi_0 =\tilde \xi_0, \bm{\xi_I} = \bm{\tilde \xi_I}} = \prob{\bm{\xi_O} = \bm{\tilde \xi_O}| \xi_0 =\tilde \xi_0, \mathcal{M}(\bm{\xi_I} = \bm{\tilde \xi_I}, \xi_0 = \tilde \xi_0)}.
\end{align}

    By Proposition \ref{prop: same_prob},
    \begin{align}
        \prob{\bm{\xi_O} = \bm{\tilde \xi_O}| \xi_0 =\tilde \xi_0, \mathcal{M}(\bm{\xi_I} = \bm{\tilde \xi_I}, \xi_0 = \tilde \xi_0)} = \prob{\bm{\xi_O} = \bm{\tilde \xi_O}| \xi_0 =\tilde \xi_0 }.
    \end{align}
    Therefore,
    \begin{align}
    \prob{\bm{\xi_O} = \bm{\tilde \xi_O}| \xi_0 =\tilde \xi_0, \bm{\xi_I} = \bm{\tilde \xi_I}} = \prob{\bm{\xi_O} = \bm{\tilde \xi_O}| \xi_0 =\tilde \xi_0 }.
    \end{align}
    By rewriting the left-hand side of this equality, we can see that it implies the anticipated equality.
    \end{proof}

    % maybe need to think about the statement that the match is independent of whether \xi_0 = 0 or 1. implications of choosing certain nodes...
    Now, we give a more detailed description of the subtree in Proposition \ref{prop: cond_indep}.
    \begin{proposition}\label{prop: one-step_cond}
        For a given non-leaf node $\xi_0$, the conditional distribution of its two children nodes $\xi_1, \xi_2$ with weight $p_1, p_2$ is as follows:
        \begin{enumerate}[(a)]
            \item If $p_1 + p_2 > 1$, then 
            \begin{align}
                \prob{(\xi_1, \xi_2) = (1, 1) | \xi_0 = 1} = 1,
            \end{align}
            and
            \begin{align}
                \prob{(\xi_1, \xi_2) = (1,0) | \xi_0 = 0} = \frac{1 - p_2}{ 2 - p_1 - p_2}, \quad \prob{(\xi_1, \xi_2) = (0,1) | \xi_0 = 0} = \frac{ 1 - p_1} { 2 - p_1 - p_2}.
            \end{align}
            \item If $p_1 + p_2 < 1$, then
            \begin{align}
                \prob{(\xi_1, \xi_2) = (1, 0) | \xi_0 = 1} = \frac{p_1}{p_1 + p_2}, \quad \prob{(\xi_1, \xi_2) = (0, 1) | \xi_0 = 1} = \frac{p_2}{p_1 + p_2},
            \end{align}
            and
            \begin{align}
                \prob{ (\xi_1, \xi_2) = (0, 0) | \xi_0 = 0} = 1.
            \end{align}
            \item If $p_1 + p_2 = 1$, then as mentioned in Remark \ref{remark: pivotal}, $\xi_0$ must be the root node. There is no match for the root. So we can talk about the probabilities without conditioning. For ease of terms later, we will include the following as one of the conditional distribution of two children nodes given the status of the parent.
            \begin{align}
                \prob{ (\xi_1, \xi_2) = (1, 0)} = p_1, \quad \prob{ (\xi_1, \xi_2) = (0, 1) | \xi_0 = 1} = p_2.
            \end{align}
        \end{enumerate}
    \end{proposition}
    
\begin{proof}[Proof of Proposition \ref{prop: one-step_cond}]

 \begin{enumerate}[(a)] % the case where $p_1 = p_2 = 0$? eg. because both parts have sum = 1 already? Can think of it as an independent pivotal...
    \item $p_1 + p_2 > 1$. The match between them will set one of them to be $1$, and promote the other to the parent node. 
    So if we know that the parent $\xi_0 = 1$, then both children nodes are set to $1$, i.e.\,$ \prob{(\xi_1, \xi_2) = (1, 1) | \xi_0 = 1} = 1$. 
    
    On the other hand, if we condition on $\xi_0 = 0$, it means that the yet undetermined node after the match between the children nodes is set to $0$ in the end. With probability $(1-p_2)/(2- p_1 - p_2)$, $\xi_1$ will be set to $1$ by the match between the children nodes; with probability $(1- p_1)/ (2 - p_1 - p_2)$, $\xi_2$ will be set to $1$ by the match. Therefore, $\prob{(\xi_1, \xi_2) = (1,0) | \xi_0 = 0} = \frac{1 - p_2}{ 2 - p_1 - p_2}$, and $\prob{(\xi_1, \xi_2) = (0,1) | \xi_0 = 0} = \frac{ 1 - p_1} { 2 - p_1 - p_2}$.

    \item $p_1 + p_2 < 1$. The match between them will set one node to be $0$, and promote the other to the parent node. 
    So, if the parent $\xi_0 = 1$, it means the yet undetermined node after the match is finally set to $1$. With probability $p_1/ ( p_1 + p_2)$, $\xi_1$ will be promoted to the parent; instead, with probability $p_2 / (p_1 + p_2)$, $\xi_2$ will be promoted to the parent. Therefore, $\prob{(\xi_1, \xi_2) = (1, 0) | \xi_0 = 1} = \frac{p_1}{p_1 + p_2}$, $\prob{(\xi_1, \xi_2) = (0, 1) | \xi_0 = 1} = \frac{p_2}{p_1 + p_2}$.

    On the other hand, if we condition on $\xi_0 = 0$, it means that the undetermined node after the match is also set to $0$ in the end. So, $\prob{ (\xi_1, \xi_2) = (0, 0) | \xi_0 = 0} = 1$.

    \item $p_1 + p_2 = 1$. It must be at the root node. And the conclusion is trivial.
 \end{enumerate}
\end{proof}

% need to prove that or state that condition on several, still independent. Or in proposition 4, say that the matches on either side are (conditionally) independent, so all things that are functions of the matches are also conditionally independent. % or prove the conditioning on several by indudction / recursively?

We prove a multiple-level version of Proposition \ref{prop: one-step_cond}. 

\begin{proposition} \label{prop: multi-step_cond}
 For a given subtree, conditioned on the status of the parent node $\xi_0$, the probability that the leaf nodes take certain statuses can be expressed by conditional probabilities at each branches in Proposition \ref{prop: one-step_cond}. Specifically, if we denote the vector of all the leaf nodes as $\bm{\xi_f}$, then for a given binary value $\tilde \xi_0$ and a binary vector $\bm{\tilde \xi_f}$ with the same length as $\bm{\xi_f}$,
 % f means leaf, not the level.
 \begin{align}
 \prob{ \bm{\xi_f} = \bm{\tilde \xi_f} | \xi_0 = \tilde \xi_0}  = \sum_{\tilde \xi^{(i)}_1, \tilde \xi^{(i)}_2, \tilde \xi^{(i)}_0 \text{binary scalars}} \prod_{i} \prob{(\xi^{(i)}_1, \xi^{(i)}_2) = (\tilde \xi^{(i)}_1, \tilde \xi^{(i)}_2) | \xi^{(i)}_0 = \tilde \xi^{(i)}_0},
 \end{align}
 where $\xi_0^{(i)}$ is the parent node of a branch in the binary tree, and $\xi_1^{(i)}, \xi_2^{(i)}$ are its two children.
\end{proposition}
\begin{proof}[Proof of Proposition \ref{prop: multi-step_cond}]
    First, we assume that all the leaf nodes are on the same level $l$ counting from the root (the root is level $0$). In other words, the subtree is complete.  We prove by induction on the level $l$. The base case $l=1$ is covered in Proposition \ref{prop: one-step_cond}. Suppose we have proved the statement for level $l-1$. The at level $l$, denote the vector of leaf nodes as $\bm{\xi_l}$. For a binary vector $\bm{\tilde \xi_l}$ matching the length, we have
    \begin{align} 
            \prob{ \bm{\xi_l} = \bm{\tilde \xi_l}  | \xi_0 = \tilde \xi_0}
         =  \sum_{\bm{\tilde \xi_{l-1}}}\prob{\bm{ \xi_{l}} = \bm{\tilde \xi_{l}} | \bm{ \xi_{l-1}} = \bm{\tilde \xi_{l-1}}, \xi_0 = \tilde \xi_0}
           \prob{(\bm{\xi_{l-1}}) = \bm{\tilde \xi_{l-1}}| \xi_0 = \tilde \xi_0}.
    \end{align}
        By induction hypothesis, the second term can be expressed by sum of products of conditional probabilities at and before level $l-1$. Let's look at the first term. 

        Since the binary tree is complete, $\xi_{l-1,1}$ is the parent of $(\xi_{l,1}, \xi_{l,2})$, $\xi_{l-1, 2}$ is the parent of $(\xi_{l,3}, \xi_{l,4})$, \dots, $\xi_{l-1, k_{l-1}}$ is the parent of $(\xi_{l,k_l-1}, \xi_{l, k_l})$. 

        We have the following lemma.

        \begin{lemma} \label{lem: decomp}
        \begin{align}
                \prob{\bm{ \xi_{l}} = \bm{\tilde \xi_{l}} |\bm{\xi_{l-1}} =\bm{\tilde \xi_{l-1}}, \xi_0 = \tilde \xi_0}
             = \prod_{i = 1}^{k_{l-1}}\prob{(\xi_{l,2i-1}, \xi_{l, 2i}) = (\tilde \xi_{l, 2i-1}, \tilde \xi_{l, 2i}) | \xi_{l-1,i} = \tilde \xi_{l-1, i}}
        \end{align}
        \end{lemma}
        \begin{proof}[Proof of Lemma \ref{lem: decomp}]
            \begin{align}
                \begin{split}
                \prob{\bm{ \xi_{l}} = \bm{\tilde \xi_{l}} |\bm{\xi_{l-1}} =\bm{\tilde \xi_{l-1}}, \xi_0 = \tilde \xi_0} 
             =  &\prob{(\xi_{l,1}, \xi_{l,2}) = (\tilde \xi_{l,1}, \tilde \xi_{l,2}) | \bm{ \xi_{l-1}}= \bm{\tilde \xi_{l-1}}, \xi_0 = \tilde \xi_0} \\
              & \times \prob{ \bm{ \xi_{l \backslash\{1, 2\}}} = \bm{ \tilde \xi_{l \backslash\{1, 2\}}}| \bm{ \xi_{l-1}}= \bm{\tilde \xi_{l-1}}, \xi_0 = \tilde \xi_0}.
              \end{split}
             \end{align}

           By the definition of conditional probability, the first term on the right-hand side can reformulated as
            \begin{align} 
                \begin{split}
                    & \prob{(\xi_{l,1}, \xi_{l,2}) = (\tilde \xi_{l,1}, \tilde \xi_{l,2}) | (\xi_{l-1,1}, \xi_{l-1,2}, \dots, \xi_{l-1, k_{l-1}}) = \tilde \xi_{l-1}, \xi_0 = \tilde \xi_0} \\
                =  & \prob{(\xi_{l,1}, \xi_{l,2}) = (\tilde \xi_{l,1}, \tilde \xi_{l,2}), (\xi_{l-2}, \dots, \xi_{l-1, k_{l-1}}) = \tilde \xi_{l-1, [k_{l-1}] \backslash 1}, \xi_0 = \tilde \xi_0 | \xi_{l-1,1} = \tilde \xi_{l-1,1}} \\
                & / \prob{(\xi_{l-2}, \dots, \xi_{l-1, k_{l-1}}) = \tilde \xi_{l-1, [k_{l-1}] \backslash 1}, \xi_0 = \tilde \xi_0 | \xi_{l-1,1} = \tilde \xi_{l-1,1}}.
                \end{split}
            \end{align}
               By the conditional independence in Proposition \ref{prop: cond_indep}, 
             \begin{align}
                \begin{split}
                        & \prob{(\xi_{l,1}, \xi_{l,2}) = (\tilde \xi_{l,1}, \tilde \xi_{l,2}), (\xi_{l-2}, \dots, \xi_{l-1, k_{l-1}}) = \tilde \xi_{l-1, [k_{l-1}] \backslash 1}, \xi_0 = \tilde \xi_0 | \xi_{l-1,1} = \tilde \xi_{l-1,1}} \\
                    = & \prob{(\xi_{l,1}, \xi_{l,2}) = (\tilde \xi_{l,1}, \tilde \xi_{l,2}) | \xi_{l-1,1} = \tilde \xi_{l-1,1}} \prob{ (\xi_{l-2}, \dots, \xi_{l-1, k_{l-1}}) = \tilde \xi_{l-1, [k_{l-1}] \backslash 1}, \xi_0 = \tilde \xi_0 | \xi_{l-1,1} = \tilde \xi_{l-1,1}}.
                \end{split}
             \end{align}
             Therefore,
             \begin{align}
                \begin{split}
                    & \prob{(\xi_{l,1}, \xi_{l,2}) = (\tilde \xi_{l,1}, \tilde \xi_{l,2}) | (\xi_{l-1,1}, \xi_{l-1,2}, \dots, \xi_{l-1, k_{l-1}}) = \tilde \xi_{l-1}, \xi_0 = \tilde \xi_0} \\
                    = & \prob{(\xi_{l,1}, \xi_{l,2}) = (\tilde \xi_{l,1}, \tilde \xi_{l,2}) | \xi_{l-1,1} = \tilde \xi_{l-1,1}}.
                \end{split}
             \end{align}
             So,
              \begin{align}
                 \begin{split}
                & \prob{(\xi_{l,1}, \xi_{l,2}, \dots, \xi_{l, k_l}) = \tilde \xi_{l} | (\xi_{l-1,1}, \xi_{l-1,2}, \dots, \xi_{l-1, k_{l-1}}) = \tilde \xi_{l-1}, \xi_0 = \tilde \xi_0} \\
             = & \prob{(\xi_{l,1}, \xi_{l,2}) = (\tilde \xi_{l,1}, \tilde \xi_{l,2}) | \xi_{l-1,1} = \tilde \xi_{l-1,1}} \\
             & \times \prob{(\xi_{l,3}, \dots, \xi_{l,k_l}) = (\tilde \xi_{l,3}, \dots, \tilde \xi_{l,k_l}) | (\xi_{l,1}, \xi_{l,2}) = (\tilde \xi_{l,1}, \tilde \xi_{l,2}), \xi_{l-1} = \tilde \xi_{l-1}, \xi_0 = \tilde \xi_0}.
                 \end{split}
            \end{align}
        By recursively singling out $(\xi_{l, 2i-1}, \xi_{l, i})$, we can finally prove the lemma.
        \end{proof}

        Now, let's go back to the proof of Proposition \ref{prop: multi-step_cond}. By Lemma \ref{lem: decomp},
        \begin{align} 
        \begin{split}
            & \prob{(\xi_{l,1}, \xi_{l,2}, \dots, \xi_{l, k_l}) = \tilde \xi_{l} | \xi_0 = \tilde \xi_0} \\
         = & \sum_{\tilde \xi_{l-1}}\prob{(\xi_{l,1}, \xi_{l,2}, \dots, \xi_{l, k_l}) = \tilde \xi_{l} | (\xi_{l-1,1}, \xi_{l-1,2}, \dots, \xi_{l-1, k_{l-1}}) = \tilde \xi_{l-1}, \xi_0 = \tilde \xi_0}\\
          & \times \prob{(\xi_{l-1,1}, \xi_{l-1,2}, \dots, \xi_{l-1, k_{l-1}}) = \tilde \xi_{l-1} | \xi_0 = \tilde \xi_0} \\
         = & \sum_{\tilde \xi_{l-1}}\left(\prod_{i = 1}^{k_{l-1}}\prob{(\xi_{l,2i-1}, \xi_{l, 2i}) = (\tilde \xi_{l, 2i-1}, \tilde \xi_{l, 2i}) | \xi_{l-1,i} = \tilde \xi_{l-1, i}}\right)\\
          & \times \prob{(\xi_{l-1,1}, \xi_{l-1,2}, \dots, \xi_{l-1, k_{l-1}}) = \tilde \xi_{l-1} | \xi_0 = \tilde \xi_0}.
        \end{split}
    \end{align}
    By induction hypothesis applied on the last term on the right hand side, we prove the proposition for the case where all the leaf nodes are on the same level.

    For the case where the leaf nodes are on several different levels, as long as we can prove the statement with the leaf nodes being on two levels, the general case follows from induction. 

    We will now prove that if all the leaf nodes of a subtree are in two different levels, then Proposition \ref{prop: multi-step_cond} also holds.

    Say the two levels are $i < j$. $\xi_{I_i}$ is a vector % indices or not... we use it in the subscript...
    of all the leaf nodes at level $i$. Similarly for $I_j$. Denote the set of the nodes at level $i$ that are ancestors of $\xi_{I_j}$ as $\xi_{I_j^i}$. % root node... conditional probabilities
    \begin{align}
        \begin{split}
                & \prob{\xi_{I_i} = \tilde{\xi}_{I_i}, \xi_{I_j} = \tilde \xi_{I_j} | \xi_0 = \tilde \xi_0}\\
            = & \sum_{\tilde \xi_{I_j^i}} \prob{\xi_{I_i} = \tilde{\xi}_{I_i}, \xi_{I_j} = \tilde \xi_{I_j}, \xi_{I_j^i} = \tilde \xi_{I_j^i} | \xi_0 = \tilde \xi_0} \\
            = & \sum_{\tilde \xi_{I_j^i}} \prob{\xi_{I_i} = \tilde{\xi}_{I_i}, \xi_{I_j^i} = \tilde \xi_{I_j^i} | \xi_0 = \tilde \xi_0} \prob{\xi_{I_j} = \tilde \xi_{I_j} | \xi_{I_i} = \tilde{\xi}_{I_i}, \xi_{I_j^i} = \tilde \xi_{I_j^i}, \xi_0 = \tilde \xi_0}.
        \end{split}
    \end{align}
    The first term boils down to the case where all the leaf nodes are on the same level, and we have proved that. By the conditional independence in Proposition \ref{prop: cond_indep}, and a similar proof as in Lemma \ref{lem: decomp}, we can prove that the second term can be decomposed as
    \begin{align} 
     \prob{\xi_{I_j} = \tilde \xi_{I_j} | \xi_{I_i} = \tilde{\xi}_{I_i}, \xi_{I_j^i} = \tilde \xi_{I_j^i}, \xi_0 = \tilde \xi_0} 
         = \prod_{m = 1}^{|I_i^j|} \prob{\xi_{I_j, m}=  \tilde \xi_{I_j, m}| \xi_{I_i^j, m} = \tilde \xi_{I_i^j, m}}. % defs...
    \end{align}
    Each of $\prob{\xi_{I_j, m}=  \tilde \xi_{I_j, m}| \xi_{I_i^j, m} = \tilde \xi_{I_i^j, m}}$ is a quantity about a subtree rooted at $\xi_{I_j^i, m}$. $\xi_{I_j, m}$ are all the leaf nodes in the subtree, and they are on the same level. So we can use the statement already proved. Therefore, each term is a sum of product of conditional probabilities at some branches. $\prob{\xi_{I_j} = \tilde \xi_{I_j} | \xi_{I_i} = \tilde{\xi}_{I_i}, \xi_{I_j^i} = \tilde \xi_{I_j^i}, \xi_0 = \tilde \xi_0}$, as the product of these terms, is also a sum of product of conditional probabilities. Going back to $\prob{\xi_{I_i} = \tilde{\xi}_{I_i}, \xi_{I_j} = \tilde \xi_{I_j} | \xi_0 = \tilde \xi_0}$, we know that it is also a sum of product of conditional probabilities.    
\end{proof}

We can say that the joint distribution of the leaf nodes in a pivotal sampling can be expressed by these conditional probabilities at branches.

\begin{proposition} \label{prop: joint}
       $\prob{\xi_f = \tilde \xi_f}$ can be expressed by the conditional probabilities at branches. % need to say about the root node
\end{proposition}

\begin{proof}[Proof of Proposition \ref{prop: joint}]
    If the tree only has two nodes except the root node, then the statement is easy. 
    In the non-trivial cases, denote the two nodes of the root being $\xi_1, \xi_2$, with probability $p_1$ and $p_2$ and $p_1 + p_2 = 1$. The leaf nodes that belong to the subtree rooted at $\xi_1$ are denoted as $\xi_{f_1}$. Similarly we define $\xi_{f_2}$.
    \begin{align} 
        \begin{split}
    \prob{\xi_f = \tilde \xi_f}  & = \prob{\xi_f = \tilde \xi_f, (\xi_1, \xi_2) = (1,0)} + \prob {\xi_f = \tilde \xi_f, (\xi_1, \xi_2) = (0, 1)}   \\
            & = p_1 \prob{\xi_f = \tilde \xi_f | (\xi_1, \xi_2) = (1, 0)}+ p_2 \prob{\xi_f = \tilde \xi_f | (\xi_1, \xi_2) = (0, 1)}.
        \end{split}
    \end{align}
    By a similar proof as in Lemma \ref{lem: decomp},
    \begin{align}
        \prob{\xi_f = \tilde \xi_f | (\xi_1, \xi_2) = (1, 0)} = \prob{\xi_{f_1} = \tilde \xi_{f_1} | \xi_1 = 1} \prob{\xi_{f_2} = \tilde \xi_{f_2} | \xi_2 = 0}.
    \end{align}
    By Proposition \ref{prop: multi-step_cond}, both terms on the right-hand side can be expressed by the conditional probabilities at branches. A similar decomposition holds for $ \prob{\xi_{f_2} = \tilde \xi_{f_2} | \xi_2 = 0}$. Therefore, we proved the statement about $\prob{\xi_f = \tilde \xi_f}$.
\end{proof}

Along the way, we can see that the joint distribution of the leaf nodes resulting from pivotal sampling can actually be determined by only two things: the conditional independence property in Proposition \ref{prop: cond_indep}, and the conditional probabilities at each branch as in Proposition \ref{prop: one-step_cond}. With these two ingredients, the joint distribution of the leaf nodes can be determined as in Proposition \ref{prop: joint}. 

Based on Proposition \ref{prop: joint}, we have the following ``top-down'' perspective of binary tree based pivotal sampling. As before, we assume that if there is a subtree that can be separated from the tree, we have done that.

\begin{proposition}[top-down perspective of pivotal sampling] \label{prop: top-down}
    If we get a sample of % the status of
    leaf nodes in the following way, the joint distribution of leaf nodes is the same as binary tree based pivotal sampling:
    
    At the root, suppose the two children nodes $\xi_1$ and $\xi_2$ have (conditional) probabilities $p_1$ and $p_2$, and $p_1 + p_2 = 1$. We throw a coin with head probability $p_1$, to decide whether we set $(\xi_1, \xi_2)$ as $(1, 0)$ or $(0, 1)$. 
    
    Now we already set $\xi_1$ and $\xi_2$. We look at the branch where $\xi_1$ is the parent node. We throw a coin with head probability corresponding to the conditional probability as in Proposition \ref{prop: one-step_cond}, independently. % independent of ...
    We then decide the status of the children nodes of $\xi_1$. Similarly, at the branch with $\xi_2$ being the parent node, we throw a coin independently, with the corresponding head probability. 
    
    We can continue this way, from top (root) to bottom, to determine the status % 
    of the nodes at each level sequentially. At each branch, we throw a coin independently. If we reach the leaf node at one branch, then we stop tracing down that branch, but continue on other branches.
    Finally, we set the status of all the leaf nodes in binary values. This is a sample from the joint distribution of leaf nodes. 
\end{proposition}

\begin{proof}
    It is easy to see that conditional independence in Proposition \ref{prop: cond_indep} is preserved in this top-down procedure. % more proof?
    By % (the proof of)
    Proposition \ref{prop: joint}, we only need to check that the conditional probabilities at each branch in the top-down procedure is the same as the original formulation. It is also true, since we set the head probability of the coins correctly.
\end{proof}

We have proved that with knowing the leaf probabilities, we can construct from a bottom-up version of pivotal sampling, to an equivalent top-down version of pivotal sampling. However, in the above formulation, the conditional distribution of two children nodes given the status of the parent node, still depends on whether the probabilities of the children nodes sum up larger than or smaller than $1$. The probabilities of non-leaf nodes are obtained from bottom to top as in Proposition \ref{prop: pivotal}. If we want to completely go from top to bottom, we should set the parameters in the tree without going from the bottom. 

We can simplify the two cases of the conditional distribution to a quantity called ``category''. Instead from going from bottom to top, figuring out the probabilities of all nodes in the tree and then determining the conditional distribution, we can directly set which of the two cases of the conditional distribution this parent node is using.
% whether it corresponds to the bottom-up...

\begin{proposition}[Another top-down pivotal sampling formulation, with parameters $(c,p)$ at each node]\label{prop: cat}
    For each non-leaf % also non-root
    node $\xi_0$ in the binary tree, we set two parameters for it: one is ``category'' $c$, which can take a value either $0$ or $1$; another is the conditional probability (detailed below) $p \in (0,1)$. 

    Given the tuple $(c, p)$, we set the conditional probability of the branch with $\xi_0$ being the parent node and $\xi_1, \xi_2$ being the children nodes as follows:

    $c = 0$ corresponds to the case where in the bottom-up view, $\xi_0$ has two children whose probabilities sum up greater than $1$, and $\xi_0$ is a node corresponding to case (a) in Proposition \ref{prop: one-step_cond}. In case (a), only the conditional distribution with $\xi_0$ is not deterministic. $p$ is set to be the conditional probability $\prob{(\xi_1, \xi_2) = (1,0) | \xi_0= 0}$. 

    $c = 1$ corresponds to the case where in the bottom-up view, $\xi_0$ has two children whose probabilities sum up smaller than $1$, and $\xi_0$ is a node corresponding to case (b) in Proposition \ref{prop: one-step_cond}. In this case, only the conditional distribution with $\xi_0  = 1$ is not deterministic. $p$ is set to be the conditional probability $\prob{(\xi_1, \xi_2) = (1, 0) | \xi_0 = 1}$.

    The tuple $(c, p)$ at each non-leaf % and non-root
    node is the initial condition for the top-down view. 
    
    Claim: A binary tree with these parameters gives us samples equivalently as in the bottom-up pivotal sampling. % wording...
\end{proposition}
\begin{proof}[Proof of Proposition \ref{prop: cat}]
    We want to prove that following this parameter setting, we can get probabilities of each node in the bottom-up view. The probabilities should satisfy that:

    If a node is of category $0$, then the sum of probabilities $p_1$ and $p_2$ of its children nodes should be larger than $1$, and $p_1 + p_2 - 1$ should be the probability of the node.

    If a node is of category $1$, then the sum of probabilities $p_1$ and $p_2$ of its children nodes should be smaller than $1$, and $p_1 + p_2$ should be the probability of the node.

    Note that the probability of a node is indeed the inclusion probability of that node. % proof?

    For the two children of the root node, we will set the probabilities of them being $(p_1, p_2)$.

    Inductively, for a branch where we know the inclusion probability of the parent node as $p_0$, and its $(c,p)$ tuple,  we can calculate the inclusion probability of the children nodes as follows.

    If $c = 0$, then
    \begin{align}\begin{split}
    p_1 &= \prob{\xi_1 = 1} = \prob{\xi_0 = 0} \prob{(\xi_1, \xi_2) = (1, 0)|\xi_0 = 0} + \prob{\xi_ 0 = 1} \prob{(\xi_1, \xi_2) = (1, 1) | \xi_0 = 1} \\ &= (1-p_0) p + p_0,
    \end{split}\end{align}
    \begin{align}
        \begin {split}
        p_2 &= \prob{\xi_2 = 1} = \prob{\xi_0 = 0} \prob{(\xi_1, \xi_2) = (0, 1)|\xi_0 = 0} + \prob{\xi_ 0 = 1} \prob{(\xi_1, \xi_2) = (1, 1) | \xi_0 = 1} \\
            & = (1-p_0) (1-p) + p_0.
        \end{split}
    \end{align}
    The value of $p_1$ and $p_2$ satisfies that $p_1 + p_2 -1 = p_0$ and $\frac{1-p_2}{2-p_1-p_2} = p$.

    If $c = 1$, then
    \begin{align} 
        p_1   = \prob{\xi_1 = 1} = \prob{\xi_0 = 1} \prob{(\xi_1, \xi_2) = (1, 0)|\xi_0 = 1} = p_0 p,
    \end{align}
    \begin{align}
        p_@ = \prob{\xi_2 = 1} = \prob{\xi_0 = 1} \prob{(\xi_1, \xi_2) = (0, 1)|\xi_0 = 1} = p_0 (1-p).
    \end{align}
    $p_1$ and $p_2$ satisfy that $p_1 + p_2 = p_0$ and $\frac{p_1}{p_1 + p_2} = p$.
    So, from top to bottom, we can set the probabilities for all nodes. And the probabilities are the same as we start from the bottom leaf node probabilities.
\end{proof}

    Now, we come to see the effect of conditioning, on the binary tree for the pivotal sampling. %We will combine both the bottom-up and the top-down perspectives of pivotal sampling. % the binary tree with which formulation?

    % conditional independence, and is still a tree
    % category not changed ... from bottom-up view
    % only some of p's will be changed... maybe not need to say how to modify? need to      show that the remaining probabilities are between 0 and 1? 
\begin{proposition} \label{prop: one-step_cond1}
        At a branch where the parent node is $\xi_0$ of category $1$, and two children nodes are $\xi_1, \xi_2$, we have
            \begin{align}
            \prob{\xi_0 = 1|\xi_1 = 1} = 1, \quad \prob{\xi_2 = 0 | \xi_1 = 1} = 1.
            \end{align}

        Instead, at a branch where the parent node is of category $0$, we have
        \begin{align}
            \prob{\xi_0 = \xi_2 | \xi_1 = 1} = 1.
        \end{align}
\end{proposition}
\begin{proof}[Proof of Proposition \ref{prop: one-step_cond1}]
    When $\xi_0$ is of category $1$, $\prob{\xi_1 = 1 | \xi_0 = 0} = 0$.
    \begin{align}
        \prob{\xi_0 = 1| \xi_1 = 1} = \frac{\prob{\xi_0 = 1} \prob{\xi_1 = 1 | \xi_0 = 1}}{\prob{\xi_0 = 1} \prob{\xi_1 = 1 | \xi_0 = 1} + \prob{\xi_0 = 0} \prob{\xi_1 = 1| \xi_0 = 0}} = 1.
    \end{align}
     Thus, we have $\prob{\xi_0 = 0 | \xi_1 = 1} = 0$. Then,
     \begin{align}
        0 \leq \prob{\xi_2 = 0, \xi_0 = 0 | \xi_1 = 1} \leq \prob{\xi_0 = 0 | \xi_1 = 1} = 0 \Longrightarrow \prob{\xi_2 = 0, \xi_0 = 0 | \xi_1 = 1} = 0.
     \end{align}
     Therefore,
    \begin{align} 
        \begin{split}
        \prob{ \xi_2 = 0 | \xi_1 = 1} & = \prob{\xi_2 = 0, \xi_0 = 0 | \xi_1 = 1} + \prob{\xi_2 = 0, \xi_0 = 1 | \xi_1 = 1} \\
            & = \prob{\xi_2 = 0, \xi_0 = 1 | \xi_1 = 1}\\
            & = \prob{\xi_0 = 1 | \xi_1 = 1} \prob{\xi_2 = 0 | \xi_0 = 1, \xi_1 = 1} \\
            & = 1.
        \end{split}
    \end{align}

    When $\xi_0$ is of category $0$, 
    \begin{align}
        \begin{split}
        \prob{\xi_0 = 1, \xi_2 = 1 | \xi_1 = 1} & = \prob{\xi_0 = 1 | \xi_1 = 1} \prob{\xi_2 = 1 | \xi_0 = 1, \xi_1 = 1} \\
            & = \prob{\xi_0 = 1 | \xi_1 = 1}.
        \end{split}
    \end{align}
    \begin{align}
            \begin{split}
        \prob{\xi_0 = 0, \xi_2 = 0 | \xi_1 = 1} & = \prob{\xi_0 = 0 | \xi_1 = 1} \prob{\xi_2 = 0 | \xi_0 = 0, \xi_1 = 1} \\
            & = \prob{\xi_0 = 0 | \xi_1 = 1}.
        \end{split}
    \end{align}
        Therefore,
        \begin{align}
            \begin{split}
        \prob{\xi_0 = \xi_2 | \xi_1 = 1} & = \prob{\xi_0 = 1, \xi_2 = 1 | \xi_1 = 1} +  \prob{\xi_0 = 0, \xi_2 = 0 | \xi_1 = 1} \\
                & = \prob{\xi_0 = 1 | \xi_1 = 1} + \prob{\xi_0 = 0 | \xi_1 = 1} = 1.
            \end{split}
        \end{align}
\end{proof}

    By Proposition \ref{prop: one-step_cond1}, conditioning on a leaf node $\xi_i = 1$, if its parent node is of category $1$, we will also set the parent node as $1$. Then, we trace up to the parent of the parent node, recursively. After we set a node to $1$ and its parent node is of category $0$, we stop. Denote the last node set to $1$ in the above procedure as $\xi_1$, the parent node of category $0$ as $\xi_0$, and the other children node of $\xi_0$ as $\xi_2$. Also by Proposition \ref{prop: one-step_cond1}, $\prob{\xi_0 = \xi_2 | \xi_1 =1} = 1$. It means that conditioning on $\xi_1 = 1$, the status of $\xi_0$ and $\xi_2$ are always the same. So we can think of $\xi_0$ and $\xi_2$ as the same node. We can think of this as the following modification of the tree: denote the parent of $\xi_0$ as $\xi_0'$. We connect $\xi_0'$ with $\xi_2$. And we remove the original two edges connecting $\xi_0'$ and $\xi_0$, and also $\xi_0$ and $\xi_2$.

    In this way, we have dealt with the branch where the last node set to $1$ in the above procedure is one of the children nodes. Let's deal with the branches where one of the other nodes (denoted as $\xi_1$) set to $1$ in the above procedure is one of the children node. For $\xi_1$, we know that its parent $\xi_0$ is of category $0$, and $\xi_0$ is set to be $1$ already. By Proposition \ref{prop: one-step_cond1}, we know that $\xi_2 = 0$. By Proposition \ref{prop: cond_indep}, the conditional distribution of leaf nodes in the subtree rooted at $\xi_2$ follows from pivotal sampling.
    %%% need to elaborate?

    We can remove these subtrees from the original tree. We can also remove those nodes set to $1$ above as tracing up from $\xi_i$. So, up to now, all the modifications to the original tree are: removing the subtrees and those set-to-$1$ nodes, and replacing two edges by one edge. %%%

    Now, we will prove the conditional version of Proposition \ref{prop: cond_indep} in the remaining tree. %

    \begin{proposition} \label{prop: cond_indep1}
     For any non-leaf node $\xi_0$ in the remaining tree, conditioned on the status of it, the conditional distribution (with $\xi_i = 1$) of nodes inside of the subtree rooted at $\xi_0$ is independent of independent of the the conditional distribution (with $\xi_i = 1$) of nodes outside of the subtree. % in the remaining tree.
        % say the two conditioning at the same time
         % need to say that the nodes being inside vs outside is not changed after changing.
     Specifically, denote $\xi_I$ as a subset % vector
        % 0 vs O, easy to see?
     of nodes in the subtree, $\xi_O$ as a subset of nodes outside of the subtree. $\tilde \xi_0$ is a binary valued scalar. $\tilde \xi_I$ and $\tilde \xi_O$ are two binary vectors. Then,
        \begin{align}
            \prob{\xi_I = \tilde \xi_I, \xi_O = \tilde \xi_O | \xi_0 = \tilde \xi_0, \xi_i = 1} = \prob{\xi_I = \tilde \xi_I| \xi_0 = \tilde \xi_0, \xi_i = 1} \prob{\xi_O = \tilde \xi_O| \xi_0 = \tilde \xi_0, \xi_i = 1}.
        \end{align}
    \end{proposition}
    \begin{proof}[Proof of Proposition \ref{prop: cond_indep1}]
         We will use the conditional inpendence Proposition \ref{prop: cond_indep} in the original tree. In the original tree, $\xi_i$ is either inside or outside of the subtree. Without loss of generality, we assume $\xi_i$ sits outside.
            \begin{align}
                \begin{split}
                \prob{\xi_I = \tilde \xi_I, \xi_O = \tilde \xi_O | \xi_0 = \tilde \xi_0, \xi_i = 1} = \frac{\prob{\xi_I = \tilde \xi_I, \xi_O = \tilde \xi_O, \xi_i = 1 | \xi_0 = \tilde \xi_0}}{\prob{\xi_i = 1 | \xi_0 = \tilde \xi_0}}.
                \end{split}
            \end{align}
            Since both $\xi$ and $\xi_O$ are outside of the subtree, $\xi_I$ is inside of the subtree, by Proposition \ref{prop: cond_indep},
            \begin{align}
                \prob{\xi_I = \tilde \xi_I, \xi_O = \tilde \xi_O, \xi_i = 1 | \xi_0 = \tilde \xi_0} = \prob{\xi_I = \tilde \xi_I | \xi_0 = \tilde \xi_0} \prob{ \xi_O = \tilde \xi_O, \xi_i = 1 | \xi_0 = \tilde \xi_0}.
            \end{align}
            For the right hand side, by Proposition \ref{prop: cond_indep}, the first term  % or and other lemmas... ?
            \begin{align}
                \prob{\xi_I = \tilde \xi_I| \xi_0 = \tilde \xi_0, \xi_i = 1} = \prob{\xi_I = \tilde \xi_I | \xi_0 = \tilde \xi_0}.
            \end{align}
            For the second term,
            \begin{align}
                 \prob{\xi_O = \tilde \xi_O| \xi_0 = \tilde \xi_0, \xi_i = 1} = 
                    \frac{\prob{\xi_O = \tilde \xi_O, \xi_i = 1 | \xi_0 = \tilde \xi_0}}{\prob{\xi_i = 1 | \xi_0 = \tilde \xi_0}}.
            \end{align}
            Combining these equalities, we get
            \begin{align}
                \prob{\xi_I = \tilde \xi_I, \xi_O = \tilde \xi_O | \xi_0 = \tilde \xi_0, \xi_i = 1} = \prob{\xi_I = \tilde \xi_I| \xi_0 = \tilde \xi_0, \xi_i = 1} \prob{\xi_O = \tilde \xi_O| \xi_0 = \tilde \xi_0, \xi_i = 1}.
            \end{align}
    \end{proof}

        Now we will say that the conditional joint probability of two leaf nodes, given the status of the parent node and $\xi_i = 1$, still falls in either category $0$ and category $1$. In fact, we will prove that the category is not changed before and after conditioning on $\xi_i = 1$. % for the remaining nodes
            % special treat for the "combined" node?
    \begin{proposition} \label{prop: cat1}
        For any non-leaf node $\xi_0$ of the remaining tree with two children nodes $\xi_1, \xi_2$, % say this parent children relation is not changed by modification?
            if in the original tree $\xi_0$ is of category $0$, then    % say the category is not changed seems not clear...
        \begin{align}
            \prob{(\xi_1, \xi_2) = (1,1) | \xi_0 = 1, \xi_i = 1} = 1,
        \end{align}
        and for some $\tilde p \in (0, 1)$, % really (0,1) ?
        \begin{align}
            \prob{(\xi_1, \xi_2) = (1,0) | \xi_0 = 0, \xi_i = 1} = \tilde p, \quad \prob{(\xi_1, \xi_2) = (0,1) | \xi_0 = 0, \xi_i = 1} = 1 - \tilde p.
        \end{align}

        If in the original tree $\xi_0$ is of category $1$, then for some $\tilde p   \in (0, 1)$,
        \begin{align}
            \prob{(\xi_1, \xi_2) = (1,0) | \xi_0 = 1, \xi_i = 1} = \tilde p, \quad \prob{(\xi_1, \xi_2) = (0,1) | \xi_0 = 1, \xi_i = 1} = 1 - \tilde p,
        \end{align}
        and 
        \begin{align}
                \prob{(\xi_1, \xi_2) = (0,0) | \xi_0 = 0, \xi_i = 1} = 1.
        \end{align}
    \end{proposition}

    \begin{proof}[Proof of Proposition \ref{prop: cat1}]
        % need to say that in the remaining tree, all non-leaf nodes have two children nodes
        We need to discuss, in the original tree, whether $\xi_i$ is outside of the subtree rooted at $\xi_0$, or inside of it. When inside, we always assume that $\xi_0$ is in the subtree rooted at $\xi_1$.

        If $\xi_0$ is of category $0$, and $\xi_i$ is outside of the $\xi_0$ subtree, then by Proposition \ref{prop: cond_indep},
            \begin{align}
                \prob{(\xi_1, \xi_2) = (1,1) | \xi_0 = 1, \xi_i = 1} = \prob{(\xi_1, \xi_2) = (1,1) | \xi_0 = 1} = 1.
            \end{align}
            And,
            \begin{align}
                 \prob{(\xi_1, \xi_2) = (1,0) | \xi_0 = 0, \xi_i = 1} = \prob{(\xi_1, \xi_2) = (1,0) | \xi_0 = 0} = p,
            \end{align}
            \begin{align}
             \prob{(\xi_1, \xi_2) = (0,1) | \xi_0 = 0, \xi_i = 1} = \prob{(\xi_1, \xi_2) = (0,1) | \xi_0 = 0} = 1 - p.
            \end{align}

        If $\xi_0$ is of category $0$, and $\xi_i$ is inside the $\xi_1$ subtree, then by Proposition \ref{prop: one-step_cond1}, $\prob{\xi_0 = \xi_2 | \xi_1 = 1} = 1$. Similarly as in Proposition \ref{prop: one-step_cond1}, we can prove $\prob {\xi_0 = 1 | \xi_1 = 0} = 0$. Therefore,
        \begin{align}
            \begin{split}
             & \prob{(\xi_1, \xi_2) = (1,1) | \xi_0 = 0, \xi_i = 1} = \frac{\prob{\xi_2 = 1, \xi_0 = 1, \xi_i = 1| \xi_1 =1} \prob{\xi_1 = 1}}{\prob{\xi_0 = 1, \xi_i = 1 | \xi_1 = 1} \prob{\xi_1 = 1} + \prob{\xi_0 = 1, \xi_i = 1| \xi_1 = 0} \prob{\xi_1= 0}}\\
                =  & \frac
                {\prob{\xi_2 = 1, \xi_0 = 1| \xi_1 =1} \prob{\xi_i = 1| \xi_1 =1} \prob{\xi_1 = 1}}
                {\prob{\xi_0 = 1 | \xi_1 = 1} \prob{\xi_i = 1 | \xi_1 = 1} \prob{\xi_1 = 1} + \prob{\xi_0 = 1| \xi_1 = 0} \prob{\xi_i = 1| \xi_1 = 0} \prob{\xi_1= 0}} = 1.
            \end{split}
        \end{align}

        \begin{align}
            \begin{split} % also need (0,1) case -- sum should be 1.
             & \prob{(\xi_1, \xi_2) = (1,0) | \xi_0 = 0, \xi_i = 1} = \frac{\prob{\xi_2 = 0, \xi_0 = 0, \xi_i = 1| \xi_1 =1} \prob{\xi_1 = 1}}{\prob{\xi_0 = 0, \xi_i = 1 | \xi_1 = 1} \prob{\xi_1 = 1} + \prob{\xi_0 = 0, \xi_i = 1| \xi_1 = 0} \prob{\xi_1= 0}}\\
                =  & \frac
                {\prob{\xi_2 = 0, \xi_0 = 0| \xi_1 =1} \prob{\xi_i = 1| \xi_1 =1} \prob{\xi_1 = 1}}
                {\prob{\xi_0 = 0 | \xi_1 = 1} \prob{\xi_i = 1 | \xi_1 = 1} \prob{\xi_1 = 1} + \prob{\xi_0 = 0| \xi_1 = 0} \prob{\xi_i = 1| \xi_1 = 0} \prob{\xi_1= 0}} \in (0,1) % need to express it with p.
            \end{split}
        \end{align}

         If $\xi_0$ is of category $1$, and $\xi_i$ is outside of $\xi_0$ subtree, then by Proposition \ref{prop: cond_indep},
         \begin{align}
            \prob{(\xi_1, \xi_2) = (1,0) | \xi_0 = 1, \xi_i = 1} = \prob{(\xi_1, \xi_2) = (1,0) | \xi_0 = 1} = p, \end{align}
        \begin{align}  \prob{(\xi_1, \xi_2) = (0,1) | \xi_0 = 1, \xi_i = 1} = \prob{(\xi_1, \xi_2) = (0,1) | \xi_0 = 1}  = 1 - p,
        \end{align}
         \begin{align}
                \prob{(\xi_1, \xi_2) = (0,0) | \xi_0 = 0, \xi_i = 1} = \prob{(\xi_1, \xi_2) = (0,0) | \xi_0 = 0} = 1.
        \end{align}

        If $\xi_0$ is of category $1$, and $\xi_i$ is inside of $\xi_1$ subtree, then by Proposition \ref{prop: one-step_cond1}, $\prob{\xi_0 = 1|\xi_1 = 1} = 1, \prob{\xi_2 = 0 | \xi_1 = 1} = 1$. So $\prob{\xi_0 = 0|\xi_1 = 1} =  0$.

        \begin{align}
            \begin{split} % also need (0,1) case -- sum should be 1.
             & \prob{(\xi_1, \xi_2) = (0, 0) | \xi_0 = 0, \xi_i = 1} = \frac{\prob{\xi_2 = 0, \xi_0 = 0, \xi_i = 1| \xi_1 =0} \prob{\xi_1 = 0}}{\prob{\xi_0 = 0, \xi_i = 1 | \xi_1 = 1} \prob{\xi_1 = 1} + \prob{\xi_0 = 0, \xi_i = 1| \xi_1 = 0} \prob{\xi_1= 0}}\\
                =  & \frac
                {\prob{\xi_2 = 0, \xi_0 = 0| \xi_1 =0} \prob{\xi_i = 1| \xi_1 =0} \prob{\xi_1 = 0}}
                {\prob{\xi_0 = 0 | \xi_1 = 1} \prob{\xi_i = 1 | \xi_1 = 1} \prob{\xi_1 = 1} + \prob{\xi_0 = 0| \xi_1 = 0} \prob{\xi_i = 1| \xi_1 = 0} \prob{\xi_1= 0}} = 1. % need to express it with p.
            \end{split}
        \end{align}

        \begin{align}
            \begin{split} % also need (0,1) case -- sum should be 1.
             & \prob{(\xi_1, \xi_2) = (1,0) | \xi_0 = 1, \xi_i = 1} = \frac{\prob{\xi_2 = 0, \xi_0 = 1, \xi_i = 1| \xi_1 =1} \prob{\xi_1 = 1}}{\prob{\xi_0 = 1, \xi_i = 1 | \xi_1 = 1} \prob{\xi_1 = 1} + \prob{\xi_0 = 1, \xi_i = 1| \xi_1 = 0} \prob{\xi_1= 0}}\\
                =  & \frac
                {\prob{\xi_2 = 0, \xi_0 = 1| \xi_1 =1} \prob{\xi_i = 1| \xi_1 =1} \prob{\xi_1 = 1}}
                {\prob{\xi_0 = 1 | \xi_1 = 1} \prob{\xi_i = 1 | \xi_1 = 1} \prob{\xi_1 = 1} + \prob{\xi_0 = 1| \xi_1 = 0} \prob{\xi_i = 1| \xi_1 = 0} \prob{\xi_1= 0}} \in (0,1) % need to express it with p.
            \end{split}
        \end{align}        
    \end{proof}

    \begin{proof}[Proof of Theorem \ref{thm: cpnc_one}]
    By Propositions \ref{prop: cond_indep1} and \ref{prop: cat1}, we know that after conditioning on $\xi_i = 1$, the remaining tree structure still has the conditional independence property, and at each branch, the joint distribution of two children nodes given the status of the parent node can only be case (a) or (b) in Proposition \ref{prop: one-step_cond}. By the perspective of pivotal sampling in Proposition \ref{prop: cat}, the conditional distribution given $\xi_i = 1$ still follows from pivotal sampling.
    \end{proof}